\documentclass[11pt]{article}
\usepackage{graphicx,color,amsfonts,amsmath,amssymb,enumerate}
 \usepackage{url,fancyhdr,indentfirst}
   \usepackage{amsthm,natbib,comment,bm}
    \usepackage[colorlinks=true,citecolor=blue]{hyperref}
\usepackage{dsfont}
\usepackage{mathrsfs}
\usepackage{tikz-qtree}
\usepackage{xr}
\usepackage{threeparttable}
\usepackage[title]{appendix}

\usepackage{makecell}

\usepackage{amsmath} 
\DeclareMathOperator*{\argmax}{arg\,max}

\def\d{\mathrm{d}}
\def\laweq{\buildrel \d \over =}

\newcommand{\E}{\mathbb{E}}
\newcommand{\R}{\mathbb{R}}

\newcommand{\M}{\mathcal{M}}

\newcommand{\N}{\mathbb{N}}
\newcommand{\p}{\mathbb{P}}

\newcommand{\id}{\mathds{1}}

\newcommand{\esssup}{\mathrm{ess\mbox{-}sup}}

\renewcommand{\ge}{\geqslant}
\renewcommand{\le}{\leqslant}
\renewcommand{\geq}{\geqslant}
\renewcommand{\leq}{\leqslant}
\renewcommand{\epsilon}{\varepsilon}

\newcommand{\bb}{\boldsymbol{\beta}}

\theoremstyle{plain}
\newtheorem{theorem}{Theorem}
\newtheorem{corollary}{Corollary}
\newtheorem{lemma}{Lemma}
\newtheorem{proposition}{Proposition}
\newtheorem{assumption}{Assumption}

\theoremstyle{definition}
\newtheorem{definition}{Definition}
\newtheorem{example}{Example}

\theoremstyle{remark}
\newtheorem{remark}{Remark}

\theoremstyle{definition}

\renewcommand{\cite}{\citet}
\renewcommand{\cdots}{\dots}
\DeclareMathOperator*{\argmin}{arg\,min}
\usepackage{tikz}

\usepackage[onehalfspacing]{setspace}

\begin{document}

\title{On Generalization and Regularization via Wasserstein Distributionally Robust Optimization}
\author{Qinyu Wu\thanks{Department of Statistics and Finance,
University of Science and Technology of China, China. E-mail: wu051555@mail.ustc.edu.cn}
\and Jonathan Yu-Meng Li\thanks{Telfer School of Management, University of Ottawa, Ottawa, Ontario K1N 6N5, Canada. E-mail: jonathan.li@telfer.uottawa.ca}
\and Tiantian Mao\thanks{Department of Statistics and Finance,
University of Science and Technology of China, China. E-mail: tmao@ustc.edu.cn}
}

 \date{\today}

\maketitle

\begin{abstract}
Wasserstein distributionally robust optimization (DRO) has gained prominence in operations research and machine learning as a powerful method for achieving solutions with favorable out-of-sample performance. Two compelling explanations for its success are the generalization bounds derived from Wasserstein DRO and its equivalence to regularization schemes commonly used in machine learning. However, existing results on generalization bounds and regularization equivalence are largely limited to settings where the Wasserstein ball is of a specific type, and the decision criterion takes certain forms of expected functions. In this paper,  we show that generalization bounds and regularization equivalence can be obtained in a significantly broader setting, where the Wasserstein ball is of a general type and the decision criterion accommodates any form, including general risk measures. This not only addresses important machine learning and operations management applications  but also expands to general decision-theoretical frameworks previously unaddressed by Wasserstein DRO. Our results are strong in that the generalization bounds do not suffer from the curse of dimensionality and the equivalency to regularization is exact. As a by-product, we show that Wasserstein DRO coincides with the recent max-sliced Wasserstein DRO for {\it any} decision criterion under affine decision rules -- resulting in both being efficiently solvable as convex programs via our general regularization results. These general assurances provide a strong foundation for expanding the application of Wasserstein DRO across diverse domains of data-driven decision problems.

\noindent \begin{bfseries}Key-words\end{bfseries}: distributionally robust optimization, Wasserstein metrics, finite-sample guarantees, regularization
\end{abstract}

\section{Introduction}
Stochastic optimization problems of the form
\begin{align}\label{eq-main01}
	\inf_{\bb \in \mathcal D} \rho(Y \cdot \bb^\top \mathbf X)
\end{align}
naturally arise in many machine learning (ML) and operations research (OR) applications, where $Y$ denotes a binary random variable, taking values from $\{-1,1\}$, and $\mathbf X$ is a random vector in ${\mathbb R}^n$. The function $
f_{\bb}(\mathbf X)=\bb^\top \mathbf X$, parameterized by the decision variable $\bb$ $\in \mathcal D\subseteq \R^n$, can generally be interpreted as an affine decision rule.
In ML, the binary random variable $Y$ often represents outcomes in classification problems.   When $Y$ is constant, the formulation \eqref{eq-main01} applies to a broad range of regression problems in ML and risk minimization problems in OR. The function $\rho$ represents a measure of risk, mapping a random variable to a real number that quantifies its risk. The notion of risk in this paper is broadly defined as the undesirability of a random variable $Z$; that is, a larger value of $\rho(Z)$ indicates that $Z$ is less preferable. 


In this paper, we study the Wasserstein distributionally robust counterpart of  \eqref{eq-main01}, specifically: 
\begin{align}\label{eq-main02}
	\inf_{\bb \in \mathcal D} \sup_{F\in {\cal B}_p(F_0,\epsilon)} \rho^F(Y \cdot \bb^\top \mathbf X),
\end{align}
where ${\cal B}_p(F_0,\epsilon)$ denotes a type-$p$ Wasserstein ball of radius $\epsilon$ centred at a reference distribution $F_0$ (c.f. \cite{K19}).
The superscript $F$ in $\rho^F$ indicates that the measure depends on a joint distribution $F$ of  $(Y,\mathbf X)$. This problem has been extensively studied and applied in various fields when $\rho^F$ takes the form of an expected value function, i.e., $ \rho^F(Z) = \mathbb{E}^F[\ell(Z)]$ for some $\ell: {\mathbb R} \rightarrow {\mathbb R}$ (see e.g., \cite{K19, SKE19, GCK17, G22}). This setup reveals two key findings that underscore the potential strength of the Wasserstein  distributionally robust optimization (DRO)  model in achieving strong out-of-sample performance. First, the model enjoys generalization bounds under mild assumptions (\cite{SKE19}). More notably, \cite{G22} shows that when the data-generating distribution satisfies transportation inequalities, $O(N^{-1/2})$-generalization bounds for the type-$p$ Wasserstein ball ${\cal B}_p(F_0,\epsilon)$ with $p \in [1,2]$ can be established, free from the curse of dimensionality. Second, the model has an exact regularization equivalent to the nominal problem \eqref{eq-main01} with an added regularizer on the decision variable $\bb$, when the Wasserstein ball is of type-1 and the loss function $\ell$ is Lipschitz continuous (\cite{SKE19}), as well as when the ball is of type-2 and the loss function $\ell$ is quadratic (\cite{BKM19}). This connection to the classical regularization scheme, prevalent in machine learning, offers a compelling interpretation of the DRO model and has spurred considerable interest in its application within ML and OR (see e.g., \cite{BK21}, \cite{BKM19}, \cite{CP18}, \cite{GCK17}, \cite{GCK22}).

It remains largely unclear, however, whether the generalization bounds established for the specific case of the expected value function can be achieved in the broader setting of \eqref{eq-main02} -- namely, for a general measure $\rho$ and any type-$p$ Wasserstein ball ${\cal B}_p(F_0,\epsilon)$ with $p \in [1,\infty]$ -- without suffering from the curse of dimensionality. In particular, the result of \cite{G22}, aside from the limitations imposed by the transportation inequality assumption on the underlying distribution and the restriction to  $p\in [1,2]$,  fundamentally relies on the unbiased property of the sample average -- a property that does not extend to risk measures such as Conditional Value-at-Risk (CVaR) (\cite{RU02}). It is also worth noting that, even in the literature on ML, the question of how to obtain generalization bounds, possibly through the classical regularization scheme, for the nominal problem \eqref{eq-main01} has largely been left open (c.f. \cite{SB14}). In this paper, we shed new light on the efficacy of Wasserstein DRO in broad out-of-sample tests by showing that the Wasserstein DRO model \eqref{eq-main02} can circumvent the curse of dimensionality for virtually any decision criterion under affine decision rules. As a crucial insight, we uncover that the Wasserstein DRO remarkably coincides with the recent max-sliced Wasserstein DRO (\cite{ORVW22}) for {\it any} decision criterion, even though the ball defined in the former is a far less conservative choice -- specifically, it is smaller and tighter than the one defined in the latter. 
Additionally, we extend our results to general decision rules under the light-tail assumption by leveraging the measure concentration property of a Wasserstein ball.

Another focus of this work is to investigate whether the key equivalence of Wasserstein DRO to classical regularization in ML can be extended to a more general setting in \eqref{eq-main02}. To answer this, we first pay particular attention to the setting where the Wasserstein ball ${\cal B}_p(F_0,\epsilon)$ is of any type, i.e., $p \in [1,\infty]$, and the measure $\rho$ is an expected function $ \rho^F(Z) = \mathbb{E}^F[\ell(Z)]$ with a loss function $\ell$ having a growth order $p$. 
We show that for many loss functions $\ell$ arising from practical applications, it is possible to establish an exact relation between the Wasserstein DRO model \eqref{eq-main02} and the classical regularization scheme in more general terms. Notably, we prove that no loss functions beyond those we identify satisfy this equivalence, offering a definitive answer to the extent to which Wasserstein DRO can be interpreted and solved from a regularization perspective. Our results also reveal the tractability of solving Wasserstein DRO for many non-Lipschitz continuous loss functions $\ell$ and higher-order Wasserstein balls ($p>1$). While \cite{SZZ21} approach such problems via approximation due to reformulation challenges, we show that exact solutions are possible through regularization reformulations. Moreover, we extend this equivalence to settings where $\rho$ goes beyond expected value functions. Previously, this exact relation was known only for variance (\cite{BCZ22}) or distortion risk measures (\cite{W14}). We expand it to a significantly broader class of measures, encompassing higher moment coherent risk measures (\cite{K07}), criteria emerging more recently from OR and ML (\cite{RU13}, \cite{RUZ08}, \cite{GU17}), and notably, general decision-theoretical models such as the celebrated rank-dependent expected utility (RDEU) (\cite{Q82}).


\subsection*{Related Work}

\emph{From the perspective of generalization bounds}.
A series of works done by \cite{BKM19}, \cite{BK21}, \cite{BMS22} take a different approach to tackle the curse of dimensionality. They study the classical setting of Wasserstein DRO, where the measure $\rho$ is an expected function, and propose a radius selection rule for Wasserstein balls. They show the rule can be applied to build a confidence region of the optimal solution, and the radius can be chosen in the square-root order as the sample size goes to infinity. Although this allows for bypassing the curse of dimensionality, the bounds derived from the rule are only valid in the asymptotic sense. \cite{BCZ22} also takes this approach to obtain generalization bounds for mean-variance portfolio selection problems. On the other hand, the generalization bounds established in this paper, like those in \cite{SKE19}, \cite{CP18}, and \cite{G22}, break the curse of dimensionality in a non-asymptotic sense for finite samples.

\emph{From the perspective of the equivalency between Wasserstein DRO and regularization}. While the focus of this work is on studying the exact equivalency between Wasserstein DRO and regularization, there is an active stream of works studying the asymptotic equivalence in the setting where the measure $\rho$ is an expected function
(see \cite{GCK17, BKM19, BMS22, VNSDMS18, BDSW20}). In particular, \cite{GCK22} introduce the notion of variation regularization and show that for a broad class of loss functions, Wasserstein DRO is asymptotically equivalent to a variation regularization problem.

%

\subsection*{Recap of Major Contributions and Practical Implications}
\begin{enumerate}
    \item We underscore the fundamental gap in achieving $O(N^{-1/2})$-generalization bounds for diverse decision criteria in Wasserstein DRO and novelly establish such guarantees by leveraging the projection set of the Wasserstein ball, revealing a broad equivalence with the recent max-sliced Wasserstein DRO for any decision criterion.
    \item By proving the precise conditions under which Wasserstein DRO has an exact regularization equivalent, we provide key insights into a wide range of OR/MS applications, highlighting when a simpler, intuitive, and computationally efficient regularization approach suffices to achieve the generalization guarantees of Wasserstein DRO -- and when solving the full Wasserstein DRO formulation is necessary. 
    \item By extending generalization bounds beyond affine decision rules, we uncover important OR/MS applications, such as foundational two-stage stochastic programs and the increasingly popular feature-based newsvendor problem, with generalization guarantees that scale effectively across diverse decision criteria.
\end{enumerate}

\section{Wasserstein Distributionally Robust Optimization Model} \label{WDRO}\label{rcr}
Let  $(\Omega, \mathcal A,\p)$ be an atomless probability space.  A random vector $\bm\xi$ is a measurable mapping from $\Omega$ to $\R^{n+1}$, $n\in\N$. Denote by $F_{\bm\xi}$ the distribution of  $\bm\xi$ under $\p$. 
For $p\ge 1$, let $L^p:=L^p(\Omega, \mathcal A,\p)$ be the set of all random variables with a finite $p$th moment, and 
$\M_p(\Xi)$ be the set of all distributions on $\Xi\subseteq\R^{n+1}$ with finite $p$th moments  in each component.  Let $L^\infty:=L^\infty(\Omega, \mathcal A,\p)$ be the set of all bounded random variables, and 
$\M_\infty(\Xi)$ be the set of all distributions on $\Xi\subseteq\R^{n+1}$ with bounded support.
Let $q$ denote the H\"{o}lder conjugate of $p$, i.e., $1/p+1/q=1$. We use $\esssup Z$ to represent the essential supremum of the random variable $Z$, i.e., $\esssup Z=\inf\{t: F_Z(t)\ge 1\}$. 
Recall that for any two distributions $F_1 ,\,F_2 \in \M_p(\Xi)$, the type-$p$ Wasserstein metric is defined as
\begin{align} \label{eq:dwasser}
	W_{d,p}\left(F_1, F_2\right)= \inf_{\pi \in \Pi(F_1,F_2)}
	 \big(\mathbb{E}^\pi [d(\bm\xi_1, \bm\xi_2)^p] \big)^{ {1}/{p}},
\end{align}
where  $d(\cdot,\cdot):\Xi\times\Xi\to\R_+\cup\{\infty\}$ is a metric on $\Xi$.
The set $\Pi(F_1,F_2)$ denotes the set of all joint distributions of $\bm\xi_1$ and $\bm\xi_2$ with marginal distributions $F_1$ and $F_2$, respectively.   In this paper, we consider $\bm\xi=(Y,{\bf X})\in \Xi$  with  $\Xi:=\{-1,1\}\times\R^n\subseteq \R^{n+1}$ and  apply the type-$p$ Wasserstein metric \eqref{eq:dwasser} with $d( \bm\xi_1 ,\bm\xi_2 )= d ((Y_1,\mathbf X_1),(Y_2,\mathbf X_2)) $, defined by the following additively separable form
\begin{align}\label{eq-classificationmetric}
	d((Y_1,\mathbf X_1), (Y_2,\mathbf X_2)) =
	\Vert \mathbf X_1-\mathbf X_2 \Vert + \Theta(Y_1-Y_2),
\end{align}
where $\Vert \cdot \Vert$ is any given norm on $\R^n$ with its dual norm $\|\cdot\|_*$ defined by $\|\mathbf y\|_*=\sup_{\|\mathbf x\|\le 1}\mathbf x^\top \mathbf y$. 
The function $\Theta: \R \rightarrow \{0,\infty\}$ satisfies $\Theta(s) = 0$ if $s=0$ and $\Theta(s) = \infty$ otherwise. Thus, the function \eqref{eq-classificationmetric} assigns an infinitely large cost to any discrepancy in $Y$, i.e., $Y_1-Y_2 \neq 0$, and reduces to a general norm on $\mathbf X$ when there is no discrepancy in $Y$, i.e., $Y_1-Y_2=0$. Using this norm,  for $F_0\in \mathcal M_p(\Xi)$ and $\epsilon\ge 0$,  we define the ball of distributions $\overline{{\cal B}}_p(F_0,\epsilon)$ as
\begin{align}\label{eq-WU-multi}
	\overline{{\cal B}}_{p}(F_0,\epsilon) =\left\{F\in \M_p(\Xi): W_{d,p}(F,F_0) \le\epsilon\right\},
\end{align}
which we refer to as the type-$p$ Wasserstein ball throughout this paper. 

In the remainder of this paper, we first demonstrate that the DRO problem  \eqref{eq-main02}, using the above definition of the Wasserstein ball:
\begin{align}\label{eq-main05}
\inf_{\bb \in \mathcal{D}} \sup_{F \in \overline{{\cal B}}_{p}(F_0,\epsilon)} \rho^F(Y \cdot \bb^\top \mathbf{X}),
\end{align}
can achieve generalization guarantees for any measure $\rho$ and  type-$p$ Wasserstein ball.
We then elucidate this guarantee for the widely used affine decision rules by exploring the connection between the optimization model \eqref{eq-main05} and the classical regularization scheme in Section \ref{reg}.  This framework covers various problem types, including:
\begin{itemize}
\item {\bf Classification}: 
$$\inf_{\bb\in \mathcal D} \sup_{F\in \overline{{\cal B}}_{p}(F_0,\epsilon)} \rho^F(Y \cdot  \bb^\top \mathbf X),$$
\item {\bf Regression}: 
$$\inf_{\bb_r\in \mathcal D}\sup_{F\in \overline{{\cal B}}_{p}(F_0,\epsilon)} \rho^F((1, - \bb_r)^{\top}\mathbf X),
$$
\item {\bf Risk minimization}:
$$\inf_{\bb\in \mathcal D} \sup_{F\in \overline{{\cal B}}_{p}(F_0,\epsilon)} \rho^F(\bb^\top \mathbf X), $$
\end{itemize}
 where $F_0(\{Y=1\})=1$ in the cases of regression and risk minimization. These formulations accommodate numerous machine learning and operations research applications not yet addressed in the existing literature on Wasserstein DRO. Table 1 highlights examples considered in this paper.

\begin{table}

\caption{Examples of measures $\rho$}\label{tab-R}\def\arraystretch{2}
	\begin{center}
		\resizebox{\textwidth}{!}{
			\begin{tabular}{llcc}
				\hline
				\textbf{Applications} &    \textbf{Measure}   & \textbf{Formulation} & \textbf{References}\\[8pt]
				\hline
				\textbf{Classification} & \makecell[l]{HO hinge loss\\[5pt] HO SVM\\[5pt] $\nu$-SVM} & \makecell{$\E[(1-Z)_+^s]$\\[5pt] $\E[|1-Z|^s]$\\[5pt] ${\rm CVaR}_\alpha(-Z)$} & \makecell{\cite{LM01}, $s=2$\\[5pt] \cite{SV99}, $s=2$\\[5pt] \cite{SSWB00}; \cite{GU17}}\\[5pt]
				\hline
				\textbf{Regression} & \makecell[l]{HO regression\\[5pt] HO $c$-insensitive\\[5pt] $\nu$-SVR} & \makecell{$\E[|Z|^s]$\\[5pt] $\E[(|Z|-r)^s_+]$\\[5pt] ${\rm CVaR}_\alpha(|Z|)$} & \makecell{the well-known least-square regression for $s=2$\\[5pt] \cite{DBKSV97}, $s=1$; \cite{LHH05}, $s=2$\\[5pt] \cite{S98}}\\[5pt]
				\hline
				\textbf{Risk minimization} & \makecell[l]{LPM\\[5pt] CVaR-Deviation\\[5pt] HM risk measure} &  \makecell{$\E[(Z-c)_+^s]$\\[5pt] ${\rm CVaR}_\alpha(Z-\E[Z])$\\[5pt] $\inf_{t\in\R}\{t+k(\E[(Z-t)_+^s])^{1/s}\}$} & \makecell{\cite{B75}; \cite{F77}; \cite{CSS11}\\[5pt] \cite{RUZ08}\\[5pt] \cite{K07}}\\[5pt]
				\hline
			\end{tabular}
		}
	\end{center}
~~\\
\footnotesize
{\it Notes}: HO: higher-order; SVM: support vector machine; SVR: support vector regression; LPM: lower partial moments;  HM: higher moment; CVaR: conditional Value-at-Risk,  ${\rm CVaR}_{\alpha}(Z)=\int_{\alpha}^1F_{Z}^{-1}(t)\d t/(1-\alpha)$, $Z\in L^1$, where $F_Z^{-1}$ is the left-quantile function of $Z$.
The range of parameters: $s\ge 1$; $\alpha\in(0,1)$; $r\ge 0$; $c\in\R$; $k\ge 1$.

\end{table}

\section{Generalization Bounds} \label{Gen}\label{sec:Gen}
Wasserstein DRO is typically applied in a data-driven setting, where the data-generating distribution $F^*$ of random variables $(Y,\mathbf X) \in\Xi= \{-1,1\} \times {\mathbb R}^n$ is only partially observed through  sample data points $(\widehat{y}_i,\widehat{\mathbf x}_i)$ for $i=1,...,N$, independently drawn from $F^*$. In this context, the empirical distribution $\widehat{F}_N$ is used as the reference distribution $F_0 \in\mathcal M_p(\Xi)$ in Wasserstein DRO, where 
$\widehat{F}_N:= \frac{1}{N} \sum_{i=1}^N \delta_{(\widehat{y}_i,\widehat{\mathbf x}_i)}$  
and  $\delta_{\mathbf x}$ denotes the point-mass at $\mathbf x$. The central question is whether 
the (Wasserstein) in-sample risk
$$\sup_{F\in \overline{{\cal B}}_p(\widehat{F}_N,\epsilon_N)}\rho^{F}(Y \cdot \bb^\top \mathbf X),$$
can provide an upper confidence bound on the out-of-sample risk
$$\rho^{F^*}(Y \cdot \bb^\top \mathbf X) $$
across all possible decisions $\bb \in {\cal D}$, with a radius $\epsilon_N$ that decays at a rate scaling gracefully with the dimensionality of the random vector $(Y,\mathbf X)$ -- in other words, generalization bounds that overcome the  curse of dimensionality. 

Current methods face significant limitations. While generalization bounds based on the measure concentration properties of type-$p$ Wasserstein balls offer flexibility in accommodating general risk measures $\rho$, they are constrained by the curse of dimensionality inherent in high-dimensional Wasserstein balls (see e.g., \cite{EK18}). Approaches such as \cite{G22} address this issue by focusing on expected value functions, i.e., $\rho^F(Z) = \mathbb{E}^F[\ell(Z)]$, leveraging the concentration properties of sample averages. However, extending these methods beyond expected value functions proves challenging, as most risk measures lack analogous concentration properties. 

We take a novel perspective to tackle this challenge by focusing on the structural properties of the projection set:
\begin{align}\label{eq-HONE}
 \overline{\mathcal B}_{p|\bb} (F_0,\epsilon):=\{F_{Y\cdot \bb^\top \mathbf X}: F_{(Y,\mathbf X)}\in\overline{\mathcal B}_p(F_0,\epsilon)\},
\end{align}
where $F_0\in\mathcal M_p(\Xi)$, $\bb\in\R^n$ and $\epsilon\ge 0$.
We uncover that this set coincides with two particularly revealing one-dimensional sets: a one-dimensional Wasserstein ball and the projection set of a high-dimensional max-sliced Wasserstein ball. To formalize this,   for $G_0\in \mathcal M_p(\R^n)$,  we first recall the definition of a type-$p$ Wasserstein ball  on $\R^n$ with the metric $d(\cdot,\cdot)$  being a norm:  
 \begin{align}\label{eq-WU-multi-d}
		{\cal B}_{p}(G_0,\epsilon) =\left\{G\in \M_p(\R^n):W_p(G,G_0)   \le\epsilon\right\} ~~{\rm with}~~W_p(G,G_0):=\inf_{\pi \in \Pi(G,G_0)}
		\left(\mathbb{E}^\pi [\|\bm\xi_1- \bm\xi_2\|^p] \right)^{ {1}/{p}},
	\end{align}
and a type-$p$ max-sliced Wasserstein ball (\cite{ORVW22}):
\begin{align}\label{eq-WU-multi-d2}
		{\cal B}_{p}^{\rm ms}(G_0,\epsilon) =\left\{G\in \M_p(\R^n):  \sup_{\bm\beta:\|\bm\beta\|_*=1}\inf_{\pi\in\Pi(G ,G_0)}\left(\E^{\pi}[|\bm\beta^\top \bm\xi_1-\bm\beta^\top \bm\xi_2|^p]\right)^{1/p} \le\epsilon\right\}.
	\end{align}
Here, $\Vert \cdot \Vert$ is the norm defined in \eqref{eq-classificationmetric} on $\R^n$. 
Without loss of generality  and to be consistent with the definition of classic one-dimensional Wasserstein metric, for $n=1$, i.e., on $\R$, we assume that $\|\cdot\|=|\cdot|$ is the absolute-value norm.
 
With these definitions in place, we can now present the following equivalence results.

\begin{proposition} \label{prop-equi}\label{thm:221010-1}
Suppose that $p\in[1,\infty]$, $\epsilon\ge 0$, $\bm\beta\in \R^n$,  $F_0\in\mathcal M_p(\Xi)$ and $(Y_0,\mathbf X_0)\sim F_0$. Then,  we have 
\begin{align}\label{eq:221010-1}
		\{ F_{Y\bf X}: F_{(Y,{\bf X})} \in \overline{\cal B}_p (F_0,\epsilon)\}
		={\cal B}_p (F_{Y_0\bf X_0},\epsilon)
\end{align}
and
\begin{align}\label{eq:1010-2}
 \overline{ \mathcal B}_{p|\bb} ( F_0,\epsilon) =\mathcal B_p\left(F_{Y_0\cdot\bm\beta^{\top}\bf X_0},\epsilon\|\bm\beta\|_*\right) =\{F_{\bm \beta^\top \mathbf Z}: F_{\bf Z}\in  \mathcal B_p^{\rm ms}(F_{Y_0\mathbf X_0},\epsilon)\},
\end{align}
where $\overline{\mathcal B}_{p|\bb}$, ${\cal B}_p  $, and $\mathcal B_p^{\rm ms}$ are defined by  \eqref{eq-HONE}, \eqref{eq-WU-multi-d}, and \eqref{eq-WU-multi-d2}, respectively.
\end{proposition}


The equivalence established above has profound implications for both Wasserstein DRO and max-sliced Wasserstein DRO. It shows that the two frameworks coincide for \emph{any} decision criterion $\rho$ under affine decision rules, both reducing to the problem of minimizing worst-case risk over a one-dimensional type-$p$ Wasserstein ball -- an inherently more tractable problem, as detailed in Section \ref{reg}. Notably, this equivalence reveals that, although the Wasserstein distance is a significantly stronger metric than the max-sliced Wasserstein distance, i.e., the type-$p$ Wasserstein ball is a much less conservative choice of an uncertainty set than its max-sliced counterpart (\cite{ORVW22}), Wasserstein DRO can still achieve the same generalization bounds. This insight directly leads to the following generalization bounds for general measures $\rho$, free from the curse of dimensionality.

\begin{theorem}\label{eq-UGB}\label{th-FSG}
Given $p\in[1,\infty)$, $\eta\in(0,1)$ and $\mathcal D\subseteq \R^n$,
if $\Gamma:=\E^{F^*}[\|\mathbf X\|^s]<\infty$ for some $s>2p$, 
then we have
	\begin{align*}
		\p\left(\rho^{F^*}(Y\cdot\bm \beta^\top\mathbf X)\le \sup_{F\in \overline{{\cal B}}_p(\widehat{F}_N,\epsilon_{p,N}(\eta))}\rho^{F}(Y\cdot\bm \beta^\top\mathbf X),~~\forall \bm\beta\in\mathcal D\right)\ge 1-\eta,
	\end{align*}
where $\epsilon_{p,N}(\eta)^p=C\log(2N+1)^{p/s}/\sqrt{N}$ and $$C=2^p p\left(180\sqrt{n+2}+\sqrt{2\log\left(\frac{3}{\eta}\right)}+\sqrt{\frac{3\Gamma}{\eta}}\frac{8}{s/2-p}\sqrt{\log{\left(\frac{24}{\eta}\right)}+2(n+2)}\right).$$
\end{theorem}

The decay rate of the radius is approximately of the order $N^{-1/(2p)}$. While this rate is unaffected by dimensionality, its dependency on the type of Wasserstein ball, specifically the parameter $p$, is noteworthy. This dependency, often overlooked in the literature compared to the curse of dimensionality, is potentially concerning, as it suggests that the radius may shrink very slowly for higher-order Wasserstein balls. This issue arises from the use of stronger metrics -- specifically, norms raised to higher powers -- when defining any ball constructed similarly to the Wasserstein ball,  including the max-sliced Wasserstein ball (\cite{ORVW22}). Intuitively, a higher-order Wasserstein ball is smaller than a lower-order one, requiring a larger radius to ensure the data-generating distribution is covered with high probability. In the extreme case of a type-$\infty$ Wasserstein ball, it would never contain a data-generating distribution with unbounded support, regardless of the radius. 


We circumvent this issue, which we refer to as the ``curse of the order $p$", by recognizing that ensuring high-probability coverage of the data-generating distribution, while sufficient, is not necessary for developing generalization bounds.  In fact, we show that even for the type-$\infty$ Wasserstein ball -- where the probability of covering a data-generating distribution with unbounded support is effectively zero -- the worst-case risk from the type-$\infty$ ball can still bound the out-of-sample risk of the data-generating distribution with high probability. This guarantee requires only the following mild assumption about the measure $\rho$.

\begin{assumption}\label{ass:order1}
There exists $\lambda\ge 1$ such that 
\begin{align*}
\sup_{|V|\le \lambda\epsilon}\rho(Z+V)\ge \sup_{\E[|V|]\le \epsilon}\rho(Z+V),~~\forall Z\in L^1,~\epsilon\ge 0.
\end{align*}
\end{assumption}


 For any measures satisfying Assumption \ref{ass:order1}, one can leverage the generalization bounds for the type-1 Wasserstein ball to derive similar bounds for the  type-$\infty$ Wasserstein ball, with a radius scaled by a constant $\lambda$. These bounds apply to any type-$p$ Wasserstein ball, for $p \in [1,\infty]$, resulting in generalization bounds where the decay rate of the radius  is no longer dependent on the order $p$.

\begin{theorem}\label{co-Lnonunion}
Let $p\in[1,\infty]$, $\eta\in(0,1)$ and $\mathcal D\subseteq \R^n$.
Suppose that $\Gamma:=\E^{F^*}[\|\mathbf X\|^s]<\infty$ for some $s>2$ and Assumption \ref{ass:order1} holds. Then, we have
	\begin{align*}
		\p\left(\rho^{F^*}(Y\cdot\bm \beta^\top\mathbf X)\le \sup_{F\in \overline{{\cal B}}_p(\widehat{F}_N,\lambda\epsilon_{1,N}(\eta))}\rho^{F}(Y\cdot\bm \beta^\top\mathbf X),~~\forall \bm\beta\in\mathcal D\right)\ge 1-\eta,
	\end{align*}
where $\epsilon_{1,N}(\eta)$ is defined in Theorem \ref{eq-UGB}.
\end{theorem}

The core idea presented in this section can be extended beyond affine decision rules, as discussed in greater detail in Section \ref{beyond}. Here, we further illustrate the generality of Assumption \ref{ass:order1} by providing additional examples.

\begin{example}[Expected function]
Let  $\rho(Z)=\E[\ell(Z)]$, 
where $\ell:\R\to\R$ is   quasi-convex and Lipschitz continuous with Lipschitz constant ${\rm Lip}(\ell)$,
and satisfies, for any differentiable point $z\in\R$, $|\ell'(z)|\ge M>0$.  
 By $\ell$ is quasi-convex, there exists  $\alpha\in[-\infty,\infty]$ such that $\ell$ is decreasing on $(-\infty,\alpha)$ and increasing otherwise, and thus $\ell$ is differentiable almost everywhere. It follows that 
\begin{align*}
   \sup_{|V|\le \epsilon}\E[\ell(Z+V)]
 &= \E[\ell(Z-\epsilon)\id_{\{Z<\alpha\}}]+\E[\ell(Z+\epsilon)\id_{\{Z\ge \alpha\}}] \\
 &\ge \E[(\ell(Z)+M \epsilon)\id_{\{Z<\alpha\}}]+ \E[(\ell(Z)+M \epsilon)\id_{\{Z\ge\alpha\}}]
 =\E[\ell(Z)]+M \epsilon.  
\end{align*}  
This, together with     
$ 
\sup_{\E[|V|]\le \epsilon}\E[\ell(Z+V)]\le \sup_{\E[|V|]\le \epsilon} \E[\ell(Z) +{\rm Lip}(\ell) |V|] \le \E[\ell(Z)]+ {\rm Lip}(\ell)\epsilon 
$ by the Lipschitz continuity of $\ell$, implies
 $\rho$ satisfies Assumption \ref{ass:order1} with $\lambda={\rm Lip}(\ell)/M $. 
In particular, the function $\ell$, defined as $\ell(x)  = a (x-c)_+ + b(x-c)_-$ with $a>0$, $b,c\in\R$, satisfies Assumption \ref{ass:order1} with $\lambda=\max \{a,|b|\}/\min\{a,|b|\}$.
\end{example}

\begin{example}[Risk measure]
Suppose that $\rho: L^1\to\R$ satisfies monotonicity and translation invariance, i.e.,  $\rho(Z_1) \leq \rho(Z_2)$ whenever $Z_1 \leq Z_2$, and $\rho(Z + m) = \rho(Z) + m$ for all $Z \in L^1$ and $m \in \mathbb{R}$. Such functionals are referred to as monetary risk measures (see, e.g., \cite{FS16}). Further,  assume that $\rho$ is Lipschitz continuous with respect to the $L^1$-norm, i.e., there exists $M > 0$ such that  $|\rho(Z_1) - \rho(Z_2)| \leq M \mathbb{E}[|Z_1 - Z_2|]$ for all $Z_1, Z_2 \in L^1$. 
Note that for any $\epsilon>0$,
$ 
\sup_{\E[|V|]\le \epsilon}\rho(Z+V)\le \sup_{\E[|V|]\le \epsilon}\{\rho(Z)+M\E[|V|]\}=\rho(Z)+M\epsilon.
$  
This, together with  $\sup_{|V|\le \epsilon}\rho(Z+V)\ge  \rho(Z+\epsilon)=\rho(Z)+\epsilon$,     yields that $\rho$ satisfies Assumption \ref{ass:order1} with $\lambda=M$. In particular, CVaR, as defined in Table \ref{tab-R} by ${\rm CVaR}_{\alpha}(Z)=\int_{\alpha}^1F_{Z}^{-1}(t)\d t/(1-\alpha)$, $Z\in L^1$, satisfies Assumption \ref{ass:order1} with $\lambda=1/(1-\alpha)$, where $F_Z^{-1}$ is the quantile function of $Z$.
\end{example}

\section{A Regularization Perspective} \label{reg}
Regularization generally refers to any means applied to avoid overfitting and enhance generalization. In this regard, with the generalization guarantees established in the previous section, the Wasserstein DRO model \eqref{eq-main05} is well justified as a general regularization model. In this section, we provide further insights into the regularization effect of model \eqref{eq-main05} on the decision variable $\bb$, offering a practically useful alternative interpretation. Specifically, we show that the previously observed equivalence between the Wasserstein DRO model and regularized empirical optimization in ML holds across much broader settings, while precisely identifying the boundary of its validity. This equivalence significantly broadens the range of both Wasserstein and max-sliced Wasserstein DRO problems that can now be solved efficiently through their regularization counterparts. Proposition \ref{thm:221010-1} plays a key role in facilitating this analysis. By \eqref{eq:221010-1} in Proposition \ref{thm:221010-1}, the Wasserstein DRO model \eqref{eq-main05} can be recast as
\begin{align}\label{eq-main5}
	\inf_{\bb \in \mathcal D} \sup_{F\in \mathcal  B_p(G_0,\epsilon)}\ \rho^F(\bb^\top \mathbf Z),
\end{align}
where $G_0\in\mathcal M_p(\R^n)$, and $\mathcal  B_p(G_0,\epsilon)$ is a Wasserstein ball on $\R^n$ defined by \eqref{eq-WU-multi-d}. 

\subsection{The case of expected function}\label{subsec:EU}
When $\rho$ is an expected function, the Wasserstein DRO model \eqref{eq-main5}:
\begin{align}\label{eq-main-ef}
	\inf_{\bb \in \mathcal D} \sup_{F\in \mathcal  B_p(G_0,\epsilon)} \mathbb{E}^F[\ell(\bb^\top \mathbf Z)],
\end{align}
is equivalent to the regularized model:
\begin{align}\label{eq-main21}
	\inf_{\bb \in \mathcal D} \left\{\mathbb{E}^{G_0}[\ell(\bb^\top \mathbf Z)] + {\rm Lip(\ell)}\epsilon\Vert\bm\beta\Vert_*\right\},
\end{align}
for $p=1$ (the type-1 Wasserstein ball), where ${\rm Lip(\ell)}$ is the Lipschitz constant of $\ell$. For higher-order Wasserstein balls ($p > 1$) or non-Lipschitz loss functions, this relationship is less understood, except for $p = 2$ and $\ell(x) = x^2$ (\cite{BKM19}). It remains an open question whether \eqref{eq-main-ef} can be tractably solved in these cases. Higher-order Wasserstein balls are less conservative than type-1 Wasserstein balls and offer practical appeal.

We show that an equivalence with a regularized model \eqref{eq-main21} exists for $p > 1$ if and only if $\ell$ is linear or takes the form of an absolute function.


\begin{theorem}\label{th-RegularizationEL}
	Let $\ell:\R\to\R$ be a convex function. For $p\in(1,\infty]$, suppose that $\E[|\ell(Z)|]<\infty$ for all $Z\in L^p$.
	Then the following statements are equivalent.
	%
	\begin{itemize}
		\item[(i)]  There exists $C>0$ such that for any $G_0\in\mathcal M_p(\R^n)$, $\epsilon\ge 0$ and $\mathcal D\subseteq \R^n$, it holds that
		\begin{align}\label{eq-EL1eq}
		\inf_{\bm\beta\in \mathcal D} \sup_{F\in \mathcal  B_p(G_0,\epsilon)} \mathbb{E}^F[\ell(\bb^\top \mathbf Z)]=\inf_{\bm\beta\in \mathcal D} \left\{\mathbb{E}^{G_0}[\ell(\bb^\top \mathbf Z)] + C\epsilon \Vert\bm\beta\Vert_*\right\}.
		\end{align}
		
		\item[(ii)] The function $\ell$ takes one of the following two forms, multiplied by $C$:
		
		\begin{itemize}
		\item[(a)]  $\ell_1(x)=  x+b$ or $\ell_1(x)=- x+b$ with some $b\in\R$;
		
		\item [(b)]$\ell_2(x)= |x-b_1|+b_2$  with some $b_1,b_2\in \R$.
		\end{itemize}
	\end{itemize}
\end{theorem}

This result, which holds for any type-$p$ Wasserstein ball, is somewhat surprising as the Wasserstein DRO model is equivalent to the same regularized model, regardless of the order $p$. This equivalence occurs when the slope of the loss function $\ell$ takes values only from a constant and its negative. It further indicates that no regularized model in the form of \eqref{eq-main21} can be derived for any other loss function $\ell$. In other words, if there is an equivalence between the Wasserstein DRO model \eqref{eq-main-ef} and a regularized model for another loss function, the regularized model must take a different form from \eqref{eq-main21}. Before discussing other forms of regularization, it is worth noting the following application of this result in regression.

\begin{example} {\bf (Regression)}
	~
 
{\bf - (Least absolute deviation (LAD) regression)}
	Applying $\ell_2(x)=C|x-b_1|+b_2$ in Theorem \ref{th-RegularizationEL} and setting $b_1=0$, $b_2=0$, and $C=1$, we arrive at the distributionally robust counterpart of the least absolute deviation regression, i.e.,
	$ 
		\inf_{\bb_r\in \mathcal D}\sup_{F\in {\cal B}_{p}(G_0,\epsilon)}\E^F[|(1, - \bb_r)^{\top}\mathbf X|].
	$ 
	It is equivalent to
	$ 
		\inf_{\bm\beta_r\in \mathcal D} \{\E^{G_0}[|(1, - \bb_r)^{\top}\mathbf X|] + \epsilon \Vert (1, -\bb_r)\Vert_* \} 
	$ 
	for any $p \geq 1$.
\end{example}

We now explore whether an alternative form of regularization on $\bb$ exists that is equivalent to the Wasserstein DRO model \eqref{eq-main-ef}. Specifically, we focus on cases where the loss function $\ell$ is not Lipschitz continuous, such as higher-order loss functions arising from the 
examples presented in Table \ref{tab-R}. It is known that the worst case expectation problem based on the type-1 Wasserstein ball can become unbounded when the loss function is not Lipschitz continuous. To address cases beyond Lipschitz continuity, we consider the following formulation of  
   Wasserstein DRO model \eqref{eq-main-ef}:
\begin{align}\label{eq-main12}
	\inf_{\bb\in \mathcal D} \sup_{F\in \mathcal  B_p(G_0,\epsilon)} \mathbb{E}^F[\ell^p(\bb^\top \mathbf Z)],
\end{align}
where $\ell^p$ denotes the function $\ell$ raised to the power of $p > 1$, matching the order of the type-$p$ Wasserstein ball. Below, we present an alternative equivalency relationship and specify the exact cases where it holds.

\begin{theorem}\label{th-highorderloss} \label{th-necessity-highorderloss}
	
	
	
	
		Let $\ell:\R\to\R_+$ be a  Lipchitz continuous and convex function. For any  $p\in(1,\infty)$, the following statements are equivalent.
		\begin{itemize}
				\item[(i)] There exists $C>0$ such that for any $G_0\in\mathcal M_p(\R^n)$, $\epsilon\ge 0$ and $\mathcal D\subseteq \R^n$,  it holds that
				\begin{align}\label{eq-EL2eq}
				\inf_{\bb\in \mathcal D} \sup_{F\in \mathcal  B_p(G_0,\epsilon)} \mathbb{E}^F[\ell^p(\bb^\top \mathbf Z)]=\inf_{\bb\in \mathcal D} \left(\left(\mathbb{E}^{G_0}[\ell^p(\bb^\top \mathbf Z)]\right)^{1/p} + C\epsilon||\bb||_*\right)^p.
				\end{align}
				
				\item[(ii)] The function $\ell$ takes one of the following four forms,  multiplied by $C$:
				\begin{itemize}
			\item [(a)]$\ell_1(x)=(x-b)_+$ with some $b\in\R$;
	
	\item [(b)]$\ell_2(x)=(x-b)_-$ with some $b\in\R$;
	
	\item [(c)]$\ell_3(x)=(|x-b_1|-b_2)_+$ with some $b_1\in\R$ and $b_2\ge 0$;
	
	\item [(d)]$\ell_4(x)=|x-b_1|+b_2$ with some $b_1\in\R$ and $b_2>0$.
			\end{itemize}
				
				
				
			\end{itemize}
\end{theorem}

This result is also an ``impossibility" theorem, providing a definitive answer to whether the equivalence relationship can hold more broadly for other loss functions. It settles any attempts to establish this equivalence for other convex Lipschitz continuous functions $\ell: \R \rightarrow \R_+$.  This theorem is fundamentally important to the study of Wasserstein DRO, especially given ongoing efforts to understand its relationship with a classical regularization perspective. It shows precisely the limits of how far this perspective can be extended.

\begin{remark} \label{re-app}
As previously discussed, Wasserstein DRO and empirical risk regularization are two widely used and competing approaches for data-driven decision-making in ML and OR/MS. Wasserstein DRO is valued for its universality, offering simple generalization guarantees under minimal assumptions, such as not relying on notions of hypothesis complexity (\cite{SKE19}), while empirical risk regularization is predominantly embraced by practitioners as a simple yet powerful heuristic, despite its desirable theoretical properties (\cite{SKE19,AML12}). While we have expanded the universality of Wasserstein DRO by establishing generalization guarantees across diverse decision criteria, free from the curse of dimensionality, a natural managerial question arises: can simpler, more intuitive heuristics, such as empirical risk regularization, achieve the same generalization guarantees? Our equivalency results provide a positive answer for many applications by specifying the exact form of regularization required. Conversely, our impossibility theorem identifies when Wasserstein DRO becomes essential for decision problems beyond these cases. To underscore the importance of our regularization results for OR/MS applications, we present below an extensive series of examples, including those highlighted in Table \ref{tab-R} and the celebrated rank-dependent expected utility (RDEU) model.
\end{remark}

\begin{remark}\label{re-SADK23}
It is worth noting that Proposition 3.9 of \cite{SADK23} provided a regularized version of   \eqref{eq-main-ef} and \eqref{eq-main12} based on a duality scheme. They showed that when  $\ell$  is proper, upper-semicontinuous and satisfies $\ell(x)\le C(1+|x|^p),$ $x\in\R$ for some $C>0$, then 
\begin{align} \label{Proposition39-D}
\inf_{\bb\in \mathcal D}\sup_{F\in\mathcal B_p(G_0,\varepsilon)} \E^F[\ell(\bm\beta^\top \mathbf Z)]=\inf_{\bb\in \mathcal D, \lambda\ge 0} \left\{\E^{G_0}[\ell_p(\bm\beta^\top \mathbf Z,\lambda)]+\lambda \varepsilon^p\|\bm\beta\|_*^p\right\},
\end{align}
where $\ell_p(x,\lambda)=\sup_{y\in\R} \{\ell(y)-\lambda |y-x|^p\}.$ 
While this formula can be applied to general loss functions and used to verify the implication (ii) $\Rightarrow$ (i) in both Theorems \ref{th-RegularizationEL} and \ref{th-necessity-highorderloss}, it offers limited insight into its connection with the empirical nominal problem and typically requires additional analysis to establish a regularized empirical risk minimization formulation. In contrast, our Theorems \ref{th-RegularizationEL} and \ref{th-necessity-highorderloss} directly leverage the structure of regularized empirical risk minimization to clearly  identify the necessary and sufficient conditions for its existence, where the necessary ``impossibility" would be difficult to derive through general duality arguments alone.

We  detail in Appendix \ref{app:EU} that the regularized models \eqref{eq-EL1eq} and \eqref{eq-EL2eq} can be established based on \eqref{Proposition39-D} when $\ell$ is given in  (ii) of Theorems \ref{th-RegularizationEL} and \ref{th-necessity-highorderloss}, respectively.
In particular, for the $\ell$  in  (ii) of Theorem \ref{th-RegularizationEL}, although the regularization in \eqref{Proposition39-D} appears to depend on $p$, 
we demonstrate that this regularization can, in fact, be reduced to a formulation independent of $p$ (Theorem \ref{th-RegularizationEL}).  
As will be shown, the calculation is lengthy, making the derivation from \eqref{Proposition39-D} somewhat cumbersome.

\end{remark}

\begin{remark}
Our regularization results offer deeper and more practically significant insights than Proposition 3.9 in \cite{SADK23}, particularly regarding the required regularization term. While Proposition 3.9 suggests the need for a term of the form $\lambda \epsilon^p\|\bm\beta\|_*^p$ in \eqref{Proposition39-D}, with the regularization raised to the power $p$, this can create computational challenges, especially as $p\rightarrow \infty$. In contrast, our regularization result reveals that the regularization term is actually independent of $p$, namely $C\epsilon\|\bm\beta\|_*$. This insight significantly enhances the computational tractability of solving the regularization problem. Moreover, the regularized formulations \eqref{eq-EL2eq} in Theorem \ref{th-necessity-highorderloss}  are all convex programs in $\bm\beta$, with complexity comparable to nominal problems. 
 In contrast, the right-hand side of (17) is nonconvex in $\bm\beta$ and $\lambda$,
and its objective function $\ell_p(x, \lambda)$ is more costly to evaluate. Thus, Theorem \ref{th-necessity-highorderloss} can also be viewed as enabling a convex program solution to \eqref{Proposition39-D}. Finally, the equivalency between Wasserstein DRO and max-sliced Wasserstein DRO established in the previous section implies that max-sliced Wasserstein DRO can likewise be solved efficiently via convex programs for the cases covered in Theorem \ref{th-necessity-highorderloss}.
\end{remark}

\begin{example} {\bf (Classification)} \label{ex1}
	~
	
{\bf - (Higher-order hinge loss)}
	Applying $\ell_2(x)=(x-b)_-$ and setting $b=1$, we have that the following classification problem with a higher-order hinge loss
	$ 
		\inf_{\bm\beta\in \mathcal D}\sup_{F\in \overline{\cal B}_{p}(F_0,\epsilon)} \E^F[(1-Y\cdot \bm \beta^{\top}\mathbf X)_+^p]
	 $ 
	is equivalent to the regularization problem
	$ 
		\inf_{\bm\beta\in \mathcal D}  ( (\E^{F_0}[(1-Y\cdot \bm \beta^{\top}\mathbf X)_+^p] )^{1/p}+ \varepsilon \|\bm \beta\|_*  )^p.
	$ 
	
	{\bf - (Higher-order SVM)}
	Applying $\ell_3(x)=(|x-b_1|-b_2)_+$ and setting $b_1=1$ and $b_2=0$, we have that the higher-order SVM classification problem
 $ 
		\inf_{\bm\beta\in \mathcal D}\sup_{F\in \overline{\cal B}_{p}(F_0,\epsilon)} \E^F[|1-Y\cdot \bm \beta^{\top}\mathbf X|^p]
	 $ 
	is equivalent to the regularization problem
	 $ 
		\inf_{\bm\beta\in \mathcal D} ( (\E^{F_0}[|1-Y\cdot \bm \beta^{\top}\mathbf X|^p])^{1/p}+ \varepsilon \|\bm \beta\|_*  )^p.
	 $ 
\end{example}

\begin{example} {\bf (Regression)} \label{ex2}
	~
	
	{\bf - (Higher-order measure)}
	Applying $\ell_3(x)=(|x-b_1|-b_2)_+$ and setting $b_1=0$ and $b_2=0$, we have that the regression with a higher-order measure
	 $ 
		\inf_{\bm\beta_r\in \mathcal D}\sup_{F\in {\cal B}_{p}(G_0,\epsilon)} \E^F[|(1, - \bb_r)^{\top}\mathbf X|^p],
	 $ 
	is equivalent to the regularization problem
 $ 
		\inf_{\bm\beta_r\in \mathcal D}  ( \left(\E^{G_0}[|(1, - \bb_r)^{\top}\mathbf X|^p]\right)^{1/p}+ \varepsilon \|(1,-\bm \beta_r)\|_*  )^p.
	 $ 
	
	{\bf - (higher-order $c$-insensitive measure)}
	Applying $\ell_3(x)=(|x-b_1|-b_2)_+$ and setting $b_1=0$ and $b_2=c$, we have that the following regression problem with a higher-order $c$-insensitive measure
	 $ 
		\inf_{\bm\beta_r\in \mathcal D}\sup_{F\in {\cal B}_{p}(G_0,\epsilon)} \E^F[(|(1, - \bb_r)^{\top}\mathbf X|-c)_+^p]
 $ 
	is equivalent to the regularization problem
 $ 
		\inf_{\bm\beta_r\in \mathcal D} ( (\E^{G_0}[(|(1, - \bb_r)^{\top}\mathbf X|-c)_+^p] )^{1/p}+ \varepsilon \|(1,-\bm \beta_r)\|_*  )^p.
 $ 
\end{example}

\begin{example} {\bf (Risk minimization)} \label{ex3}
	~
	
	{\bf - (Lower partial moments)}
	Applying $\ell_1(x)=(x-b)_+$ and setting $b=c$, we have that the risk minimization with lower partial moments
	 $ 
		\inf_{\bm\beta\in \mathcal D}\sup_{F\in {\cal B}_{p}(G_0,\epsilon)} \E^F[(\bm\beta^{\top}\mathbf X-c)_+^p]
 $ 
	is equivalent to the regularization problem
	 $ 
		\inf_{\bm\beta\in \mathcal D} (  (\E^{G_0}[(\bm\beta^{\top}\mathbf X-c)_+^p] )^{1/p}+ \varepsilon \|\bm\beta\|_*  )^p.
 $ 
\end{example}



Even though Theorem \ref{th-necessity-highorderloss} highlights the impossibility of establishing an equivalence relation for a more general loss function $\ell$, Theorem \ref{th-highorderloss} can still be applied more broadly as a powerful foundation to derive alternative equivalence relations for a richer family of measures. Specifically, there is a large family of measures that can be generally expressed in the following two forms:
\begin{align}\label{eq-infEU}
	\mathcal V^F(Z) = \inf_{t \in \R} \E^F[\ell^p(Z,t)]~~~
	{\rm and}~~~
	\rho^F(Z) = \inf_{t \in \R} \left\{t+\left(\E^F[\ell^p(Z,t)]\right)^{1/p}\right\}
\end{align}
for some loss function $\ell$. 

We show in the appendix (see Lemma \ref{lm-minmaxeq}) that for a wide range of loss functions $\ell$, the following switching of $\sup$ and $\inf$ is valid
$$\sup_{F\in \mathcal  B_p(G_0,\epsilon)} \inf_{t \in \R} \pi_{i,\ell}(F,t)=
\inf_{t \in \R} \sup_{F\in \mathcal  B_p(G_0,\epsilon)}\pi_{i,\ell}(F,t),~~i=1,2,$$
where
$$
\pi_{1,\ell}(F,t)=\E^F[\ell^p(\bb^\top \mathbf Z,t)] ~~{\rm and}~~\pi_{2,\ell}(F,t)=t+\left(\E^F[\ell^p(\bb^\top \mathbf Z,t)]\right)^{1/p} .
$$
This, combined with Theorem \ref{th-highorderloss}, leads to the following.

	
	\begin{corollary} \label{cor1} For any $p\in [1,\infty)$ and $c>0$, let $\mathcal V$ and  $\rho$ be defined by \eqref{eq-infEU} with   $\ell(z,t)=c\ell(z-t)$. Take $G_0\in\mathcal M_p(\R^n)$ and $\epsilon\ge 0$.
 
		(i)   If $\ell$ is  $\ell_3$ or $\ell_4$ in Theorem \ref{th-highorderloss}, 
		then
		$$
		\inf_{\bb\in\mathcal D}\sup_{F\in \mathcal  B_p(G_0,\epsilon)}  	\mathcal V^F (\bb^\top \mathbf Z)= \inf_{\bb\in\mathcal D}\left( \left(  	\mathcal V^{G_0}(\bb^\top \mathbf Z) \right)^{1/p} +  c\epsilon\|\bb\|_* \right) ^p.
		$$
	 {If $\ell$ is  $\ell_1$ or $\ell_2$ in Theorem \ref{th-highorderloss}, then 
  $\mathcal V^F (\bb^\top \mathbf Z)=0$ for any $\bm\beta\in\R^n$ and $F\in\mathcal M_p(\R^n)$.}
		
		(ii)  
  If $c>1$ and $\ell$ is one of  $\ell_1,\ell_3,\ell_4$ in Theorem \ref{th-highorderloss} or $\ell(z,t)=c(|z|-t)_+$,  
		then it holds that 
		$$
		\inf_{\bb\in\mathcal D}\sup_{F\in \mathcal  B_p(G_0,\epsilon)} \rho^F (\bb^\top \mathbf Z)=   \inf_{\bb\in\mathcal D}\left\{\rho^{G_0}(\bb^\top \mathbf Z) +  c \epsilon\|\bb\|_*\right\}.
		$$
   {If  $\ell=\ell_2$ in Theorem \ref{th-highorderloss}, then 
  $\rho^F (\bb^\top \mathbf Z)=-\infty$ for any $\bm\beta\in\R^n$ and $F\in\mathcal M_p(\R^n)$.}
		\end{corollary}
  
Corollary \ref{cor1} can accommodate cases where variance is used as the measure. Variance is not listed as an example in Section \ref{rcr} because it has already been studied in \cite{BCZ22}. However, our method for deriving the equivalency relation is different from, and significantly more general than, the approach in \cite{BCZ22}. The following example illustrates the broader applicability of our approach.
	
	\begin{example}  {(\cite{BCZ22})} \label{ex7}
		When $p=2$ and $\ell(z,t) = |z-t|$ (i.e., $\ell_3(z-t)$ with $b_1=0$ and  $b_2=0$), 
		we have $\mathcal V^F = {\rm\mathbb{V}ar}^F$, where ${\rm\mathbb{V}ar}$ represents the variance. That is,
		 $ 
            {\rm\mathbb{V}ar}^F(\bm\beta^\top\mathbf X)=
            \inf_{t\in\R} \E^F[(\bm\beta^\top\mathbf X-t)^2].
		 $ 
		Applying Corollary \ref{cor1} (i) yields
		$ 
			\sup_{F\in \mathcal B_{2}(G_0,\epsilon)}{\rm\mathbb{V}ar}^F(\bm\beta^\top\mathbf X)
			 = ( ({\rm\mathbb{V}ar}^{G_0}(\bm\beta^\top\mathbf X))^{1/2}
			+\epsilon\|\bm\beta\|_* )^2.
		$
	\end{example}
	
The following example demonstrates how Corollary \ref{cor1} can be applied to derive the regularization counterpart for higher moment risk measures in risk minimization, a previously unknown result. 
	
	\begin{example} {\bf (Risk minimization)} \label{ex6}
		For $c>1$, setting $\ell(z,t)=c(z-t)_+$ (i.e., $c\ell_1(z-t)$ with $b=0$),  we have that the following problem of minimizing higher moment risk measures
		$  
			\inf_{\bm\beta\in  \mathcal D}\sup_{F\in {\cal B}_{p}(G_0,\epsilon)}
			\inf_{t\in\R} \{t+c (\E^F[(\bm\beta^{\top}\mathbf X-t)_+^p] )^{1/p} \}
		$ 
		is equivalent to the regularization problem
		$ 
			\inf_{\bm\beta\in  \mathcal D, t\in\R} \{t+c (\E^{G_0}[(\bm\beta^{\top}\mathbf X-t)_+^p] )^{1/p}
			+ c\varepsilon   \|\bm\beta\|_* \}.
		 $ 
	\end{example}
	

\vspace{0.1in}	
\subsection{The case of non-expected function with distortion}\label{subsec:DT}
A powerful perspective in the risk measure literature, especially when studying non-expected functions such as {\color{blue}Conditional Value-at-Risk (CVaR)}, is to view these functions as equivalent to taking an expectation with respect to a distorted probability distribution. The concept of distorted expectation is central to foundational theories, including Yaari's dual utility theory (\cite{Y87}), Choquet Expected Utility (\cite{S89}), and various business applications. In this section, we adopt this perspective to study the problem \eqref{eq-main-ef} with a distorted expectation:
\begin{align} \label{distorted}
\inf_{\bb \in \mathcal D} \sup_{F\in \mathcal  B_p(G_0,\epsilon)} \rho_h^F(\ell(\bm\beta^\top \mathbf Z)),
\end{align}
where 
\begin{align*}
		\rho_h^F(Z)=\int_0^1 F^{-1}(s)\d h(s)
\end{align*}
is an ``expectation" taken with respect to a distortion function $h:[0,1]\to \R$ \footnote{Strictly speaking, a distortion function $h$ should additionally satisfy $h(0)=1-h(1)=0$ and be increasing, so that $\rho_h^F$ can be considered as distorted expectation. However, our results do not hinge on these requirements.}.  Here, $F^{-1}$ represents the left-quantile function of $F$, i.e., $F^{-1}(s)=\inf\{x: F(x)\ge s\}$ for $s\in(0,1]$, and $F^{-1}(0)=\inf\{x: F(x)>0\}$. The problem \eqref{distorted} encompasses \eqref{eq-main-ef} as a special case, since $\rho_h^F(Z) = \E^F[Z]$ when $h(s)=s$, $s\in [0,1]$, and accommodates CVaR, since $\rho_h^F(Z) = {\rm CVaR}_\alpha^F(Z)$ when $h(s)= (s-\alpha)_+/(1-\alpha)$, $s\in [0,1]$. A comprehensive list of risk measures with distorted expectation representations can be found in \cite{WWW20} and \cite{CLM23}.

It is intriguing to consider how the problem \eqref{distorted} might be solved, particularly whether a regularization counterpart, similar to that of the expected function \eqref{eq-main-ef}, remains available. Our findings are surprisingly general: for any problem \eqref{distorted} with an increasing and convex distortion function $h$, a regularization counterpart always exists for the type-1 Wasserstein ball. This similarity to the classical regularization result for \eqref{eq-main-ef} is particularly notable, given the broader applicability of \eqref{distorted}. This insight even extends to the general type-$p$ Wasserstein ball. As we show below, the necessary and sufficient conditions for a regularization counterpart to exist are remarkably similar for both \eqref{distorted} and \eqref{eq-main-ef} in the general type-$p$ case.

Before presenting the result, we introduce the notation that $\|g\|_a=(\int_{A} |g(x)|^a\d x)^{1/a}$ for a function $g:A\to\R$ with $A\subseteq \R$ and $a\in[1,\infty)$, and $\|g\|_\infty=\sup_{x\in A} |g(x)|$.
For a convex distortion function $h:[0,1]\to\R$ with $\lim_{x\to1-}h(x)=h(1)$, we denote by $h'_-$ the left-derivative of $h$ on $(0,1]$, noting that $h'_-(1)$ may be infinity.

\begin{theorem}\label{th-DTp=1} \label{th-uniqueDT}
For $p\in [1,\infty]$, let $h:[0,1]\to\R$ be an increasing and convex distortion function satisfying
$\lim_{x\to 1-}h(x)=h(1)$ and
$\|h'_-\|_q\in(0,\infty)$,
and let $\ell:\R\to\R$ be a convex function.

For $p=1$, if
 ${\rm Lip}(\ell)<\infty$,
then for any $G_0\in\mathcal M_1(\R^n)$, $\epsilon\ge 0$ and $\mathcal D\subseteq \R^n$, we have
\begin{align*}
\inf_{\bm\beta\in \mathcal D}\sup_{F\in \mathcal B_1(G_0,\epsilon)}\rho_h^F(\ell(\bm\beta^\top \mathbf Z))=\inf_{\bm\beta\in\mathcal D}\left\{\rho_h^{G_0}(\ell(\bm\beta^\top\mathbf Z))+{\rm Lip}(\ell)\|h'_-\|_\infty\epsilon\|\bm\beta\|_*\right\}.
\end{align*}

For $p\in(1,\infty]$, if $\rho_h(|\ell(Z)|)<\infty$ for all $Z\in L^p$, then the following statements are equivalent.
	\begin{itemize}
		\item[(i)] For any $G_0\in\mathcal M_p(\R^n)$, $\epsilon\ge 0$ and $\mathcal D\subseteq \R^n$,  there exists $C>0$ such that
		$$
		\inf_{\bm\beta\in \mathcal D} \sup_{F\in \mathcal  B_p(G_0,\epsilon)} \rho_h^F(\ell(\bb^\top \mathbf Z))=\inf_{\bm\beta\in \mathcal D} \left\{\rho_h^{G_0}(\ell(\bb^\top \mathbf Z)) + C\|h'_-\|_q\epsilon \Vert\bm\beta\Vert_*\right\}.
		$$
		
		\item[(ii)] The function $\ell$ takes one of the following two forms, multiplied by $C$:
		
		\begin{itemize}
		\item[(a)]  $\ell_1(x)=  x+b$ or $\ell_1(x)=-  x+b$ with some $b\in\R$;
		
		\item [(b)]$\ell_2(x)= |x-b_1|+b_2$  with some $b_1,b_2\in \R$.
		\end{itemize}
	\end{itemize}
\end{theorem}

Interestingly, the ``exact" class of loss functions $\ell$ that admit a regularization counterpart is the same as in the previous section, despite distorted expectations being significantly more general than expected functions. This result is particularly significant from a decision-theoretical perspective, as the formulation \eqref{distorted} closely resembles the celebrated rank-dependent expected utility (RDEU) model, known for resolving the paradox in expected utility (\cite{Q82}). Given its importance, we highlight below the application of Theorem \ref{th-DTp=1} in RDEU.

\begin{example} {\bf (Rank-Dependent Expected Utility (RDEU))} \label{RDEU}
The decision criteria of RDEU admits the form $V_{u,h} (Z) = \rho_h(u(Z))$, where $\rho_h$ is a distortion function and $u: \R\to\R$ is a dis-utility function. For decision makers who are risk-averse (\cite{CKS87}), i.e.,  
$h$ and $u$ are increasing convex with ${\rm Lip}(u)\in (0,\infty)$, 
we have
\begin{align*}
\inf_{\bm\beta\in \mathcal D}\sup_{F\in \mathcal B_1(G_0,\epsilon)}V_{u,h}^F(\ell(\bm\beta^\top \mathbf X))=\inf_{\bm\beta\in\mathcal D}\left\{V_{u,h}^{G_0}(\ell(\bm\beta^\top\mathbf X))+ {\rm Lip}(u)\cdot{\rm Lip}(\ell)\|h'_-\|_\infty\epsilon\|\bm\beta\|_*\right\},
\end{align*}
for any convex loss function $\ell$ with ${\rm Lip}(\ell)\in(0,\infty)$, $G_0\in\mathcal M_1(\R^n)$, $\epsilon\ge 0$ and $\mathcal D\subseteq \R^n$.
\end{example}


We further demonstrate the practical applicability of the above theorem by showing how it can be directly applied to derive the regularization counterpart for the $\nu$-support vector regression example.

\begin{example} {\bf (Regression)} \label{ex-dis-reg}
	~
	
{\bf - ($\nu$-support vector regression)}
For $\alpha\in [0,1)$, let $h(s)=(s-\alpha)_+/(1-\alpha)$, $s\in [0,1]$.  Applying  $\ell_2(x)=C|x-b_1|+b_2$ and setting $C=1$ and $b_1=b_2=0$,
we have from Theorem \ref{th-DTp=1} that
the $\nu$-support vector regression
	$ 
		\inf_{\bm\beta_r\in \mathcal D}\sup_{F\in {\cal B}_{p}(G_0,\epsilon)} {\rm CVaR}_\alpha^F(|(1, - \bb_r)^{\top}\mathbf X|)
	$ 
is equivalent to the regularization problem
	$ 
		\inf_{\bm\beta_r\in \mathcal D}  \{{\rm CVaR}_\alpha^{G_0}|((1, - \bb_r)^{\top}\mathbf X|)+ \varepsilon (1-\alpha)^{-1/p} \|(1,-\bm \beta_r)\|_* \}.
	$ 
\end{example}

Measures like the CVaR-Deviation example in risk minimization, as shown below, require a convex distortion function $h$, which is not covered by the theorem above. We now show that when the loss function $\ell$ in \eqref{distorted} is linear, a regularization counterpart can also be found for any convex distortion function $h$.

\begin{proposition}\label{prop:disabsolute} 
  For $p\in[1,\infty]$,  let $h: [0,1]\to\R$ be a convex function with $\lim_{x\to0+}h(x)=h(0)$ and $\lim_{x\to 1-}h(x)=h(1)$.
		Then for any $G_0\in\mathcal M_p(\R^n)$, $\epsilon\ge 0$ and $\mathcal D\subseteq \R^n$, we have
\begin{align*}
\inf_{\bm\beta\in \mathcal D}\sup_{F\in \mathcal B_p(G_0,\epsilon)}\rho_h^F(\bm\beta^\top \mathbf Z)=\inf_{\bm\beta\in\mathcal D}\left\{\rho_h^{G_0}(\bm\beta^\top\mathbf Z)+\|h'_-\|_q\epsilon\|\bm\beta\|_*\right\}.
\end{align*}
		
	\end{proposition}
 

Note first that applying Proposition \ref{prop:disabsolute} to classification problems is fairly straightforward, namely yielding the following equivalency:
\begin{align*}
\inf_{\bm\beta\in \mathcal D}\sup_{F\in \overline{\mathcal B}_p(F_0,\epsilon)}\rho_h^F(-Y\cdot\bm\beta^\top \mathbf X)=\inf_{\bm\beta\in\mathcal D}\left\{\rho_h^{F_0}(-Y\cdot \bm\beta^\top\mathbf X)+\|h'_-\|_q\epsilon\|\bm\beta\|_*\right\}.
\end{align*}
This equivalency is readily applicable to specific cases, such as the $\nu$-support vector machine.

\begin{example} {\bf (Classification)} \label{ex-dis-cf}

{\bf - ($\nu$-support vector machine)}
Setting $h(s)=(s-\alpha)_+/(1-\alpha)$, $s\in [0,1]$ with $\alpha\in [0,1),$
we have that the classification problem with $\nu$-support vector machine
$ 
		\inf_{\bm\beta\in \mathcal D}\sup_{F\in \overline{\cal B}_{p}(F_0,\epsilon)} {\rm CVaR}_\alpha^F(-Y\cdot\bm\beta^{\top}\mathbf X)
	$ 
	is equivalent to the regularization problem
$ 
		\inf_{\bm\beta\in \mathcal D}  \{{\rm CVaR}_\alpha^{F_0}(-Y\cdot\bm\beta^{\top}\mathbf X)+\varepsilon (1-\alpha)^{-1/p} \|\bb\|_* \}.
$ 
\end{example}

Proposition \ref{prop:disabsolute} can be further applied to general deviation measures defined by a distorted expectation, such as the CVaR-Deviation example in risk minimization. Observe that $\rho_h(Z - \E[Z]) = \rho_{\widetilde h}(Z)$ holds for any distortion function $h$, where ${\widetilde h}(s) := h(s) + (h(0) - h(1))s$ for $s \in [0,1]$. Moreover, ${\widetilde h}$ is convex whenever $h$ is convex. Thus, we apply Proposition \ref{prop:disabsolute} and arrive at the following equivalency:
\begin{align*}
\inf_{\bm\beta\in \mathcal D}\sup_{F\in \mathcal B_p(G_0,\epsilon)}\rho_h^F(\bb^\top\mathbf Z-\E^F[\bb^\top\mathbf Z])=
	\inf_{\bm\beta\in \mathcal D} \left\{\rho^{G_0}_h(\bm\beta^{\top}\mathbf Z
				-\E^{G_0}[\bm\beta^{\top}\mathbf Z])
				+  \|h'_-+h(0)-h(1)\|_q  \varepsilon\|\bm\beta\|_*\right\}.
\end{align*}	
This leads us to our final example.
\begin{example} {\bf (Risk minimization)} \label{ex-dis-rm}

{\bf - (CVaR-Deviation)}
Setting $h(s)=(s-\alpha)_+/(1-\alpha)$, $s\in [0,1]$ with $\alpha\in [0,1),$
we have that the risk minimization with CVaR-Deviation
	$ 
		\inf_{\bm\beta\in \mathcal D}\sup_{F\in {\cal B}_{p}(G_0,\epsilon)} {\rm CVaR}_\alpha^F(\bm\beta^{\top}\mathbf X-\E^F[\bm\beta^{\top}\mathbf X])
	$ 
	is equivalent to the regularization problem
	$ 
		\inf_{\bm\beta\in \mathcal D}  \{{\rm CVaR}_\alpha^{G_0}(\bm\beta^{\top}\mathbf X-\E^{G_0}[\bm\beta^{\top}\mathbf X])+\varepsilon  (\alpha+\alpha^q(1-\alpha)^{1-q} )^{1/q} \|\bb\|_* \}.
$ 
\end{example}

\begin{remark} We offer some comments here on the recent work by \cite{CLT24}, which was released well after our work had been made available online. While \cite{CLT24} initially set out to establish an equivalency relationship for a general setting of Wasserstein DRO -- considering a general loss function $\ell$ and cost function $d$ as defined in \eqref{eq:dwasser} -- their actual results do not extend much beyond our settings and findings in Section \ref{subsec:EU} (expected functions and extensions) and are limited to discussing sufficient conditions for the equivalency. Although some examples presented in Table 2 of \cite{CLT24} are not covered in Section \ref{subsec:EU}, we attribute this largely to our focus on cases that admit a final convex optimization formulation, rather than a limitation of our analysis. Indeed, we believe our analysis is foundational, reaching the very core of the equivalency problem. This is evident in our ability to prove not only sufficient conditions but also necessary ones for both the case of expected functions and general distorted expectations -- an achievement that appears beyond the reach of \cite{CLT24}, as it inevitably hinges on Lemma \ref{lm-onehigheq} provided in our analysis. To substantiate our claim that our analysis is not limited to the examples presented in this paper, we demonstrate in Proposition \ref{prop:noncx} in Appendix \ref{subsec:appDT} how to generalize Theorem \ref{th-DTp=1} (the case $p=1$) to non-convex loss functions $\ell$, including the truncated pinball loss from \cite{CLT24} as a special case.\footnote{In \cite{CLT24}, the center of the Wasserstein ball is assumed to be the empirical distribution, which has a finite set of vectors as its support. Consequently, Proposition \ref{prop:noncx} is directly applicable.}
As another example, it is not difficult to confirm that our findings in Proposition \ref{thm:221010-1}, Lemma \ref{lm-onehigheq}, and all the results in Section \ref{reg} can be readily extended to a general Hilbert space $\mathcal H$, with the norm induced by its inner product. This extension covers the case of nonparametric scalar-on-function linear regression $\E[\ell^p(Y - \int_0^1 \mathbf X(t)\bb(t)\d t)]$ from \cite{CLT24}, where $\ell$ takes the forms outlined in our Theorem \ref{th-highorderloss}.\footnote{
The expression $y-\int_0^1\mathbf x(t)\bb(t)\d t$ can be interpreted as the inner product between $(y,\mathbf x)$ and $(1,-\bb)$ on $\mathcal H:=\R\times \mathcal L^2([0,1])$, where $\mathcal L^2([0,1])$ is the space of all square integrable functions on $[0,1]$. In this setting, the decision rule is the inner product $\langle  (y,\mathbf x),   (1,-\bb)\rangle$ for each decision vector $\bb$ which is affine  and the Wasserstein ball $\overline{\cal B}_p (F_0,\epsilon)$, similar in form to \eqref{eq-WU-multi}, is now encompassed within the set of probability measures on $\mathcal H$.
}
\end{remark}

\section{Numerical Demonstration of Generalization Bounds}
In this section, we seek to provide a further demonstration of the generalization result obtained in this paper. In particular, we numerically illustrate the decay rate of the Wasserstein radius for measures $\rho$ that are non-expected functions.
We closely follow the experimental setup presented in \cite{SKE19}, where the authors investigate the scaling behaviour of the smallest Wasserstein radius for the synthetic \texttt{threenorm}
classification problem (\cite{B96}). While \cite{SKE19} uses classical support vector machine for classification, in which the measure $\rho$ is an expected function defined by a hinge loss function $\ell$, we employ $\nu$-support vector machine, where $\rho$ is represented by CVaR. Specifically, we consider the distributionally robust counterpart of the standard $\nu$-support vector machine formulation, 
\begin{align}\label{eq-NRnu}
\min_{\bm\beta\in\R^n}~~{\rm CVaR}_{\alpha}^{\widehat{F}_N}(-Y\cdot \bm\beta^\top \mathbf X)+\frac{1}{2}\|\bb\|_2^2.
\end{align}
By choosing $\infty$-norm as the norm $\|\cdot\|$ on the input space in the transport cost and setting $p=1$, like \cite{SKE19}, and following Proposition \ref{prop:disabsolute} or
Example \ref{ex-dis-cf},
we have the following regularized form for the Wasserstein robust $\nu$-support vector machine problem:
\begin{align}\label{eq-WRnu}
	\min_{\bm\beta\in\R^n}~~{\rm CVaR}_{\alpha}^{\widehat{F}_N}(-Y\cdot \bm\beta^\top \mathbf X)+\frac{1}{2}\|\bb\|_2^2+\frac{\epsilon \|\bb\|_1}{1-\alpha}.
\end{align}

The experiment is based on an output $y$ drawn uniformly from $\{-1,1\}$ and a $20$-dimensional input $\mathbf x$ from a multivariate normal distribution. Specifically, if $y=-1$, then $\mathbf x$ is drawn from standard multivariate
normal distribution shifted by $(c,\dots,c)$ or $(-c,\dots,-c)$ with equal probabilities, where $c=2/\sqrt{20}$. If $y=1$, then $\mathbf x$ is drawn from standard multivariate
normal distribution shifted by $(c,-c,\dots, c,-c)$. We consider different sizes of training samples $N$ in the set $\{10,\dots,90\}\cup\{100,\dots,900\}\cup\{1000,\dots,10000\}$ as well as $10^5$ test samples. Each size of training sample involves $100$ simulation trials. 

Throughout the experiment, the search space of radius is chosen as $\mathcal E=\{a\cdot 10^{-b}:~a\in \{1,\dots,10\}, ~b\in\{1,2,3\}\}\cup\{1.1,1.2,\dots,2\}$. Similar to \cite{SKE19}, we use the following three approaches to choose the smallest Wasserstein radius:
\begin{itemize}
\item \emph{Cross validation}: For a set of $N$ training samples, denoted as $\{(y_i,\mathbf x_i)\}_{i=1}^N$, we partition them into $k=5$ subsets. We use one subset as the validation dataset and combine  the remaining $k-1$ subsets as the training dataset. This results in $k$ pairs of validation and training datasets. We choose the radius in $\mathcal E$ such that the average validation error over these $k$ pairs is the smallest. This operation is repeated across all $100$ trials, and we then report the average of the radii.

\item \emph{Optimal radius}: In each trial, we choose the radius in $\mathcal E$ that has the smallest testing error and then report the average of the radii across all $100$ trials.

\item \emph{Generalization bound}: We choose the smallest radius such that the optimal value of \eqref{eq-WRnu} exceeds the value of the nominal problem on the test samples in at least $95\%$ of all 100 trials.
\end{itemize}


In our initial experiments, we set $\alpha=0.2$. Figure \ref{fig-nuSVM} displays the selected Wasserstein radii in relation to various sample sizes $N$ for the three approaches above. Note that while the first two approaches determine the radius by averaging the radii induced by all simulation trials (potentially leading to a radius not necessarily in the set $\mathcal E$), the third approach specifically selects a radius from $\mathcal E$ based on the percentage criteria, utilizing 
the testing samples across all 100 trials.  One can observe that the Wasserstein radii of all three approaches decay approximately as $1/\sqrt{N}$, which is consistent with the theoretical generalization bound of Theorem \ref{th-FSG}.

\begin{figure}[pht]
\caption{Wasserstein radius versus the number of training samples, with $\alpha=0.2$}
	\begin{center}
		\includegraphics[width=0.45\textwidth]{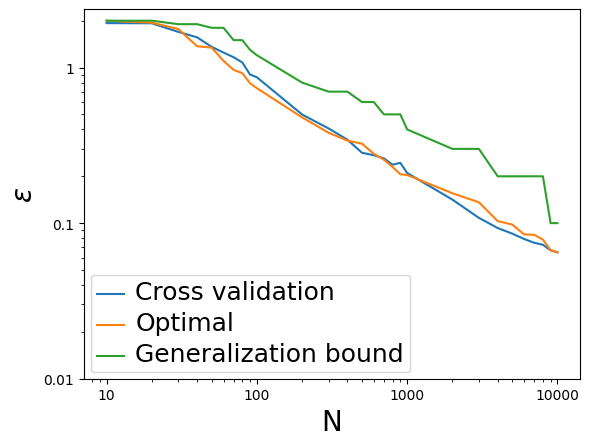}
		\label{fig-nuSVM}
		\medskip
	\end{center}
\end{figure}
It is natural to wonder if varying choices of $\alpha$ could affect the scaling behavior. This becomes particularly intriguing when considering higher values of $\alpha$. However, literature on the $\nu$-support vector machine suggests that the optimal solution of \eqref{eq-NRnu} can become degenerate (i.e., equal to zero) when $\alpha$ surpasses a specific threshold. We indeed observe that the experiment setup of \cite{SKE19} only allows us to generate meaningful solutions for lower values of $\alpha$. To address the potential degeneracy issue at higher $\alpha$ values, we modified the method for generating simulated data. Specifically, for each sample $(y,\mathbf x)\in  \{-1,1\}\times\R^{20}$, if $y=-1$, then $\mathbf x$ is drawn from standard multivariate normal distribution shifted by $(-c,\dots,-c)$ or $(-2c,\dots,-2c)$ with equal probabilities, where $c=1/3$, whereas if $y=1$, then $\mathbf x$ is drawn from standard multivariate
normal distribution shifted by $(c,2c,\dots, c,2c)$. Figure \ref{fig-nuSVM_alpha} contains three panels, each displaying the selected radii for a distinct level of $\alpha$: $\alpha=0.2,~0.5~{\rm or}~0.8$. Notably, the decay rate of the radius is remarkably similar across different levels of $\alpha$, approximately as $1/\sqrt{N}$. This is in line with our theoretical finding that the rate in general does not hinge on the specification of the measure $\rho$.  

\begin{figure}[pht]
	\caption{Wasserstein radius versus the number of training samples, with different $\alpha$ level (Left: $\alpha=0.2$; Middle: $\alpha=0.5$; Right: $\alpha=0.8$).}
	\begin{center}
		\includegraphics[width=0.32\textwidth]{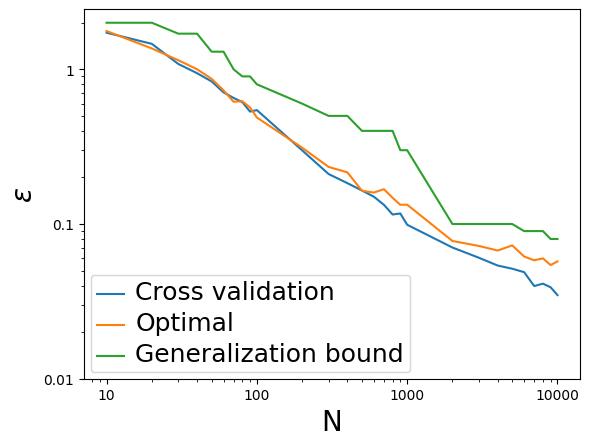}
		\includegraphics[width=0.32\textwidth]{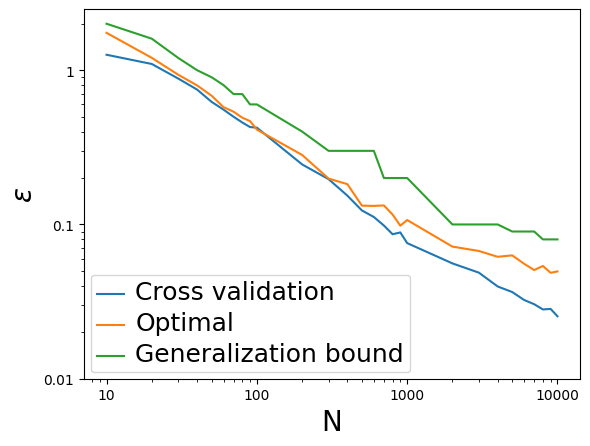}
		\includegraphics[width=0.32\textwidth]{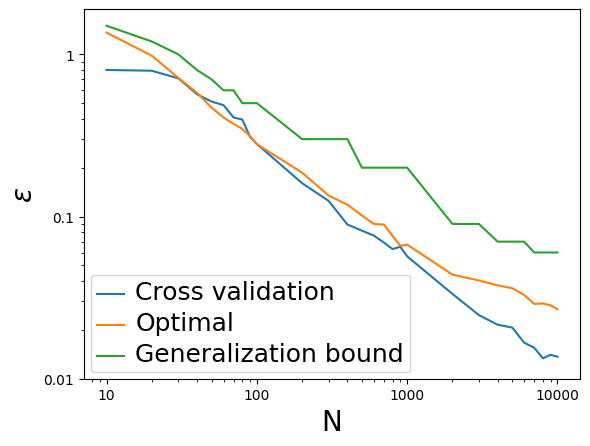}
		\label{fig-nuSVM_alpha}
		\medskip
	\end{center}
\end{figure}


\section{Beyond Affine Decision Rules} \label{beyond}
While our results thus far have focused on affine decision rules, the underlying principles of our approach to establishing generalization bounds for a broad class of measures $\rho$ can be extended to a wider range of decision rules. In this section, we illustrate this extension by deriving generalization bounds for the following Wasserstein DRO problem with a decision rule denoted by $f_{\bb}$:
\begin{align*}
	\inf_{\bb \in \mathcal D} \sup_{F\in \overline{\cal B}_p(F_0,\epsilon)} \rho^F(Y \cdot f_{\bb}(\mathbf X)),
\end{align*}
where $F_0\in \M_p(\Xi)$ with $\Xi=\{-1,1\}\times\R^n$ and $\overline{\cal B}_p(F_0,\epsilon)$ is the type-$p$ Wasserstein ball defined by \eqref{eq-WU-multi}.  
Throughout this section, we assume that $\mathcal D\subseteq \R^m$ for some $m\in\N$, with $\|\cdot\|_{\mathcal D}$ representing any given norm on $\mathcal D$.

The key to establishing generalization bounds applicable to a broad class of measures continues to lie in focusing on the projection set, now expressed in a more general form:
\begin{align}\label{eq-HONE-1}
\overline{\mathcal B}_{p|f_{\bb}}(F_0,   \epsilon)=\{F_{Y\cdot f_{\bb}(\mathbf X)}: F_{(Y,\mathbf X)}\in\overline{\mathcal B}_p(F_0,\epsilon)\}.
\end{align}
Although the exact equivalence of this set to a one-dimensional Wasserstein ball or the projection set of a max-sliced Wasserstein ball, as shown in Proposition \ref{prop-equi}, no longer holds in the general case, we can still establish generalization bounds by instead identifying a one-dimensional Wasserstein ball dominated by $\overline{\mathcal B}_{p|f_{\bb}}(F_0, \epsilon)$, i.e. 
\begin{equation} \label{incl}
\overline{\mathcal B}_{p|f_{\bb}}(F_0, \epsilon) \supseteq \mathcal B_p(F_{Y_0\cdot f_{\bb}(\mathbf X_0)},g(\epsilon)),
\end{equation}
for some increasing function $g:\R_+\to\R_+$, where $(Y_0,\mathbf X_0)\sim F_0$.  Indeed, if this dominance relationship \eqref{incl} can be established, one can then leverage the measure concentration property of a one-dimensional Wasserstein ball to determine the radius of this one-dimensional ball
$\epsilon_N$
such that
$$\p\left( F_{Y^* \cdot f_{\bb}(\mathbf X^*)} \in \overline{\mathcal B}_{p|f_{\bb}}( \widehat{F}_N,  g^{-1}(\epsilon_N) ) \right)\ge \p\left( F_{Y^* \cdot f_{\bb}(\mathbf X^*)} \in \mathcal B_p(F_{Y_0\cdot f_{\bb}(\mathbf X_0)},\epsilon_N)\right) \geq 1-\eta$$ with $(Y^*,\mathbf X^*) \sim F^*$, where $F^*$ and $\widehat{F}_N$, defined in Section \ref{sec:Gen},  are the data-generating distribution  and  empirical distribution, respectively. This leads to the confidence bound:\footnote{
One may wish to further gauge how tight (i.e., small) the Wasserstein in-sample risk, $\sup_{F \in \overline{{\cal B}}p(\widehat{F}_N, \epsilon)} \rho^{F}(Y \cdot f_{\bb}(\mathbf{X}))$, is -- e.g., relative to the empirical in-sample risk -- when the radius $\epsilon$ is sufficiently small. To this end, we note that, under Assumption \ref{assum-4} and the condition that $f_{\bb}$ is Lipschitz continuous with a Lipschitz constant $L_{\bb}$ (a condition satisfied in the affine case), the relationship
$\sup_{F \in \overline{\mathcal{B}}_p(\widehat{F}_N, \epsilon)} \rho^F(Y \cdot f_{\bb}(\mathbf{X})) \leq \rho^{\widehat{F}_N}(Y \cdot f_{\bb}(\mathbf{X})) + M L_{\bb} \epsilon$ holds.
}
\begin{align}\label{eq-nonaffineconfidence}
\p\left(\rho^{F^*}(Y \cdot f_{\bb}(\mathbf X))\le \sup_{F\in \overline{{\cal B}}_p(\widehat{F}_N,  g^{-1}(\epsilon_N) )}\rho^{F}(Y \cdot f_{\bb}(\mathbf X))\right)\ge 1-\eta,
\end{align}
for a fixed $\bb$.
\footnote{
Notably, \cite{ALCD02} also studied the set inclusion property between the projection of a Wasserstein ball and a lower-dimensional Wasserstein ball. However, their focus was on establishing the existence of a lower-dimensional Wasserstein ball as a superset of the projection (i.e., the converse of \eqref{incl}), which cannot be used to derive generalization bounds based on (Wasserstein) in-sample risk, i.e. $\sup_{F\in \overline{\cal B}_p(\widehat{F}_N,\epsilon)}\rho^{F}(Y\cdot f_{\bb}(\mathbf X))$.}  
Extending this to generalization bounds -- i.e., the union bound of the above -- is achievable using covering number techniques. 
We leave the exact technical details and derivations to Appendix \ref{sec:proof6}, and present here the final generalization bounds along with the necessary assumptions.

\begin{assumption}\label{ass:UB021}\label{ass:UB2}
There exists $a>p$ such that 
\begin{align*}
A:=\sup_{\bb\in\mathcal D}\E^{F^*}[\exp(|f_{\bb}(\mathbf X)|^a)]<\infty.
\end{align*}
\end{assumption} 
\begin{assumption}\label{ass:UB03}
The decision rule $f_{\bb}:\R^n\to\R$ is continuous for each $\bb \in\mathcal D$, and there exists a strictly increasing convex function $g:\R_+\to\R_+$  such that for every $\bb \in\mathcal D$,  $\mathbf x\in \R^n,$ and $\epsilon\ge 0$,
\begin{align*}
\sup_{\|\mathbf y\|\le \epsilon} f_{\bb}(\mathbf x+\mathbf y)-f_{\bb}(\mathbf x)\ge g(\epsilon)
~~{and}~~ 
 f_{\bb}(\mathbf x) -\inf_{\|\mathbf y\|\le \epsilon} f_{\bb}(\mathbf x+\mathbf y)\ge g(\epsilon) .
\end{align*}
\end{assumption}

\begin{assumption}\label{ass:UB01}\label{ass:UB02}
There exist $a_1,a_2\ge 0$ and $r\in[1,p]$ such that 
\begin{align*}
|f_{\bb}(\mathbf x)-f_{\widetilde{\bb}}(\mathbf x)|\le (a_1\|\mathbf x\|^r+a_2)\|\bb-\widetilde{\bb}\|_{\mathcal D},~~\forall\, \mathbf x\in\R^n~{\rm and}~\bb,\widetilde{\bb}\in\mathcal D.
\end{align*}
\end{assumption}

\begin{assumption}\label{assum-4}
There exist $M>0$ and $k\in[1,p]$ such that
	\begin{align*}
		|\rho(Z_1)-\rho(Z_2)|\le M \left(\E[|Z_1-Z_2|^k]\right)^{1/k},~~\forall\, Z_1,Z_2\in L^p.
	\end{align*}
\end{assumption}

Assumption \ref{ass:UB021}  is a fairly standard light-tail condition, necessary for invoking the measure concentration property of a one-dimensional Wasserstein ball;  see Proposition \ref{prop-lighttail} in Appendix \ref{sec:sufcoU2} for a sufficient condition for Assumption \ref{ass:UB021} for  a class of decision rules $\{f_{\bb}\}_{\bb\in\mathcal D}$.
Assumption \ref{ass:UB03}, however, is more notable, as it successfully characterizes decision rules that ensure the set inclusion \eqref{incl} holds. Simply put, decision rules must exhibit sufficiently increasing and decreasing behaviour; otherwise, a Wasserstein ball dominated by the projection set would not exist. Assumptions \ref{ass:UB01} and \ref{assum-4} are regularity conditions required for constructing union bounds through covering number techniques.

\begin{theorem}\label{co-nonunion}
For $p\in[1,\infty)$, $\eta\in(0,1)$ and $\mathcal D\subseteq \R^m$ with some $m\in\N$, 
assume that Assumptions \ref{ass:UB021} and \ref{ass:UB03} hold, Assumptions \ref{ass:UB01} and
\ref{assum-4} hold with $rk\le p$, and $U_{\mathcal D}:=\sup_{\bb\in\mathcal D}\|\bb\|_{\mathcal D}<\infty$. Then  we have
\begin{align*}
\p\left(\rho^{F^*}(Y\cdot f_{\bb}(\mathbf X))\le \sup_{F\in \overline{\cal B}_p(\widehat{F}_N,\epsilon_N)}\rho^{F}(Y\cdot f_{\bb}(\mathbf X))+\tau_N,~~\forall~\bm \beta\in\mathcal D \right) \ge 1-\eta-\frac1N,
\end{align*}
where $\tau_N= {M[(2^{r-1}+1)a_1(\E^{F^*}[\|\mathbf X\|^{rk}])^{1/k}+2^{r-1}a_1({\rm\mathbb{V}ar}^{F^*}(\|\mathbf X\|^{rk}))^{1/(2k)}+2^{r-1}a_1\epsilon_N^r+2a_2]}/{N}$, 
\begin{align*}
 \epsilon_N=g^{-1}(\epsilon_{p,N})~~{\rm with}~~\epsilon_{p,N}=\begin{cases}
\left(\frac{\log(c_1/\eta)+m\log(1+2 U_{\mathcal D}N)}{c_2 N}\right)^{1/(2p)}&\text{if}~N\ge \frac{\log(c_1/\eta)+m\log(1+2 U_{\mathcal D}N)}{c_2},\\
\left(\frac{\log(c_1/\eta)+m\log(1+2 U_{\mathcal D}N)}{c_2 N}\right)^{1/a}&\text{if}~N< \frac{\log(c_1/\eta)+m\log(1+2 U_{\mathcal D}N)}{c_2},
\end{cases}
\end{align*}
and
$c_1,c_2$ are positive constants that only depend on $p$, $a$ and $A$.
\end{theorem}

The decay rate here, $g^{-1}\left((m/N)^{1/(2p)}\right)$, maintains an order independent of dimensionality, mirroring the $(m/N)^{1/2}$ rate reported in \cite{G22}. Although our rate may not be as rapid, due to its dependence on the function $g$ and the parameter $p$, 
it is derived under less restrictive distributional assumptions and is valid for any type-$p$ Wasserstein ball. Most importantly, it accommodates a significantly broader class of measures $\rho$. A notable class of applications for these bounds includes any regression problem and decision rule, with a rate that, as demonstrated below, is independent of the function $g$.

\begin{example}[Regression]\label{ex-reg}
In the regression problem,  the data-generating distribution $F^*$ satisfies $F^*(\{Y=1\})=1$  and the decision rule has the form: $f_{\bb}(\mathbf x)=x_1-h_{\bb}(\mathbf x_{(2,:)})$, ${\bf x}\in\R^n$, where $x_1$ is the first component of $\mathbf x$ and $\mathbf x_{(2,:)}$ represents the rest components, i.e., we use $h_{\bb}(\mathbf x_{(2,:)})$ to predict the output $x_1$. 
Suppose that the norm on $\R^n$ satisfies $\|(1,0,\dots,0)\|=1$.
Then, Assumption \ref{ass:UB03} holds for any $h_{\bb}$ with $g(\epsilon)=\epsilon$. Indeed, we have  for any $\mathbf x\in\R^n$ and $\epsilon\ge 0$,
$ 
 \sup_{\|\mathbf y\|\le\epsilon} f_{\bb}(\mathbf x+\mathbf y)-f_{\bb}(\mathbf x)
 \ge f_{\bb}(\mathbf x+(\epsilon,0,\dots,0))-f_{\bb}(\mathbf x)
 = \epsilon 
$ 
 and 
$ 
f_{\bb}(\mathbf x)-\inf_{\|\mathbf y\|\le\epsilon} f_{\bb}(\mathbf x+\mathbf y)
 \ge  f_{\bb}(\mathbf x) - f_{\bb}(\mathbf x-(\epsilon,0,\dots,0))= \epsilon.
$ 
Hence, any decision rule for a regression problem satisfies Assumption \ref{ass:UB03} with $g(\epsilon)=\epsilon$.
\end{example}

We further illustrate the application of this example in the context of operations management. Specifically, the feature-based newsvendor problem (\cite{ZYG24}), also known as the big-data newsvendor (\cite{BR18}), represents as a special case of regression problems.

\begin{example}[Feature-based newsvendor]\label{ex:FBnews}
In the newsvendor problem, the decision maker seeks the optimal order quantity for a product facing uncertain demand $X_1$, with holding costs of $a> 0$ and back-order costs of $b> 0$. In the world of big data, the decision maker has access to a feature vector $\mathbf x_{(2,:)}\in \R^{n-1}$, which can be leveraged to enhance ordering decisions based on this additional information. In this setting,  define the decision rule as $f_{\bb}(\mathbf x)=x_1-h_{\bb}(\mathbf x_{(2,:)})$ with $\bb\in\mathcal D$ and the data-driven distributionally robust feature-based newsvendor problem has the form 
\begin{align*}
\inf_{\bb\in\mathcal D}\sup_{F\in \overline{\mathcal B}_p(\widehat{F}_N,\epsilon)} \rho^F(b (f_{\bb}(\mathbf X))_++a (f_{\bb}(\mathbf X))_-).
\end{align*}
Applying Example \ref{ex-reg}, we know that any decision rule of feature-based newsvendor problem satisfies Assumption \ref{ass:UB03} with $g(\epsilon)=\epsilon$.
\end{example}

The above generalization bound can be extended to an even broader class of decision rules when 
 the measure $\rho$ satisfies monotonicity, i.e., $\rho(Z_1)\le \rho(Z_2)$ whenever $Z_1\le Z_2$, a natural property in problems such as risk minimization.
Specifically, we consider the following data-driven Wasserstein DRO problem:
 \begin{align*}
\inf_{\bb \in \mathcal D} \sup_{F\in {\cal B}_p(\widehat{F}_{\mathbf X,N},\epsilon)} \rho^F(f_{\bb}(\mathbf X)),
\end{align*}
where $\rho$ satisfies monotonicity, 
$\widehat{F}_{\mathbf X,N}=\frac{1} {N}\sum_{i=1}^N\delta_{\widehat{\mathbf x}_i}$ represents the empirical distribution of the component $\mathbf X$,
and ${\cal B}_p(\widehat{F}_{\mathbf X,N},\epsilon)$ is defined by \eqref{eq-WU-multi-d}.
In these cases,
Assumption \ref{ass:UB03} can be relaxed, requiring only that each decision rule $f_{\bb}$ exhibits sufficiently increasing behavior. 

\begin{assumption}\label{ass:UB05}\label{ass:UB04}
The decision rule $f_{\bb}:\R^n\to\R$ is continuous, and there exists a strictly increasing convex function $g:\R_+\to\R_+$ such that for every $\bb \in\mathcal D$, 
\begin{align*}
\sup_{\|\mathbf y\|\le \epsilon} f_{\bb}(\mathbf x+\mathbf y)-f_{\bb}(\mathbf x)\ge g(\epsilon),~~ \forall \mathbf x  \in \R^n,~\epsilon\ge 0.
\end{align*}
\end{assumption}

\begin{theorem}\label{co-nonunion1}
For $p\in[1,\infty)$, $\eta\in(0,1)$ and $\mathcal D\subseteq \R^m$ with some $m\in\N$, assume that
Assumptions \ref{ass:UB2} and \ref{ass:UB05} hold, Assumptions \ref{ass:UB01} and \ref{assum-4} hold with $rk\le p$, $U_{\mathcal D}:=\sup_{\bb\in\mathcal D}\|\bb\|_{\mathcal D}<\infty$, and $\rho:L^p\to\R$ satisfies monotonicity. Then, 
\begin{align*}
		\p\left(\rho^{F^*}(f_{\bb}(\mathbf X))\le \sup_{F\in {\cal B}_p(\widehat{F}_{{\bf X},N},\,\epsilon_N)}\rho^{F}(f_{\bb}(\mathbf X))+\tau_N,~~\forall~\bm \beta\in\mathcal D \right) \ge 1-\eta-\frac1N,
\end{align*}
where  $\tau_N= {M[(2^{r-1}+1)a_1(\E^{F^*}[\|\mathbf X\|^{rk}])^{1/k}+2^{r-1}a_1({\rm\mathbb{V}ar}^{F^*}(\|\mathbf X\|^{rk}))^{1/(2k)}+2^{r-1}a_1\epsilon_N^r+2a_2]}/{N}$ and $\epsilon_N=g^{-1}(\epsilon_{p,N})$ with $\epsilon_{p,N}$  defined in Theorem \ref{co-nonunion}.
\end{theorem}

In particular, Assumption \ref{ass:UB05} accommodates non-linear decision rules, including quadratic and max-affine functions (cf. \cite{EK18} and Example \ref{ex-maxaffine}), which can be bounded below.

\begin{example}[Quadratic function]\label{ex:quadratic}
Define the decision variable as $\bb=(\Sigma,\mathbf a,b)$, where $\mathbf a\in\R^n$, $b\in\R$, and $\Sigma\in\R^{n\times n}$ is a positive semi-definite matrix. 
Let $\mathcal D$ be the feasible set of $\bb$ such that $M:=\inf_{\bb\in\mathcal D}\|\Sigma\|_S>0$, where
$\|\Sigma\|_{S}$ is the spectral norm of $\Sigma$ and is defined as the largest absolute value of all eigenvalues. 
Define the norm on $\mathcal D$ as 
$ 
 \|\bb\|_{\mathcal D}
 =\|\Sigma\|_S+\|\mathbf a\|_2+|b|~{\rm with}~\bb=(\Sigma,\mathbf a,b)\in\mathcal D.
 $
Let
$f_{\bb}(\mathbf x)=\mathbf x^{\top}\Sigma\mathbf x+\mathbf a^\top \mathbf x+b$ for $x\in\R^n$ and $\bb\in\mathcal D$. 
For this specific case,
we first show that $\{f_{\bb}\}_{\bb\in\mathcal D}$
satisfies
Assumption \ref{ass:UB02} with $a_1=a_2=r=2$ and Assumption \ref{ass:UB04}
with $g(\epsilon)=M\epsilon^2$. 
For $\bb=(\Sigma,\mathbf a,b)\in \mathcal D$ and $\widetilde{\bb}=(\widetilde\Sigma,\widetilde{\mathbf a},\widetilde b)\in \mathcal D$, we have
\begin{align*}
|f_{\bb}(\mathbf x)-f_{\widetilde{\bb}}(\mathbf x)|
&\le |\mathbf x^\top(\Sigma-\widetilde{\Sigma})\mathbf x|+|(\mathbf a-\widetilde{\mathbf a})^\top\mathbf x|+|b-\widetilde b|\\
&\le \|\mathbf x\|_2^2\|\Sigma-\widetilde{\Sigma}\|_S+\|\mathbf x\|_2\|\mathbf a-\widetilde{\mathbf a}\|_2+|b-\widetilde b|
\le 2(\|\mathbf x\|_2^2+1)\|\bb-\widetilde{\bb}\|_{\mathcal D}.
\end{align*}
Hence, Assumption \ref{ass:UB02} holds with $a_1=a_2=r=2$.
For any $\mathbf x\in\R^n$, $\epsilon\ge 0$ and $\bb=(\Sigma,\mathbf a,b)\in \mathcal D$, it follows from the orthogonal decomposition of $\Sigma$, i.e., $\Sigma=Q\Lambda Q^\top$ where $Q$ is an orthogonal matrix and $\Lambda$ is a diagonal matrix, that
\begin{align*}
\sup_{\|\mathbf y\|_2\le \epsilon} f_{\bb}(\mathbf x+\mathbf y)-f_{\bb}(\mathbf x)&=\sup_{\|\mathbf y\|_2\le \epsilon} \{\mathbf y^{\top}\Sigma\mathbf y+(2\mathbf x^\top \Sigma+\mathbf a^\top) \mathbf y\}\\
&=\sup_{\|\mathbf y\|_2\le \epsilon}\left\{(Q^\top\mathbf y)^\top \Lambda(Q^\top\mathbf y)+(2\mathbf x^\top Q\Lambda Q^\top+\mathbf a^\top)\mathbf y \right\}\\
&=\sup_{\|\mathbf z\|_2\le \epsilon}\left\{\mathbf z^\top \Lambda\mathbf z+(2\mathbf x^\top Q\Lambda+\mathbf a^\top Q)\mathbf z \right\}\ge \|\Sigma\|_S \epsilon^2\ge M\epsilon^2.
\end{align*}
Hence, Assumption \ref{ass:UB04} holds with $g(\epsilon)=M\epsilon^2$. 
Further, if we assume that $U_{\mathcal D}:=\sup_{\mathcal D}\|\bb\|_{\mathcal D}<\infty$ and there exists $a>p$ such that $\E^{F^*}[\exp(2^{2a-1}U_{\mathcal D}^a\|\mathbf X\|^{2a})]<\infty$, then  one can verify that Assumption \ref{ass:UB2} holds by
Proposition \ref{prop-lighttail} in Appendix \ref{sec:sufcoU2}.
\end{example}

To further emphasize its implications for operations management, we next illustrate how Theorem \ref{co-nonunion1} applies to two-stage stochastic programs.

\begin{example}[Two-stage stochastic program]\label{ex-maxaffine} In the context of two-stage stochastic programs with right-hand-side uncertainty, the recourse function is defined as (\cite{EK18}):
$f_{\bb}(\mathbf x) = \inf_{y} \left\{ q^\top y \;|\; 
Wy \geq H \mathbf x+\bb
\right\},
$
where $\bb \in\mathcal D\subseteq \R^{m}$ with $\|\bb\|_\mathcal D=\|\bb\|$, denotes the first-stage decision variable, and $y$ represents the recourse decision. As shown in \cite{EK18}, this function can be reformulated as a max-affine function 
$f_{\bb}(\mathbf x)=\max_{1 \le k\le K}(v_k^\top H\mathbf x + v_k^\top \bb),$
where $v_k\in \R^m$, $k=1,\ldots,K$ are the vertices of the dual feasible set $\left\{\theta \geq 0\;|\; W^\top \theta = q\right\}$. Suppose $M:= \min_{1\le k\le K}\|H^\top v_k   \|_*>0$.  
We show that $f_{\bb}$ for $\bb\in\mathcal D$ satisfies Assumption \ref{ass:UB02} with $a_1=0$ and $a_2=\max_{1 \le k\le K}  \| v_k\|_*$ and Assumption \ref{ass:UB04} with $g(\epsilon)=M\epsilon$. 
 For $\bb,\,\widetilde{\bb} \in\mathcal D$, we have
\begin{align*}
|f_{\bb}(\mathbf x)-f_{\widetilde{\bb}}(\mathbf x)|
& =\left|\max_{1 \le k\le K}(v_k^\top H  \mathbf x  + v_k^\top \bb) -\max_{1 \le k\le K}(v_k^\top H  \mathbf x  + v_k^\top \widetilde{\bb}) \right|\\
& \le \max_{1 \le k\le K}\left|(v_k^\top H  \mathbf x  + v_k^\top \bb) - (v_k^\top H \mathbf x  + v_k^\top \widetilde{\bb})\right|\\
&= \max_{1 \le k\le K} \left| v_k^\top (\bb -   \widetilde{\bb})\right|
 \le\max_{1 \le k\le K}  \| v_k\|_* \|\bb -  \widetilde{\bb} \|   .
\end{align*}
Hence, Assumption \ref{ass:UB02} holds with $a_1=0$ and $a_2=\max_{1 \le k\le K}  \| v_k\|_*$.
For any $\mathbf x\in\R^n$, $\epsilon\ge 0$ and $\bb \in\mathcal D$, we have
\begin{align*}
\sup_{\|\mathbf y\|\le \epsilon} f_{\bb}(\mathbf x+\mathbf y)-f_{\bb}(\mathbf x)
&=\sup_{\|\mathbf y\|\le \epsilon} \max_{1 \le k\le K}(v_k^\top H  (\mathbf x+\mathbf y) + v_k^\top \bb) 
-\max_{1 \le k\le K}(v_k^\top H   \mathbf x + v_k^\top \bb) \\
&=\max_{1 \le k\le K} \sup_{\|\mathbf y\|\le \epsilon} (v_k^\top H  (\mathbf x+\mathbf y) + v_k^\top \bb) 
-\max_{1 \le k\le K}(v_k^\top H   \mathbf x + v_k^\top \bb) \\
&=\max_{1 \le k\le K}   (v_k^\top H  \mathbf x+ \epsilon \|  H^\top v_k \|_*+ v_k^\top \bb) 
-\max_{1 \le k\le K}(v_k^\top H   \mathbf x + v_k^\top \bb) \\
&\ge \min_{1 \le k\le K}   \epsilon \|  H^\top v_k \|_*  = M\epsilon.
\end{align*}
Hence, Assumption \ref{ass:UB04} holds with $g(\epsilon)=M\epsilon$.
Further, if we assume that $U_{\mathcal D}:=\sup_{\bb\in\mathcal D}\|\bb\|_{\mathcal D}<\infty$ and there exists $a>p$ such that $\E^{F^*}[\exp(2^{a-1}R^a\|\mathbf X\|^{a})]<\infty$ with $R:=\max_{1\le k\le K}\|H^{
\top}v_k\|_*$, then one can verify that Assumption \ref{ass:UB2} holds by
Proposition \ref{prop-lighttail} in Appendix \ref{sec:sufcoU2}. We note that extending this discussion to cases where the matrix $H$ also represents a first-stage decision is straightforward, and verifying this extension closely mirrors our existing analysis, thus it is omitted here for brevity.
\end{example}

Finally, we show that for problems such as regression or risk minimization, the ``curse of the order $p$" -- i.e., the dependence of the decay rate on $p$ -- can be overcome for any non-linear decision
rules covered by Assumption \ref{ass:UB05}.


\begin{theorem}\label{co-nonunion2}
For $p\in[1,\infty]$, $\eta\in(0,1)$ and $\mathcal D\subseteq \R^m$ with some $m\in\N$,  
assume that Assumptions \ref{ass:order1} and \ref{ass:UB05} hold, Assumption \ref{ass:UB2} holds with $a>1$, Assumption \ref{ass:UB01} holds with $r=1$ and
$U_{\mathcal D}:=\sup_{\bb\in\mathcal D}\|\bb\|_{\mathcal D}<\infty$, and $\rho:L^1\to\R$ satisfies monotonicity and  
$|\rho(Z_1)-\rho(Z_2)|\le M\esssup|Z_1-Z_2|$ for  $Z_1,Z_2\in L^1$. Then,  
\begin{align*}
		\p\left(	\rho^{F^*}(f_{\bb}(\mathbf X))\le \sup_{F\in  {\cal B}_p(\widehat{F}_{{\bf X},N},\epsilon_N)}\rho^{F}(f_{\bb}(\mathbf X))+\tau_N,~~\forall~\bm \beta\in\mathcal D \right) \ge 1-\eta-\frac1N,
\end{align*}
where 
$\tau_N= {\lambda M[a_1\E^{F^*}[\|\mathbf X\|]+a_1({\rm\mathbb{V}ar}^{F^*}\!\!(\|\mathbf X\|))^{1/2}+a_1\epsilon_N+2a_2]}/{N}$ and $\epsilon_N=g^{-1}(\lambda\epsilon_{1,N})$ with $\epsilon_{1,N}$  defined in Theorem \ref{co-nonunion}.
\end{theorem}

We note that Assumption \ref{ass:order1} and the Lipschitz continuity with respect to the $L^\infty$-norm together imply Assumption \ref{assum-4} (with $k=1$). Hence, Assumption \ref{assum-4} is omitted in Theorem \ref{co-nonunion2} (see Appendix \ref{app:th10} for details).


\section{Conclusion}\label{sec:conclusion}
As ML and OR/MS applications increasingly adopt diverse decision criteria, identifying methods that ensure robust, data-driven solutions with strong generalization guarantees across a wide range of criteria is essential. This paper establishes the universality of Wasserstein DRO in achieving such guarantees. We show that Wasserstein DRO attains generalization bounds comparable to max-sliced Wasserstein DRO for any decision criterion under affine decision rules, even though the Wasserstein ball is a significantly less conservative ambiguity set. Furthermore, we extend these guarantees to general criteria and decision rules under light-tail assumptions, avoiding the curse of dimensionality. Our projection-based analysis of the Wasserstein ball offers key insights into why these guarantees require minimal reliance on specific properties of decision criteria.

Beyond developing generalization guarantees, our work emphasizes their practical significance for OR/MS problems. We demonstrate that these guarantees can often be achieved through regularization equivalents, as detailed in Theorems \ref{th-RegularizationEL}, \ref{th-necessity-highorderloss}, and \ref{th-uniqueDT}, all of which are efficiently solvable as convex programs with complexity comparable to nominal problems. Importantly, our impossibility theorems identify precisely when solving the full Wasserstein DRO problem becomes necessary, defining the boundaries where simpler regularization approaches no longer suffice. These results deepen the theoretical foundation of Wasserstein DRO while charting a clear course for developing more efficient methods to address full Wasserstein DRO formulations.

\newpage

\appendix
\renewcommand{\appendixname}{Appendix~\Alph{section}}

\section{Proofs of Section \ref{Gen}}\label{sec:EC1}
Before presenting the proofs in Section \ref{sec:EC1}, we first introduce a lemma inspired by the optimal coupling (see, e.g., Theorem 4.1 of \cite{V09}),\footnote{We thank a referee for bringing this to our attention.} which will be used throughout the proofs in the appendix. 

For $m\in \N$, a function $c:\R^m\times \R^m\to\R\cup\{+\infty\}$  is a \emph{cost function} if $c(\mathbf x,\mathbf y)=c(\mathbf y,\mathbf x)\ge 0$ for all $\mathbf x,\mathbf y\in\R^m$. For a given cost function $c$, define $W_{c}(F,F_0)= \inf_{\pi\in\Pi(F,F_0)}\E^{\pi}[c(\bm\xi,\bm\xi_0)]$  for two distributions $F,F_0$ on $\R^m$.
\begin{lemma}\label{lm-optcoupling}
Let $c:\R^m\times \R^m \to\R\cup\{+\infty\}$ with $m\in\N$ be a lower
semicontinuous cost function. Given two distributions $F,F_0$  on $\R^m$ and $\mathbf X_0\sim F_0$, 
 there exists $\mathbf X\sim F$ such that $W_{c}(F,F_0)=\E[c(\mathbf X,\mathbf X_0)]$. 
 As a result, we have 
 for $\epsilon\ge 0$,
 \begin{align}\label{eq-optc}
\left\{G: W_{c}(G,F_0)\le \epsilon\right\}=\left\{F_{\mathbf X}: \E[c(\mathbf X,\mathbf X_0)]\le \epsilon\right\}.
 \end{align}
\end{lemma}

\proof{Proof.} 
Note that by the optimal coupling (see e.g., Theorem 4.1 of \cite{V09}),  there exists $\mathbf Y_0\sim F_0$ and $\mathbf Y\sim F$ such that
 $ 
W_{c}(F,F_0)
=\E[c(\mathbf Y,\mathbf Y_0)].
$  
Since $\mathbf X_0\laweq \mathbf Y_0$,  by Theorem 8.17 of \cite{K21}, there exists $\mathbf X\laweq \mathbf Y\sim F$ such that $(\mathbf X,\mathbf X_0)\laweq (\mathbf Y,\mathbf Y_0)$, and thus, 
$  
\E[c(\mathbf X,\mathbf X_0)]=\E[c(\mathbf Y,\mathbf Y_0)]=
W_{c}(F,F_0).
$   
This completes the proof. 
\hfill $\Box$

Consider the type-$p$ Wasserstein balls $\overline{\mathcal B}_{p}$ and $\mathcal B_p$  defined by \eqref{eq-WU-multi}  and  \eqref{eq-WU-multi-d}, respectively. 
By
Lemma \ref{lm-optcoupling}, we have for   $(Y_0,\mathbf X_0)\sim F_0$, it holds that 
\begin{align*}
\overline{\cal B}_p (F_0,\epsilon)
=\{F_{(Y,\mathbf X)}: \E[d((Y,\mathbf X),(Y_0,\mathbf X_0))^p]\le \epsilon^p\}~~{\rm and}~~
{\cal B}_p (F_{{\bf X}_0},\epsilon)=\{F_{\mathbf X}: \E[\|\mathbf X-\mathbf X_0\|^p]\le \epsilon^p\}.
\end{align*}
The two formulas above will be used frequently throughout the appendix, and we may apply them directly without referencing Lemma \ref{lm-optcoupling}.


\noindent{\bf Proof of Proposition \ref{prop-equi}.} 
Note that
\begin{align*}
	\{ F_{Y\bf X}: F_{(Y,{\bf X})} \in \overline{\cal B}_p (F_0,\epsilon)\}
	&=\left\{F_{Y\bf X}: \E[d((Y,\mathbf X),(Y_0,\mathbf X_0))^p]\le \epsilon^p\right\}\\
	&=\left\{F_{Y_0\bf X}: \E[\|\mathbf X-\mathbf X_0\|^p]\le \epsilon^p\right\}\\
	&=\left\{F_{Y_0\bf X}: \E[\|Y_0\mathbf X-Y_0\mathbf X_0\|^p]\le \epsilon^p\right\}=: \mathcal B_{(1)},
\end{align*}
where the first equality is due to Lemma \ref{lm-optcoupling},
the second equality follows from the definition in \eqref{eq-classificationmetric} which implies  $Y=Y_0$ almost surely for any  $(Y, {\bf X})$ satisfying $ \E[d((Y,\mathbf X),(Y_0,\mathbf X_0))^p]\le \epsilon^p$, and the third equality follows from $\|\mathbf X-\mathbf X_0\|=\|Y_0\mathbf X-Y_0\mathbf X_0\|$ as $|Y_0|=1$. 
Additionally, it follows from Lemma \ref{lm-optcoupling} that
\begin{align*}
	{\cal B}_p (F_{Y_0\bf X_0},\epsilon)=\left\{F_{\bf Z}: \E[\|{\bf Z}- Y_0{\bf X_0}\|^p]\le\epsilon^p\right\}=:\mathcal B_{(2)},
\end{align*}
it holds that $\mathcal B_{(1)}\subseteq\mathcal B_{(2)}$ obviously.
To see the converse inclusion, note that 
for any $F\in \mathcal B_{(2)}$, there exists $\mathbf Z$ such that $\mathbf Z\sim F$ and
$
\E[\| {\bf Z}-Y_0{\bf X}_0\|^p ]\le\epsilon^p.
$
Taking $(Y,{\bf X}) = (Y_0, {\bf Z}/Y_0 )$, we have $Y=Y_0$ and
$$
\E[\| Y_0{\bf X}-Y_0{\bf X}_0 \|^p ]
=\E\left[\left\| {\bf Z}-Y_0{\bf X}_0 \right\|^p \right]
\le\epsilon^p.
$$
Hence, we have $F \in \mathcal B_{(1)}$, and thus $\mathcal B_{(2)}\subseteq \mathcal B_{(1)}$.  Therefore, we have $\mathcal B_{(1)}=\mathcal B_{(2)}$. This completes the proof of \eqref{eq:221010-1}. 
To show \eqref{eq:1010-2},  denote by 
\begin{align*}
    \mathcal B_1=\overline{\mathcal B}_{p|\bb}( F_0,\epsilon),~~ 
\mathcal B_2=\mathcal B_p\left(F_{Y_0\cdot\bm\beta^{\top}\bf X_0},\epsilon\|\bm\beta\|_*\right)~~{\rm and}~~\mathcal B_3=\{F_{\bm \beta^\top \mathbf Z}: F_{\bf Z}\in  \mathcal B_p^{\rm ms}(F_{Y_0\mathbf X_0},\epsilon)\}.
\end{align*}We first verify that $\mathcal B_1=\mathcal B_2$.
With the aid of  \eqref{eq:221010-1},  we can verify this similarly as  the proof of  Theorem 5 in \cite{MWW22}. For ease of completeness, we give a proof of $\mathcal B_1=\mathcal B_2$ here. The case of $\|\bm\beta\|=0$ is trivial, and we assume now $\|\bm\beta\|>0$.
Note that
\begin{align}\label{eq-B_1}
\mathcal B_1&=
	\{F_{Y\cdot \bm \beta^\top \mathbf X}: F_{(Y,\mathbf X)}\in \overline{\cal B}_p(F_0,\epsilon)\}\notag\\
 &=\left\{F_{Y\cdot\bm\beta^\top \mathbf X}: \E[d((Y,\mathbf X),(Y_0,\mathbf X_0))^p]\le \epsilon^p\right\}
	=\left\{F_{Y_0\cdot\bm\beta^\top \mathbf X}: \E[\|\mathbf X-\mathbf X_0\|^p]\le \epsilon^p\right\},
\end{align}
where the last equality follows from the definition in \eqref{eq-classificationmetric} and Lemma \ref{lm-optcoupling}. 
On one hand, for $F\in \mathcal B_1$, there exists $\mathbf X$ such that $Y_0\cdot\bm\beta^\top \mathbf X\sim F$ and $\E[\|\mathbf X-\mathbf X_0\|^p]\le \epsilon^p$. Thus, we have
\begin{align}\label{eq-proone}
	\E[|Y_0\cdot\bm\beta^\top \mathbf X-Y_0\cdot\bm\beta^\top \mathbf X_0|^p]
	&=\E[|\bm\beta^\top \mathbf X-\bm\beta^\top \mathbf X_0|^p]
	\le \|\bm\beta\|_*^p\E[\|\mathbf X-\mathbf X_0\|^p]
	\le  \epsilon^p\|\bm\beta\|_*^p,
\end{align}
where we use H\"{o}lder's inequality in the first inequality. 
This means
$F=F_{Y_0\cdot\bb^{\top}\mathbf X}\in \mathcal B_2$, and hence, $\mathcal B_1\subseteq \mathcal B_2$.
On the other hand, for $F\in \mathcal B_2$, using Lemma \ref{lm-optcoupling} yields that
there exists $Z\sim F$ such that $\E[|Z-Y_0\cdot\bm\beta^{\top}{\bf X_0}|^p]\le \epsilon^p\|\bm\beta\|_*^p$.
It follows from the definition of dual norm that there exists $\bm\beta_0\in\R^n$ such that $\|\bm\beta_0\|=1$ and $\|\bm\beta\|_*=\bm\beta^\top\bm\beta_0$.
Define
$$
T=Z-Y_0\cdot\bm\beta^{\top}{\bf X_0}~~\mbox{and}~~
{\mathbf X}={\mathbf X_0}+\frac{\bm\beta_0 T}{Y_0\|\bm\beta\|_*}.
$$
It holds that $\E[|T|^p]\le \epsilon^p\|\bm\beta\|_*^p$, and thus,
\begin{align*}
	\E[\|\mathbf X-\mathbf X_0\|^p]=\E\left[\left\|\frac{\bm\beta_0 T}{Y_0\|\bm\beta\|_*}\right\|^p\right]=\frac{\E[|T|^p]}{\|\bm\beta\|_*^p}\le \epsilon^p.
\end{align*}
It follows from \eqref{eq-B_1} that $F_{Y_0\cdot\bm\beta^{\top}{\bf X}}\in\mathcal B_1$.
Noting that $Z=Y_0\cdot\bm\beta^{\top}{\bf X}$ and $Z\sim F$, we have $F\in\mathcal B_1$. Hence, we conclude $ \mathcal B_2\subseteq \mathcal B_1$. This completes the proof of $\mathcal B_1=\mathcal B_2$.

It remains to prove that $\mathcal B_1\subseteq \mathcal B_3\subseteq \mathcal B_2$.  
For  $G_1,G_2\in \mathcal M_p(\R^n)$, denote \begin{align*} 
 {W}_p^{\rm ms}(G_1,G_2) := \sup_{\bm\gamma:\|\bm\gamma\|_*=1}\inf_{\pi\in\Pi(G_1,G_2)}\left(\E^{\pi}[|\bm\gamma^\top \bm\xi_1-\bm\gamma^\top \bm\xi_2|^p]\right)^{1/p},
	\end{align*}
    and thus,
$ 
   \mathcal B_3=\{F_{\bm \beta^\top \mathbf Z}:  {W}_p^{\rm ms}(F_{\bf Z},F_{Y_0\mathbf X_0})  \le \epsilon\}.
$ For $F\in\mathcal B_1$ and using \eqref{eq-B_1}, there exists $\mathbf X$ such that $Y_0\cdot\bm\beta^\top \mathbf X\sim F$ and $\E[\|\mathbf X-\mathbf X_0\|^p]\le \epsilon^p$. It holds that
\begin{align*}
{W}_p^{\rm ms}(F_{Y_0 \mathbf X},F_{Y_0\mathbf X_0})^p&=\sup_{\bm\gamma:\|\bm\gamma\|_*=1}\inf_{\pi\in \Pi(F_{ Y_0 \mathbf X},F_{Y_0\mathbf X_0})}\E^{\pi}[|\bm\gamma^\top \bm\xi_1-\bm\gamma^\top\bm\xi_2|^p]\\[3pt]
&\le \sup_{\bm\gamma:\|\bm\gamma\|_*=1} \E[|Y_0\cdot \bm\gamma^{\top}(\mathbf X-\mathbf X_0)|^p]\\[3pt]
&\le \sup_{\bm\gamma:\|\bm\gamma\|_*=1} \|\bm\gamma\|_*^p\E[\|\mathbf X-\mathbf X_0\|^p]\\[3pt]
&=\E[\|\mathbf X-\mathbf X_0\|^p]\le \epsilon^p.
\end{align*}
This implies $F=F_{Y_0\cdot \bm\beta^{\top}\mathbf X}\in\mathcal B_3$, and hence, $\mathcal B_1\subseteq \mathcal B_3$. Suppose now $F\in \mathcal B_3$. 
 There exists $\mathbf Z$ such that $\bm\beta^{\top} \mathbf Z\sim F$ and  ${W}_p^{\rm ms}(F_{\mathbf Z},F_{Y_0\mathbf X_0}) \le \epsilon$. It follows that 
\begin{align*}
\epsilon^p
& \ge  \sup_{\bm\gamma:\|\bm\gamma\|_*=1}\inf_{\pi\in \Pi\left(F_{\mathbf Z},F_{Y_0\mathbf X_0}\right)}\E^{\pi}[|\bm\gamma^\top \bm\xi_1-\bm\gamma^\top\bm\xi_2|^p]\\[3pt]
&\ge \inf_{\pi\in \Pi\left(F_{\mathbf Z},F_{Y_0\mathbf X_0}\right)}\frac1{\|\bm\beta\|_*^p} \E^{\pi}[|{\bm\beta}^\top \bm\xi_1- {\bm\beta}^\top\bm\xi_2|^p]\\[3pt]
&\ge\frac{1}{\|\bm\beta\|_*^p}\inf_{\pi\in \Pi\left(F_{{\bm\beta}^\top\mathbf Z},F_{Y_0\cdot{\bm\beta}^\top\mathbf X_0}\right)}\E^{\pi}[|\xi_1-\xi_2|^p],
\end{align*}
where the second inequality follows from $\bm \beta/\|\bm \beta\|_*\in \{\bm \gamma:\|\bm \gamma\|_*=1\}$.
This implies 
\begin{align*}
W_p \left(F_{{\bm\beta}^\top\mathbf Z},F_{Y_0\cdot{\bm\beta}^\top\mathbf X_0}\right) =\inf_{\pi\in \Pi\left(F_{{\bm\beta}^\top\mathbf Z},F_{Y_0\cdot{\bm\beta}^\top\mathbf X_0}\right)} \left(\E^{\pi}[|\xi_1-\xi_2|^p] \right)^{1/p} \le \|\bm\beta\|_* \epsilon .
\end{align*}
Hence, we have $F=F_{\bm\beta^\top \mathbf Z}\in\mathcal B_2$, which yields $\mathcal B_3\subseteq\mathcal B_2$. This completes the proof.
\endproof

~~
%
	
~~

The proof of Theorem \ref{eq-UGB} relies on our projection result in Proposition \ref{prop-equi} and a concentration result for the max-sliced Wasserstein distance in Theorem 3 of \cite{ORVW22}, which is given as follows. 

\begin{lemma}[Theorem 3 of \cite{ORVW22}]\label{lm-ORVW}
Let $p\ge 1$ and $\eta\in(0,1)$. Suppose that $\widetilde{F}^*\in\mathcal M_p(\R^n)$ satisfies $\Gamma:=\E^{\widetilde{F}^*}[\|\mathbf X\|^s]<\infty$ for some $s>2p$ and $\epsilon_{p,N}(\eta)$ is defined in Theorem \ref{eq-UGB}. Then, we have
\begin{align*}
\p\left(\widetilde{F}^* \in {\cal B}_p^{\rm ms}\big(\widetilde{F}_N, \epsilon_{p,N}(\eta)\big)\right)\ge 1-\eta,
\end{align*}
where $\widetilde{F}_N$ is the empirical distribution of $\widetilde{F}^*.$
\end{lemma}

\noindent{\bf Proof of Theorem \ref{eq-UGB}.}
Denote by $\widetilde{F}_{N}=\frac{1}{N}\sum_{i=1}^N \delta_{\widehat{y}_i \widehat{\mathbf x}_i}$, 
$F^*_{\bm\beta} =F_{Y^*\cdot \bm\beta^\top \mathbf X^*}$, where $(Y^*,\mathbf X^*)\sim F^*$, and $\epsilon_{N} = \epsilon_{p,N}(\eta)$. 
It is clear that $\widetilde{F}_{N}$ is an empirical distribution of $F_{Y^* \mathbf X^*}$. 
It holds that
\begin{align*}
 &\p\left(\rho^{F^*}(Y\cdot\bm \beta^\top\mathbf X)\le
			\sup_{F\in \overline{\cal B}_p(\widehat{F}_N,\epsilon_{N})}\rho^{F}(Y\cdot\bm \beta^\top\mathbf X),~\forall \bm\beta\in\mathcal D\right) \\[3pt]
			&=  \p\left(\rho^{F^*_{\bm\beta}}(Z)\le
			\sup_{F\in \overline{\mathcal B}_{p|\bb}(\widehat{F}_N,\epsilon_{N}) }\rho^{F}(Z),~\forall \bm\beta\in\mathcal D\right) \\[3pt]
			&\ge \p\left(F^*_{\bm\beta}\in\overline{\mathcal B}_{p|\bb}(\widehat{F}_N,\epsilon_{N}) ,~\forall \bm\beta\in\mathcal D\right)\\[3pt]
    &= \p\left(F^*_{\bm\beta}\in  \{ F_{\bb^\top \bf Z}:  F_{\bf Z}\in  {\cal B}_p^{\rm ms}(\widetilde{F}_N,\epsilon_{N} ), 
       ~\forall \bm\beta\in\mathcal D \right)\\[3pt]
&\ge   \p\left( F_{Y^* \mathbf X^*}\in {\cal B}_p^{\rm ms}\big(\widetilde{F}_N, \epsilon_{N}\big)\right) \ge 1-\eta,
\end{align*}
where the second equality follows from $\overline{\mathcal B}_{p|\bb}(\widehat{F}_N,\epsilon_{N}) = \{ F_{\bb^\top \bf Z}: F_{\bf Z}\in  {\cal B}_p^{\rm ms}(\widetilde{F}_N,\epsilon_{N} )\}$ by (ii) of Proposition \ref{prop-equi}  
and we have used Lemma \ref{lm-ORVW} in the last inequality. This completes the proof.
\endproof

{\bf Proof of Theorem \ref{co-Lnonunion}.}
Let $(\widehat{Y}_N,\widehat{\mathbf X}_N)\sim \widehat{F}_N$.  
By the definition of $\overline{\mathcal B}_{p|\bb} ( \widehat{F}_N,\epsilon)$, it follows that 
\begin{align*}
\sup_{F\in \overline{{\cal B}}_p(\widehat{F}_N,\epsilon)}\rho^{F}(Y\cdot\bm \beta^\top\mathbf X)
&=\sup_{F\in \overline{\mathcal B}_{p|\bb} ( \widehat{F}_N,\epsilon) }\rho^F(Z)\\
& =\sup\left\{\rho(Z): \left(\E[|Z-\widehat{Y}_N\cdot\bb^{\top}\widehat{\mathbf X}_N|^p]\right)^{1/p}\le \epsilon\|\bb\|_*\right\}\notag\\
&=\sup\left\{\rho\left(\widehat{Y}_N\cdot\bb^{\top}\widehat{\mathbf X}_N+V\right):(\E[|V|^p])^{1/p}\le \epsilon\|\bb\|_*\right\},
\end{align*}
where 
the second inequality follows from  \eqref{eq:1010-2} in Proposition \ref{prop-equi}, which states that  
$$ \overline{\mathcal B}_{p|\bb} ( \widehat{F}_N,\epsilon)=\mathcal B_p(F_{\widehat{Y}_N\cdot\bb^{\top}\widehat{\mathbf X}_N},\epsilon\|\bb\|_*)
=\left\{F_{Z}: \left(\E[|Z-\widehat{Y}_N\cdot\bb^{\top}\widehat{\mathbf X}_N|^p]\right)^{1/p}\le \epsilon\|\bb\|_*\right\}.$$
In particular,  we have for any $\epsilon>0$
\begin{align}\label{eq:v2-1016-1}
\sup_{F\in \overline{{\cal B}}_\infty(\widehat{F}_N,\epsilon)}\rho^{F}(Y\cdot\bm \beta^\top\mathbf X)
&=\sup\left\{\rho\left(\widehat{Y}_N\cdot\bb^{\top}\widehat{\mathbf X}_N+V\right): |V| \le \epsilon\|\bb\|_*\right\}\\
\sup_{F\in \overline{{\cal B}}_1(\widehat{F}_N,\epsilon)}\rho^{F}(Y\cdot\bm \beta^\top\mathbf X)
&=\sup\left\{\rho\left(\widehat{Y}_N\cdot\bb^{\top}\widehat{\mathbf X}_N+V\right): \E[|V|] \le \epsilon\|\bb\|_*\right\}
.\label{eq:v2-1016-2}
\end{align}
Therefore, we have 
\begin{align}\label{eq-th2chain}
\sup_{F\in \overline{{\cal B}}_p(\widehat{F}_N,\lambda\epsilon)}\rho^{F}(Y\cdot\bm \beta^\top\mathbf X)
&\ge \sup_{F\in \overline{{\cal B}}_{\infty}(\widehat{F}_N,\lambda\epsilon)}\rho^{F}(Y\cdot\bm \beta^\top\mathbf X)\notag\\[3pt]
&=\sup\left\{\rho(\widehat{Y}_N\cdot\bb^{\top}\widehat{\mathbf X}_N+V) : |V| \le \lambda\epsilon\|\bb\|_*\right\}\notag\\[3pt]
&\ge\sup\left\{\rho(\widehat{Y}_N\cdot\bb^{\top}\widehat{\mathbf X}_N+V) : \E[|V|] \le \epsilon\|\bb\|_*\right\}\notag\\[3pt]
&= \sup_{F\in \overline{{\cal B}}_1(\widehat{F}_N,\epsilon)}\rho^{F}(Y\cdot\bm \beta^\top\mathbf X),~~\forall \bb\in\mathcal D,~\epsilon\ge 0,
\end{align}
where  the first inequality is due to $\overline{{\cal B}}_{\infty}(\widehat{F}_N,\lambda\epsilon)\subseteq \overline{{\cal B}}_{p}(\widehat{F}_N,\lambda\epsilon)$,
and the second inequality is due to Assumption \ref{ass:order1},  and   the two equalities come from \eqref{eq:v2-1016-1} and \eqref{eq:v2-1016-2}, respectively. Combining \eqref{eq-th2chain} and Theorem \ref{eq-UGB} with type-$1$ Wasserstein ball yields the desired result.
\endproof

{\color{blue}\section{Proofs of Section \ref{reg}}\label{sec:proofs4}
}

	For ease of exposition, we need the following notations.	
	For a random variable $Z$, we define $\|Z\|_p=(\E[|Z|^p])^{1/p}$.
	We use $\id_A$ to represent the indicator function, i.e., $\id_A(\omega)=1$ if $\omega\in A$, and $\id_A(\omega)=0$ otherwise.
	The sign function on $\R$ is defined as
	$$
	{\rm sign}(x)=-\id_{(-\infty,0)}(x)+\id_{[0,\infty)}(x).
	$$
	To prove the results in Section \ref{reg},  we need the following auxiliary lemmas. The first lemma   follows immediately from Proposition  \ref{thm:221010-1}.
	\begin{lemma}\label{lm-eqonehigh} 
 Given $\bb \in \R^n$, we have
		\begin{align*}
			\sup_{F\in {\cal B}_{p}(G_0,\epsilon)} \rho^F(\bb^\top \mathbf Z)=	\sup_{F\in \mathcal  B_p\left(F_{\bb^\top \mathbf Z_0},\epsilon\|\bb\|_*\right)}
			\rho^F(Z),
		\end{align*}
		where $\mathbf Z_0\sim G_0$ and $F_{\bb^\top \mathbf Z_0}$ is the distribution of $\bb^\top \mathbf Z_0$.
	\end{lemma}

	Based on Lemma \ref{lm-eqonehigh}, we present the following lemma, which will be used in the proofs of the main results.
	
	\begin{lemma}\label{lm-onehigheq}
		Let $p\in[1,\infty]$ and $C>0$. For $\rho: L^p\to\R$, the following statements are equivalent.
		\begin{itemize}
			\item[(i)] For any $G_0\in\mathcal M_p(\R^n)$, $\epsilon\ge 0$ and $\mathcal D\subseteq \R^n$,
			\begin{align}\label{eq-WDROeq}
				\inf_{\bb\in\mathcal D}	\sup_{F\in {\cal B}_p(G_0,\epsilon)} \rho^F(\bb^{\top}\mathbf Z)=\inf_{\bb\in\mathcal D} \left\{\rho^{G_0}(\bb^{\top}\mathbf Z)+C\epsilon\|\bb\|_*\right\}.
			\end{align}
			\item [(ii)] For any $Z\in L^p$ and $\epsilon\ge 0$,
			\begin{align}\label{eq-regEU}
				\sup_{\|V\|_p\le \epsilon}\rho(Z+V)=\rho(Z)+C\epsilon.
			\end{align}
		\end{itemize}
	\end{lemma}

\proof{Proof.}
\underline{(i) $\Rightarrow$ (ii)}: For $Z\in L^p$, take $\mathcal D=\{\bb_0\}$ with $\bb_0=(1,0,\cdots,0)/\|(1,0,\cdots,0)\|_*$, and let $G_0$ be the distribution of $(Z,0,\dots,0)$.  It holds that
$\|\bb_0\|_*=1$,  and thus, the left-hand side of \eqref{eq-WDROeq} reduces to 
\begin{align*}
\inf_{\bb\in\mathcal D}	\sup_{F\in {\cal B}_p(G_0,\epsilon)} \rho^F(\bb^{\top}\mathbf Z)=	\sup_{F\in {\cal B}_p(G_0,\epsilon)} \rho^F(\bb_0^{\top}\mathbf Z)
=\sup_{F\in\mathcal B_p(F_Z,\epsilon)}\rho^F(Y)
=\sup_{\|Y-Z\|_p\le \epsilon}\rho(Y)
=\sup_{\|V\|_p\le \epsilon} \rho(Z+V),
\end{align*}
where the second equality follows from Lemma \ref{lm-eqonehigh}.
Note that the right-hand side of \eqref{eq-WDROeq} reduces to 
\begin{align*}
\inf_{\bb\in\mathcal D} \left\{\rho^{G_0}(\bb^{\top}\mathbf Z)+C\epsilon\|\bb\|_*\right\}=\rho(Z)+C\epsilon.
\end{align*}
Combining the above two equations with \eqref{eq-WDROeq}
yields \eqref{eq-regEU}. This completes the proof of the implication (i) $\Rightarrow$ (ii).

\underline{(ii) $\Rightarrow$ (i)}: For $G_0\in\mathcal M_p(\R^n)$, $\epsilon\ge 0$ and $\mathcal D\subseteq \R^n$, let $\mathbf Z_0\sim G_0$. We have for any $\bb\in\mathcal D$,
\begin{align*}
\sup_{F\in {\cal B}_p(G_0,\epsilon)} \rho^F(\bb^{\top}\mathbf Z)
&=\sup_{F\in \mathcal  B_p\left(F_{\bb^\top \mathbf Z_0},\epsilon\|\bb\|_*\right)}\rho^F(Z)\\[3pt]
& =\sup_{ \|Z-\bb^{\top}\mathbf Z_0\|_p\le \epsilon\|\bb\|_*} \rho(Z)\\[3pt]
&=\sup_{\|V\|_p\le \epsilon\|\bb\|_*}\rho(\bb^\top \mathbf Z_0+V) 
 =\rho(\bb^\top \mathbf Z_0)+C\epsilon\|\bb\|_*,
\end{align*}
where the first equality follows from Lemma \ref{lm-eqonehigh} and the last equality is due to \eqref{eq-regEU}. Hence, \eqref{eq-WDROeq} holds immediately, and thus, we complete the proof.\endproof


\begin{remark}\label{re:EC1}
In Lemma \ref{lm-onehigheq}, neither the distribution $G_0\in\mathcal M_p(\R^n)$ or the set $\mathcal D\subseteq \R^n$ is fixed. If we fix  $G_0\in\mathcal M_p(\R^n)$ and $\mathcal D \subseteq \R^n$, then  by similar arguments, we can show that \eqref{eq-WDROeq} holds for all $\epsilon\ge 0$  if   \eqref{eq-regEU}  holds with  $Z=\bb^{\top} \mathbf{Z}_0$ for all $\bb\in\mathcal D$ and $\epsilon\ge 0$, where $\mathbf{Z}_0\sim G_0$.
\end{remark}

\subsection{Proofs of Section \ref{subsec:EU}}\label{app:EU}

We present the following lemma, which will be used in the proofs of Theorems \ref{th-RegularizationEL} and \ref{th-necessity-highorderloss}.
 
\begin{lemma}\label{lm-regnecessity}
		Let $p\in(1,\infty)$, $\epsilon>0$ and $\eta\in(0,\epsilon]$.
		Define $\mathcal V=\{V\in L^p: \|V\|_p\le\epsilon,~\E[|V|\id_{\{|V|\le 2\epsilon\}}]\le \eta\}$. We have
		$
		\E[|V|]\le \epsilon 2^{-p/q}+\eta
		$
		for all $V\in\mathcal V$.
	\end{lemma}
	
	\proof{Proof.}
		Using the Chebyshev's inequality, we have
		$
		(2\epsilon)^p\p(|V|>2\epsilon)\le \|V\|_p^p,
		$
		which implies $\p(|V|>2\epsilon)\le 2^{-p}$ for all $\|V\|_p\le \epsilon$. Hence, for any $V\in\mathcal V$, it holds that
		\begin{align*}
			\E[|V|]&=\E\left[|V|\id_{\{|V|> 2\epsilon\}}\right]+\E\left[|V|\id_{\{|V|\le 2\epsilon\}}\right]\le \|V\|_p (\p(|V|> 2\epsilon))^{1/q}+\eta\le \epsilon 2^{-p/q}+\eta,
		\end{align*}
		where the first inequality follows from H\"{o}lder's inequality. This completes the proof.
	\endproof
~~

{\bf Proof of Theorem \ref{th-RegularizationEL}.} 
Throughout the proof, we assume  $C=1$ without loss of generality.
	By Lemma \ref{lm-onehigheq}, 
	it suffices to show that  (ii) holds  if and only if the following (i)' holds. 
	\begin{itemize}
		\item [(i)'] For any $Z\in L^p$ and $\epsilon\ge 0$, we have \eqref{eq-regEU} holds, i.e.,
  \begin{align}\label{eq-ELeq}
  \sup_{\|V\|_p\le \epsilon}\E[\ell(Z+V)]
  =
\E[\ell(Z)]+\epsilon,~~\forall Z\in L^p,~\epsilon\ge 0.
  \end{align}
\end{itemize}

To see (ii)$\Rightarrow$(i)', note that $\ell_1(Z+V) =\ell_1(Z)\pm V $ and thus,   
\begin{align*}
	\sup_{\|V\|_p\le \epsilon}\E[\ell_1(Z+V)]
	= \E[\ell_1(Z)]+\sup_{\|V\|_p\le \epsilon}\E[\pm V] =\E[\ell_1(Z)]+\epsilon
\end{align*}
and
\begin{align*}
	\sup_{\|V\|_p\le \epsilon}\E[\ell_2(Z+V)]
	&= \sup_{\|V\|_p\le \epsilon}\E[|Z+V-b_1|+b_2]\\[3pt]
	&\le \sup_{\|V\|_p\le \epsilon}\left\{\E[|Z-b_1|+b_2]+\E[|V|]\right\}=\E[\ell_2(Z)]+\epsilon.
\end{align*}
The inequality  reduces to equality if $V=\epsilon\ {\rm sign}(Z-b_1)
$. Hence, we complete the proof of (ii)$\Rightarrow$(i)'.

To see (i)'$\Rightarrow$(ii),  we first show that  ${\rm Lip}(\ell)\le 1$. Otherwise, by that a convex function has derivative almost everywhere, there exists $x$ such that $|\ell'(x)|>1$. If $\ell'(x)>1$, then we have 
\begin{align*}
	\sup_{\|V\|_p\le \epsilon}\E[\ell(x+V)]&\ge \ell(x+\epsilon )\ge \ell(x)+|\ell'(x)|\epsilon>\ell(x)+\epsilon.
\end{align*}
If $\ell'(x)<-1$, then we have 
\begin{align*}
	\sup_{\|V\|_p\le \epsilon}\E[\ell(x+V)]&\ge \ell(x-\epsilon )\ge \ell(x)- \ell'(x)\epsilon>\ell(x)+\epsilon.
\end{align*}
Those two cases both yield a contradiction to \eqref{eq-ELeq}. 

Next, we aim to verify the following fact
\begin{align}\label{eq-assEU}
	|\ell'(x)|=1~{\rm{for~all}}~x\in\R~{\rm as~long~as}~\ell~{\rm is~differentiable~at}~x.
\end{align}
This will complete the proof since a convex function that satisfies \eqref{eq-assEU} must be one of the forms of $\ell_1$ and $\ell_2$ with $C=1$.
To see \eqref{eq-assEU},
assume by contradiction that there exists $x_0\in\R$ such that $|\ell'(x_0)|<1$. We first consider the case  $p=\infty$. In this case, we have
\begin{align*}
	\sup_{\|V\|_\infty\le \epsilon}\E[\ell(x_0+V)] 
	=\max\{\ell(x_0-\epsilon),\ell(x_0+\epsilon)\}<\ell(x_0)+\epsilon,
\end{align*}
where
the strict inequality follows from $|\ell'(x_0)|<1$ and ${\rm Lip}(\ell)\le 1$, which yields a contradiction.
Suppose now $p\in(1,\infty)$. Define 
 $$
k=\max\left\{\left|\frac{\ell(x_0+2\epsilon)-\ell(x_0)}{2\epsilon}\right|,\left|\frac{\ell(x_0)-\ell(x_0-2\epsilon)}
{2\epsilon}\right|\right\}. 
$$
We have that    $|\ell(x_0+v)-\ell(x_0)|\le k|v|$ for all $v\in[-2\epsilon,2\epsilon]$ by the convexity of $\ell$, and thus,
\begin{align*}
	\sup_{\|V\|_p\le \epsilon}\E[\ell(x_0+V)]
	&=\sup_{\|V\|_p\le \epsilon}\{\E[\ell(x_0+V)\id_{\{|V|\le 2\epsilon\}}]+\E[\ell(x_0+V)\id_{\{|V|> 2\epsilon\}}]\}\\[3pt]
	&\le \sup_{\|V\|_p\le \epsilon}\{\E[(\ell(x_0)+k|V|)\id_{\{|V|\le 2\epsilon\}}]
	+\E[(\ell(x_0+V)\id_{\{|V|> 2\epsilon\}}]\}\\[3pt]
 &\le \sup_{\|V\|_p\le \epsilon}\{\E[(\ell(x_0)+k|V|)\id_{\{|V|\le 2\epsilon\}}]
	+\E[(\ell(x_0)+|V|)\id_{\{|V|> 2\epsilon\}}]\}\\[3pt]
	&=\ell(x_0)+\sup_{\|V\|_p\le \epsilon}\{\E[|V|]-(1-k)\E[|V|\id_{\{|V|\le 2\epsilon\}}]\}=:I,
\end{align*}
where the  second inequality follows from ${\rm Lip}(\ell)\le 1$.
By defining \begin{align}\label{eq-V2}
	\mathcal V_1=\left\{V\in L^p: \|V\|_p\le \epsilon,~\E\left[|V|\id_{\{|V|\le 2\epsilon\}}\right]\le\frac{(1-2^{-p/q})\epsilon}{2}\right\}
	,~~~
	\mathcal V_2=\{V: \|V\|_p\le \epsilon\}\setminus \mathcal V_1,
\end{align}
we can rewrite $I$ as  $I= \max\{I_1,I_2\}$ with
$$
I_i=\ell(x_0)+\sup_{V\in\mathcal V_i} \{\E[|V|]-(1-k)\E[|V|\id_{\{|V|\le 2\epsilon\}}]\},~~i=1,2.
$$
Since $|\ell'(x_0)|<1$ and ${\rm Lip}(\ell)\le 1$, we know that $k<1$. Hence, it holds that
\begin{align}\label{eq-regEUI1}
	I_1&\le \ell(x_0)+\sup_{V\in\mathcal V_1}\E[|V|]
	\le \ell(x_0)+\frac{(1+2^{-p/q})\epsilon}{2}<\ell(x_0)+\epsilon,
\end{align}
where the second inequality follows from Lemma \ref{lm-regnecessity}. 
Further note that
\begin{align}\label{eq-regEUI2}
	I_2&\le \ell(x_0)+\sup_{V\in\mathcal V_2}\left\{\E[|V|]-(1-k) \frac{(1-2^{-p/q})\epsilon}{2}\right\}\notag\\
	&\le \ell(x_0)+\epsilon-(1-k) \frac{(1-2^{-p/q})\epsilon}{2}\notag\\
 &<\ell(x_0)+\epsilon,
\end{align}
where the first inequality follows from the definition of $\mathcal V_2$, and the second one holds because $\|V\|_p\le \epsilon$ implies $\E[|V|]\le \epsilon$.
Combining  \eqref{eq-regEUI1} and \eqref{eq-regEUI2}, we conclude that
$$
\sup_{\|V\|_p\le \epsilon}\E[\ell(x_0+V)]\le \max\{I_1,I_2\}<\ell(x_0)+\epsilon.
$$
This yields a contradiction. Hence,  \eqref{eq-assEU} is verified, and thus we complete the proof.
\endproof
~~

To prove Theorem \ref{th-necessity-highorderloss}, we begin by introducing two auxiliary lemmas. Specifically, Lemma \ref{lm-supmain} will be employed in the proof of Theorem \ref{th-necessity-highorderloss}, and Lemma \ref{lem:v2-1} serves to establish Lemma \ref{lm-supmain}.

\begin{lemma} \label{lem:v2-1} 
For $p\in(1,\infty)$, $\epsilon>0$ and $\eta\in (0,\epsilon)$, the following problem has a positive optimal value.
\begin{align} \label{eq:1015-2}
    \min &~~\E [(Z-\epsilon)_+]\\
   {\rm s.t.} &~~\E[Z]=\epsilon, ~~(\E[Z^{1/p}])^{p} \le \epsilon-\eta,~~Z\ge0. \notag
\end{align} 
\end{lemma}
\begin{proof}{Proof.} 
Note that Problem \eqref{eq:1015-2} has two moment constraints. By \cite{S95}, the optimal value of Problem \eqref{eq:1015-2} is equivalent to 
\begin{align} \label{eq:1016-2}
    \min &~~\E [(Z-\epsilon)_+]\\
   {\rm s.t.} &~~\E[Z]=\epsilon, ~~(\E[Z^{1/p}])^{p} \le \epsilon-\eta,~~Z\ge0,~~Z\in\mathcal X_3, \notag
\end{align} 
where $\mathcal X_3$ is the set of all random variables that take at most three values. Let $Z$ be a feasible solution of Problem \eqref{eq:1016-2}. 
Denote by $\overline{z}=\esssup Z$, and let $Z^*$ be a random variable such that
\begin{align}\label{eq:1015-1} 
Z^* \sim (p_1+p_2)\delta_{\epsilon} + q\delta_{\overline{z}}+(1-(p_1+p_2+q))\delta_0
\end{align}
with the conditions that $p_1 \epsilon = \E[Z\id_{\{Z<\epsilon\}}]$,  $ p_2 \epsilon + \overline{z} q = \E[Z\id_{\{Z\ge \epsilon\}}]$ and $p_2 =\p(Z\ge \epsilon)-q $.
It is straightforward to check that $\E[Z^*]=\E[Z]=\epsilon$ and 
$\E[(t-Z)_+]\le \E[(t-Z^*)_+]$ for all $t\in\R$. It follows from Theorem 3.A.1 of \cite{SS07} that $Z^*\le_{\rm cv} Z$.\footnote{For two random variables $Z_1$ and $Z_2$, $Z_1$ is said to be smaller than $Z_2$ in the concave order, denoted by $Z_1\le_{\rm cv} Z_2$, if $\E[\phi(Z_1)]\le \E[\phi(Z_2)]$ for any concave function $\phi$.} This implies that $ \E[(Z^*)^{1/p}] \le   \E[Z^{1/p}]$ since $x\mapsto x^{1/p}$ is a concave function on $\R_+$. Therefore, $Z^*$ is also a feasible solution of Problem \eqref{eq:1016-2}. In addition, we have 
\begin{align*}
\E [(Z-\epsilon)_+]
=\E[Z\id_{\{Z\ge \epsilon\}}]-\epsilon\p(Z\ge \epsilon)
=p_2\epsilon+\overline{z}q-\epsilon(p_2+q)=q(\overline{z}-\epsilon)=\E [(Z^*-\epsilon)_+].
\end{align*}
Hence, to solve Problem \eqref{eq:1016-2}, it suffices to consider all random variables whose distributions take the form of \eqref{eq:1015-1}, i.e., $p_1\delta_0+p_2\delta_{\epsilon}+p_3\delta_x$ with $p_1,p_2,p_3\ge 0$, $p_1+p_2+p_3=1$ and $x\ge \epsilon$. The corresponding problem can be represented as follows:
\begin{align}\label{eq-1016-4}
    \min_{p_i\ge 0, {x}} &~~( {x}-\epsilon)p_3 ~~~
   {\rm s.t.}  ~~\epsilon p_2 +  {x} p_3 =\epsilon, 
     ~~\epsilon^{1/p} p_2 + {x}^{1/p} p_3 \le (\epsilon-\eta)^{1/p},~~  {x}\ge \epsilon,~~p_1+p_2+p_3= 1. 
\end{align} 
Note that  $({x}-\epsilon)p_3  =\epsilon(1-p_2) -\epsilon p_3 = \epsilon p_1$ whenever $(p_1,p_2,p_3,x)$ is a feasible solution of Problem \eqref{eq-1016-4},
 and thus, Problem \eqref{eq-1016-4} is equivalent to
\begin{align*}
    \min_{p_i\ge 0,  {x}} &~~ \epsilon p_1~~~
   {\rm s.t.}  ~~\epsilon p_2 + {x} p_3 =\epsilon, 
     ~~\epsilon^{1/p} p_2 +  {x}^{1/p} p_3 \le (\epsilon-\eta)^{1/p},~~  {x}\ge \epsilon,~~p_1+p_2+p_3=1. 
\end{align*} 
Further, letting $x=\lambda\epsilon$ yields the following equivalent problem
\begin{align}\label{eq-1016-3}
    \min_{p_i\ge 0, \lambda} &~~ \epsilon p_1~~~
   {\rm s.t.}  ~~ p_2 + \lambda p_3 =1, 
     ~~ p_2 +  {\lambda}^{1/p} p_3 \le \left(1-\frac{\eta}{\epsilon}\right)^{1/p},~~  {\lambda}\ge 1,~~p_1+p_2+p_3=1. 
\end{align} 
Let $(p_1,p_2,p_3,\lambda)$ be a feasible solution
of Problem \eqref{eq-1016-3}. It holds that $p_1=(\lambda-1)p_3$ and $p_2=1-\lambda p_3$ and $\lambda\ge 1$. Substituting $p_2=1-\lambda p_3$ into the second constraint of Problem \eqref{eq-1016-3} yields
\begin{align*}
\left(1-\frac{\eta}{\epsilon}\right)^{1/p}\ge 1-\lambda p_3+\lambda^{1/p}p_3=-(\lambda-1)p_3-p_3+1+\lambda^{1/p}p_3
=-p_1+(\lambda^{1/p}-1)p_3+1
\ge -p_1+1,
\end{align*}
which implies $p_1\ge 1-(1-\eta/\epsilon)^{1/p}$. Hence, the optimal value of Problem \eqref{eq-1016-3} is no less than $\epsilon-\epsilon(1-\eta/\epsilon)^{1/p}$ which is a positive value. This completes the proof.
\endproof
\end{proof}

\begin{lemma}\label{lm-supmain}
Let $p\in(1,\infty)$, $t,\,\epsilon>0$ and $\eta\in(0,\epsilon)$. For $V\in L^p$, the following statements hold.
\begin{itemize}
\item[(i)] If $\|V\|_p\le \epsilon$, then $\E[(|V|+t)^p]\le (\epsilon+t)^p$.

\item[(ii)] If $\|V\|_p\le \epsilon$ and $\E[|V|]\le \epsilon-\eta$, then there exists $\Delta>0$ that only depends on $p,t,\epsilon,\eta$ such that $\E[(|V|+t)^p]\le (\epsilon+t)^p-\Delta$. In particular, if $p$ is an integer, then $\E[(|V|+t)^p]\le (\epsilon+t)^p-pt^{p-1}\eta$. 
\end{itemize}
\end{lemma}

\proof{Proof.}
(i) Suppose that $\|V\|_p\le \epsilon$ and denote by $Z=|V|^p$. It holds that $\E[Z]\le \epsilon^p$ and
\begin{align*}
\E[(|V|+t)^p]=\E[(Z^{1/p}+t)^p]\le ((\E[Z])^{1/p}+t)^p\le (\epsilon+t)^p,
\end{align*}
where we have used Jensen's inequality in the first inequality by noting that $x\mapsto (x^{1/p}+t)^p$ is concave on $\R_+$.

(ii) The case that $p$ is an integer can be verified by the following direct calculations:
\begin{align*}
		\E[(|V|+t)^p]&=\E\left[\sum_{i=0}^p {p\choose i} t^{p-i}|V|^i\right] \\
  &=\sum_{i\neq 1} {p\choose i} t^{p-i} \E[|V|^i]+pt^{p-1}\E[|V|]\\
		&\le \sum_{i\neq 1} {p\choose i} t^{p-i}\epsilon^i+pt^{p-1}(\epsilon-\eta) \\
  &=\sum_{i=0}^p {p\choose i} t^{p-i}\epsilon^i-pt^{p-1}\eta
  =(\epsilon+t)^p-pt^{p-1}\eta,
\end{align*}
where the inequality holds because  $\|V\|_p\le \epsilon$ implies $\|V\|_i\le \epsilon$ for $i=1,2,\cdots,p$.

For general case $p$,  it suffices to verify that the optimal value of the  following problem  is  smaller than $(\epsilon+t)^p$:
\begin{align}\label{eq-1016-21}
\sup \,\E[(Z+t)^p]~~{\rm s.t.}~~\E[Z^p]\le \epsilon^p,~~\E[Z] \le \epsilon -\eta,~~Z\ge 0.
\end{align}
We first assert that Problem \eqref{eq-1016-21} is equivalent to 
\begin{align} \label{eq:1017-1}
 \sup \,\E[(Z+t)^p]~~{\rm s.t.}~~\E[Z^p]= \epsilon^p,~~\E[Z] \le \epsilon -\eta,~~Z\ge 0.
\end{align}
To see this, 
let $Z$ be a feasible solution of Problem \eqref{eq-1016-21} with $\E[Z]>0$. 
Let $U$ be a uniform random variable on $(0,1)$ 
and for $\alpha\in (0,1)$, define
\begin{align*}
Z_{\alpha}=(F_{Z}^{-1}(U)+f(\alpha))\id_{\{U\ge \alpha\}},~~\alpha\in[0,1),~~{\rm where}~f(\alpha)=\frac{\int_0^{\alpha}F_Z^{-1}(s)\d s}{1-\alpha}.
\end{align*}
It is straightforward to check that $\E[Z_{\alpha}]=\E[Z]$ for all $\alpha\in[0,1)$. 
Moreover, it holds that
\begin{align}\label{eq-1016-22}
\E[(Z_{\alpha})^p] 
\ge \E[f(\alpha)^p\id_{\{U\ge \alpha\}}]
=\frac{\left(\int_0^{\alpha}F_Z^{-1}(s)\d s\right)^p}{(1-\alpha)^{p-1}}\to \infty~~{\rm as}~\alpha\uparrow 1.
\end{align}
Note that the mapping $\alpha\mapsto \E[(Z_{\alpha})^p]$ is continuous as $f(\alpha)$ is continuous, and $\E[(Z_{\alpha})^p]=\E[Z^p]\le \epsilon$ if $\alpha=0$. Combining with \eqref{eq-1016-22}, there exists $\alpha^*\in[0,1)$ such that $\E[Z_{\alpha^*}^p]=\epsilon^p$. 
Moreover, one can check that for any $\alpha\in[0,1)$, $\int_{\gamma}^1 F_Z^{-1}(s)\d s\le \int_{\gamma}^1 F_{Z_{\alpha}}^{-1}(s)\d s$ for all $\gamma\in(0,1)$. It follows from Theorem 2.5 of \cite{BB06} that $Z_{\alpha}\le_{\rm cv}Z$. 
Hence, we have $\E[(Z+t)^p]\le \E[(Z_{\alpha^*}+t)^p]$ as $x\mapsto -(x+t)^p$ is a concave function on $\R_+$.  Therefore, Problem \eqref{eq-1016-21} is equivalent to Problem \eqref{eq:1017-1}, 
which is further equivalent to
\begin{align}\label{eq:1015-3}
 \sup \,\E[g(Z)]~~{\rm s.t.}~~\E[Z ]=\epsilon^p,~~\E[Z^{1/p}] \le \epsilon -\eta,~~Z\ge 0,  
\end{align}
where $g(x) = (x^{1/p}+t)^p$ is a strictly concave function on $\R_+$. 
It remains to verify that the optimal value of Problem \eqref{eq:1015-3} is strictly less than $g(\epsilon^p)$.

To see this, let $Z$ be a feasible solution of Problem \eqref{eq:1015-3}. 
Define 
$$Z^* = z_1 \id_{\{Z <\epsilon^p\}} + z_2 \id_{\{Z \ge \epsilon^p\}},$$
where $z_1 =\E[Z|Z<\epsilon^p]$ and $z_2 = \E[Z|Z\ge \epsilon^p]$ satisfy $z_1\lambda + z_2 (1-\lambda) =\epsilon^p$ with $\lambda = \p(Z <\epsilon^p)$. 
Note that by Lemma \ref{lem:v2-1}, there exists $\Delta>0$ that only depends on $p,t,\epsilon,\eta$ such that $\E[(\epsilon^p-Z)_+]=\E[(Z-\epsilon^p)_+]\ge \Delta$. By definition of $z_1$ and $z_2$, this implies 
$
\Delta\le 
 (\epsilon^p-z_1) \p(Z< \epsilon^p)\le \epsilon^p-z_1
$ and $\Delta\le 
 (z_2-\epsilon^p)\p(Z\ge \epsilon^p)\le z_2-\epsilon^p
 $. Therefore, we have 
\begin{align}\label{eq-1016-23}
z_1\le \epsilon^p-\Delta~~{\rm and}~~z_2\ge \epsilon^p+\Delta.
\end{align}
Further note that $Z\le_{\rm cv} Z^*$ and also note that  $g$ is  concave on $\R_+$. This implies
\begin{align*}
\E[g(Z)]  \le \E[g(Z^*)]=\lambda g(z_1)+(1-\lambda)g(z_2) 
 \le  \frac{g(\epsilon^p-\Delta)+ g(\epsilon^p+\Delta)}2 <g(\epsilon^p) ,
 \end{align*}
where  the second inequality is due to the concavity of $g$,  $\lambda z_1+(1-\lambda) z_2 = (\epsilon^p-\Delta + \epsilon^p +\Delta )/2$ and \eqref{eq-1016-23}; and 
the strict inequality is due to $\Delta>0$ and the strict concavity of $g$. Noting that $\Delta$ is independent of the random variable $Z$, this completes the proof.

 \endproof

~~

\paragraph{Proof of Theorem \ref{th-necessity-highorderloss}.} 

\underline{(ii) $\Rightarrow$ (i)}:
By Lemma \ref{lm-onehigheq}, it suffices to show that \eqref{eq-regEU} holds for any $Z\in L^p$ and $\epsilon\ge0$. 
For $\ell=\ell_1$, we have
\begin{align}\label{eq-highl1}
	\sup_{\|V\|_p\le \epsilon}\|\ell_1(Z+V)\|_p
	&=\sup_{\|V\|_p\le \epsilon}\|(Z+V-b_1)_+\|_p\notag\\[3pt]
 &\le \sup_{\|V\|_p\le \epsilon}\|(Z-b_1)_++|V|\|_p\notag\\[3pt]
	&\le \sup_{\|V\|_p\le \epsilon}\left\{\|(Z-b_1)_+\|_p+\|V\|_p\right\}\notag\\[3pt]
 &\le \|(Z-b_1)_+\|_p+\epsilon=\|\ell_1(Z)\|_p+\epsilon.
\end{align}
If $\p(Z>b_1)>0$, then
all the inequalities are equalities, and the maximizer can be chosen as
$V=\lambda(Z-b_1)_+$ with some $\lambda> 0$ such that $\|V\|_p=\epsilon$. If $Z\le b_1$ a.s., then we take $\{V_n\}_{n\in\N}$ such that $V_n$ has distribution $(1-1/n^p)\delta_0+(1/n^p)\delta_{n\epsilon}$, and $\{V_n=n\epsilon\}\subseteq\{Z\ge F^{-1}_Z(1-1/n^p)\}$ where $F_Z^{-1}$ is the left-quantile function of $Z$
for all $n$. We have $\|V_n\|_p=\epsilon$ and
\begin{align*}
	\sup_{\|V\|_p\le \epsilon}\|\ell_1(Z+V)\|_p
	&\ge \|(Z+V_n-b_1)_+\|_p\\[3pt]
 &\ge  \frac{1}{n}\left(n\epsilon+F_{Z}^{-1}\left(1-\frac{1}{n^p}\right)-b_1\right)_+
	\to \epsilon=\|\ell_1(Z)\|_p+\epsilon,
	~~{\rm as~}n\to\infty.
\end{align*}
Combining with \eqref{eq-highl1}, we verify that
 \eqref{eq-regEU} holds with $\ell_1$
for all $Z\le b_1$ a.s. This completes the proof of the case $\ell=\ell_1$.
For $\ell=\ell_2$, the proof is similar to that of $\ell_1$.
For $\ell=\ell_3$, we have
\begin{align}\label{eq-highl3}
	\sup_{\|V\|_p\le \epsilon}\|\ell_3(Z+V)\|_p
	&=\sup_{\|V\|_p\le \epsilon}\|(|Z+V-b_1|-b_2)_+\|_p\notag\\[3pt]
 &\le \sup_{\|V\|_p\le \epsilon}\|(|Z-b_1|-b_2+|V|)_+\|_p\notag\\[3pt]
	&\le \sup_{\|V\|_p\le \epsilon}\|(|Z-b_1|-b_2)_++|V|\|_p\notag\\[3pt]
 &\le \sup_{\|V\|_p\le \epsilon}\left\{\|(|Z-b_1|-b_2)_+\|_p+\|V\|_p\right\}\notag\\[3pt]
	&\le \|(|Z-b_1|-b_2)_+\|_p+\epsilon=\|\ell_3(Z)\|_p+\epsilon.
\end{align}
If $\p(|Z-b_1|>b_2)>0$, then
one can check that
all inequalities are equalities, and the maximizer can be chosen as
$V=\lambda(|Z-b_1|-b_2)_+{\rm sign}(Z-b_1)$ with some $\lambda\ge 0$ such that $\|V\|_p=\epsilon$. If $|Z-b_1|\le b_2$ a.s., then we have $Z\in[b_1-b_2,b_1+b_2]$. Taking a sequence $\{V_n\}_{n\in\N}$ as shown in the case of $\ell_1$, i.e., $V_n$ with distribution $(1-1/n^p)\delta_0+(1/n^p)\delta_{n\epsilon}$, it holds that for large enough $n$ such that $n\epsilon\ge \max\{b_1,b_2\}$,
\begin{align*}
	\sup_{\|V\|_p\le \epsilon}\|\ell_3(Z+V)\|_p
	&\ge \|(|Z+V_n-b_1|-b_2)_+\|_p\\
	&\ge  \left(\E[\left(|Z-b_1|-b_2\right)_+^p\id_{\{V_n=0\}}]
	+\E[\left(|Z-b_1+n\epsilon|-b_2\right)_+^p\id_{\{V_n=n\epsilon\}}]\right)^{1/p}\\
	&\ge \left(\E[\left(n\epsilon-b_2-b_2\right)_+^p\id_{\{V_n=n\epsilon\}}]\right)^{1/p}\\
	&=\left(\epsilon-\frac{2b_2}{n}\right)_+\to \epsilon=\|\ell_3(Z)\|_p+\epsilon,~~{\rm as}~n\to\infty.
\end{align*}
Combining with \eqref{eq-highl3}, we  have that
\eqref{eq-regEU} holds with $\ell_3$
for all $|Z-b_1|\le b_2$ a.s. This completes the proof of the case $\ell_3$.
For $\ell=\ell_4$, we have
\begin{align*}
	\sup_{\|V\|_p\le \epsilon}\|\ell_4(Z+V)\|_p
	&=\sup_{\|V\|_p\le \epsilon}\||Z+V-b_1|+b_2\|_p\\[3pt]
 &\le \sup_{\|V\|_p\le \epsilon}\||Z-b_1|+b_2+|V|\|_p\\[3pt]
	&\le \sup_{\|V\|_p\le \epsilon}\left\{\||Z-b_1|+b_2\|_p+\|V\|_p\right\}\\[3pt]
 &\le \||Z-b_1|+b\|_p+\epsilon=\|\ell_4(Z)\|_p+\epsilon,
\end{align*}
where all inequalities can be equality, and the maximizer can be chosen as $V=\lambda (|Z-b_1|+b_2){\rm sign}(Z-b_1)$ for some $\lambda\ge 0$ such that $\|V\|_p=\epsilon$. Hence, we conclude that  \eqref{eq-regEU} holds with $\ell=\ell_4$ and  complete the proof of (ii) $\Rightarrow$ (i).

\underline{(i) $\Rightarrow$ (ii)}:
The proof of this direction is in the same spirit to that of Theorem \ref{th-RegularizationEL}.  Assume without loss of generality that $C=1$.
By Lemma \ref{lm-onehigheq}, we have for any $Z\in L^p$ and $\epsilon\ge 0$,
\begin{align}\label{eq-reghighEU1}
	\sup_{\|V\|_p\le \epsilon}\|\ell(Z+V)\|_p=\|\ell(Z)\|_p+\epsilon.
\end{align}
By the same arguments in the proof of Theorem \ref{th-RegularizationEL}, we can show that ${\rm Lip}(\ell)\le 1$. 
Next, assume that $\ell$ is differential at $x$ when we use the notation $\ell'(x)$, and we show the following facts.
\begin{align}\label{eq-asshigh}
	{\rm If}~|\ell'(x)|>0,~{\rm then}~ |\ell'(x)|=1.~~~{\rm If}~ \ell'(x)=0,~{\rm then}~ \ell(x)=0.
\end{align}
This will complete the proof since one can check that \eqref{eq-asshigh} implies that $\ell$ has one of the forms of $\ell_1$, $\ell_2$, $\ell_3$ and $\ell_4$ with $C=1$.
To see \eqref{eq-asshigh},
we assume by contradiction that there exists $x_0\in\R$ such that $|\ell'(x_0)|<1$ and $\ell(x_0)>0$ (note that $|\ell'(x_0)|\in(0,1)$ implies $\ell(x_0)>0$). 
Define 
$$
k=\max\left\{\left|\frac{\ell(x_0+2\epsilon)-\ell(x_0)}{2\epsilon}\right|,\left|\frac{\ell(x_0)-\ell(x_0-2\epsilon)}{2\epsilon}\right|\right\},
$$
and it holds that $0\le \ell(x_0+v)\le \ell(x_0)+k|v|$ for all $v\in[-2\epsilon,2\epsilon]$ as $\ell$ is nonnegative and convex. Further, noting that $|\ell'(x_0)|<1$ and ${\rm Lip}(\ell)\le 1$, we have $k<1$ and
\begin{align}\label{eq-highcx}
(\ell(x_0)+|v|)^p- (\ell(x_0)+k|v|)^p \ge 
p\ell^{p-1}(x_0)(1-k)|v|,~\forall v\in\R.
\end{align}
Therefore,
\begin{align}
\sup_{\|V\|_p\le \epsilon}\E[\ell^p(x_0+V)]
	&=	\sup_{\|V\|_p\le \epsilon}\left\{\E[\ell^p(x_0+V)\id_{\{|V|\le 2\epsilon\}}]+\E[\ell^p(x_0+V)\id_{\{|V|> 2\epsilon\}}]\right\}\notag\\[3pt]
&\le	\sup_{\|V\|_p\le \epsilon}\left\{ \E[(\ell(x_0)+k|V|)^p\id_{\{|V|\le 2\epsilon\}}]+\E[\ell^p(x_0+V)\id_{\{|V|> 2\epsilon\}}]\right\}\notag\\[3pt]
&\le \sup_{\|V\|_p\le \epsilon}\left\{ \E[(\ell(x_0)+k|V|)^p\id_{\{|V|\le 2\epsilon\}}]+\E[(\ell(x_0)+|V|)^p\id_{\{|V|> 2\epsilon\}}]\right\}\label{eq-high1}\\[3pt]
	&= 	\sup_{\|V\|_p\le \epsilon}\left\{\E[(\ell(x_0)+|V|)^p]-\E[((\ell(x_0)+|V|)^p-(\ell(x_0)+k|V|)^p)\id_{\{|V|\le 2\epsilon\}}]\right\}\notag\\[3pt]
	&\le \sup_{\|V\|_p\le \epsilon}\left\{\E[(\ell(x_0)+|V|)^p]-p\ell^{p-1}(x_0)(1-k)\E[|V|\id_{\{|V|\le 2\epsilon\}}]\right\}
	\label{eq-high2}\\[3pt]
	&=:I\notag,
\end{align}
where \eqref{eq-high1} holds because $\ell$ is nonnegative with ${\rm Lip}(\ell)\le 1$, and \eqref{eq-high2} follows from \eqref{eq-highcx}.
Recalling $\mathcal V_1$ and $\mathcal V_2$ defined by \eqref{eq-V2} in the proof of Theorem \ref{th-RegularizationEL}, we can rewrite $I=\max\{I_1,I_2\}$ with
\begin{align*}
I_i=\sup_{V\in\mathcal V_i}\left\{\E[(\ell(x_0)+|V|)^p]-p\ell^{p-1}(x_0)(1-k)\E[|V|\id_{\{|V|\le 2\epsilon\}}]\right\},~~i=1,2.
\end{align*}
By Lemma \ref{lm-regnecessity} and the definition of $\mathcal V_1$, we have 
\begin{align}\label{eq-lmEC5}
\E[|V|]\le \epsilon 2^{-p/q}+\frac{(1-2^{-p/q})\epsilon}{2}=\frac{(1+2^{-p/q})\epsilon}{2}<\epsilon,~~\forall~V\in\mathcal V_1.
\end{align}
It holds that
\begin{align}\label{eq-highEUI1}
	I_1&\le \sup_{V\in\mathcal V_1}\E[(\ell(x_0)+|V|)^p]
 <(\ell(x_0)+\epsilon)^p,
\end{align}
where the strict inequality follows from \eqref{eq-lmEC5} and Statement (ii) of Lemma \ref{lm-supmain} by noting that $\ell(x_0)>0$. For $I_2$, we have
\begin{align}\label{eq-highEUI2}
	I_2&\le \sup_{V\in\mathcal V_2}\E[(\ell(x_0)+|V|)^p]-\inf_{V\in\mathcal V_2}p\ell^{p-1}(x_0)(1-k)\E[|V|\id_{\{|V|\le 2\epsilon\}}]\notag\\[3pt]
	&\le  (\ell(x_0)+\epsilon)^p-p\ell^{p-1}(x_0)(1-k)\inf_{V\in\mathcal V_2}\E[|V|\id_{\{|V|\le 2\epsilon\}}]\notag\\[3pt]
	&\le (\ell(x_0)+\epsilon)^p-p\ell^{p-1}(x_0)(1-k)\frac{(1-2^{-p/q})\epsilon}{2}\notag\\[3pt]
 &<(\ell(x_0)+\epsilon)^p,
\end{align}
where the second inequality follows from Statement (i) of Lemma \ref{lm-supmain}, and the third inequality is due to the definition of $\mathcal V_2$.
Combining  \eqref{eq-highEUI1} and \eqref{eq-highEUI2}, we have
\begin{align*}
	\sup_{\|V\|_p\le \epsilon}\|\ell(x_0+V)\|_p\le I^{1/p}= \max\{I_1^{1/p},I_2^{1/p}\}<\ell(x_0)+\epsilon,
\end{align*}
which yields a contradiction to \eqref{eq-reghighEU1}. Hence,  \eqref{eq-asshigh} holds, which completes the proof.
\endproof
~~

As outlined in Remark \ref{re-SADK23}, 
we detail below the steps to verify that for the function $\ell$ listed in (ii) of Theorems \ref{th-RegularizationEL} and  \ref{th-necessity-highorderloss}, 
 the regularization version given by Proposition 3.9 of \cite{SADK23} can be reformed as the corresponding regularization version in our framework.
To be specific, Proposition 3.9 of \cite{SADK23} states that for $\ell$ that is proper, upper-semicontinuous and satisfies $\ell(x)\le C(1+|x|^p)$ for some $C>0$ and all $x\in\R$,
\begin{align*}
\sup_{F\in\mathcal B_p(G_0,\epsilon)} \E^F[\ell(\bm\beta^\top \mathbf Z)]=\inf_{\lambda\ge 0} \left\{\E^{G_0}[\ell_p(\bm\beta^\top \mathbf Z,\lambda)]+\lambda \epsilon^p\|\bm\beta\|_*^p\right\},
\end{align*}
where $\ell_p(x,\lambda)=\sup_{y\in\R} \{\ell(y)-\lambda |y-x|^p\}$ and $\|\cdot\|_*$ is the dual norm of $\|\cdot\|$. 
When $\ell$ is given  by (ii) of Theorem \ref{th-RegularizationEL}, we have 
\begin{align}\label{eq-eqreg-1}
\inf_{\lambda\ge 0} \left\{\E^{G_0}[\ell_p(\bm\beta^\top \mathbf Z,\lambda)]+\lambda \epsilon^p\|\bm\beta\|_*^p\right\} = \E^{G_0}[\ell (\bm\beta^\top \mathbf Z)]+C\epsilon\|\bm\beta\|_*;
\end{align}
when $\ell$  is given  by (ii) of Theorem \ref{th-necessity-highorderloss}, we have
\begin{align}\label{eq-eqreg-1-Thm4}
\inf_{\lambda\ge 0} \left\{\E^{G_0}[(\ell^p)_p(\bm\beta^\top \mathbf Z,\lambda)]+\lambda \epsilon^p\|\bm\beta\|_*^p\right\} =  \left(\left(\mathbb{E}^{G_0}[\ell^p(\bb^\top \mathbf Z)]\right)^{1/p} + C\epsilon||\bb||_*\right)^p.
\end{align}
\proof{Proof.}
We only give the proof of \eqref{eq-eqreg-1} for the case $\ell(x)=|x-b|$ for some $b\in\R$ and the proof of \eqref{eq-eqreg-1-Thm4} for $\ell(x)=(|x-b_1|-b_2)_+^p$ for some $b_1\in\R$ and $b_2\ge 0$
as the other cases can be proved similarly. 

In the case that $\ell(x)=|x-b|$,  the function $\ell_p$ reduces to $\ell_p(x,\lambda)=\sup_{y\in\R} f_{\lambda,x}(y)$, where
  $f_{\lambda,x}(y):=|y-b|-\lambda|y-x|^p$ for $\lambda\ge 0$ and $x,y\in\R$.
We first show 
\begin{align}\label{eq-ell_p}
\ell_p(x,\lambda)=|x-b|+(\lambda p)^{-1/(p-1)}-\lambda(\lambda p)^{-p/(p-1)}
\end{align}
by considering the following two cases: 
\begin{itemize}\item [(i)]
If $x\ge b$, then we have $\ell_p(x,\lambda)=\sup_{y\in\R}f_{\lambda,x}(y)=\max\{I_1,I_2,I_3\}$ where 
\begin{align*}
I_1=\sup_{y\le b}f_{\lambda,x}(y),~~I_2=\sup_{b<y<x}f_{\lambda,x}(y)~~{\rm and}~~ I_3=\sup_{y\ge x}f_{\lambda,x}(y).
\end{align*} 
For $I_1$, by calculating the derivative of $f_{\lambda,x}(y)$ for $y < b $, 
one can verify that  
\begin{align*}
I_1= 
f_{\lambda,x}(b)
=-\lambda|b-x|^p 
\end{align*}
if $x\ge b+(\lambda p)^{-1/(p-1)}$,  and  
\begin{align*}
I_1 =
f_{\lambda,x}\left(x-(\lambda p)^{-1/(p-1)}\right)
=b-x+(\lambda p)^{-1/(p-1)}-\lambda(\lambda p)^{-p/(p-1)}, 
\end{align*}
otherwise. 
For $I_2$, since $f_{\lambda,x}$ is increasing on $[b,x]$, we have $$I_2 = f_{\lambda,x}(b)\le I_3.$$ 
Similarly,  by calculating the derivative of $f_{\lambda,x}(y)$ for $y >x$, we have
\begin{align*}
I_3=f_{\lambda,x}\left(x+(\lambda p)^{-1/(p-1)}\right)
=x-b+(\lambda p)^{-1/(p-1)}-\lambda(\lambda p)^{-p/(p-1)}.
\end{align*}
Note that $x\ge b$. Comparing $I_1$, $I_2$ and $I_3$, it is straightforward to check that $I_1\le I_3$ and $I_2\le I_3$, and hence,
\begin{align*}
\ell_p(x,\lambda) 
=I_3
=x-b+(\lambda p)^{-1/(p-1)}-\lambda(\lambda p)^{-p/(p-1)}.
\end{align*}

\item [(ii)] If $x<b$,  then by similar analysis, we obtain
\begin{align*}
\ell_p(x,\lambda)
=f_{\lambda,x}\left(x-(\lambda p)^{-1/(p-1)}\right)=b-x+(\lambda p)^{-1/(p-1)}-\lambda(\lambda p)^{-p/(p-1)}.
\end{align*}
\end{itemize}
Combining the above two cases, we have \eqref{eq-ell_p} holds. 
Substituting \eqref{eq-ell_p} into  the left-hand side of \eqref{eq-eqreg-1} yields
\begin{align*}
&\inf_{\lambda\ge 0} \left\{\E^{G_0}[\ell_p(\bm\beta^\top \mathbf Z,\lambda)]+\lambda \varepsilon^p\|\bm\beta\|_*^p\right\}\\
&=\inf_{\lambda\ge 0} \left\{\E^{G_0}[|\bm\beta^\top\mathbf Z-b|]+(\lambda p)^{-1/(p-1)}-\lambda(\lambda p)^{-p/(p-1)}+\lambda \varepsilon^p\|\bm\beta\|_*^p\right\}\\
&=\E^{G_0}[\ell(\bm\beta^\top \mathbf Z)]+\inf_{\lambda\ge 0}\{p^{-p/(p-1)}(p-1)\lambda^{-1/(p-1)}+\lambda\varepsilon^p\|\bm\beta\|_*^p\}\\
&=\E^{G_0}[\ell(\bm\beta^\top \mathbf Z)]+\varepsilon\|\bm\beta\|_*.
\end{align*}
This completes the proof of \eqref{eq-eqreg-1}. 

In the case that $\ell(x)=(|x-b_1|-b_2)_+ $, that is,  $\ell^p(x)=(|x-b_1|-b_2)_+^p$, we have 
\begin{align*}
(\ell^p)_p(x,\lambda)=\sup_{y\in\R} f_{\lambda,x}(y)
=\max\{I_1,I_2,I_3\},
\end{align*}
where $ f_{\lambda,x}(y)=(|y-b_1|-b_2)_+^p-\lambda|y-x|^p$ and 
\begin{align*}
I_1=\sup_{y\le b_1-b_2}f_{\lambda,x}(y),~~I_2=\sup_{b_1-b_2<y<b_1+b_2}f_{\lambda,x}(y)~~{\rm and}~~ I_3=\sup_{y\ge b_1+b_2}f_{\lambda,x}(y).
\end{align*}  
It is straightforward to check that
$(\ell^p)_p(x,\lambda)=\infty$ if $\lambda<1$. Hence, the constraint of $\lambda$ can be replaced by $\lambda\ge 1$. Let now $\lambda \ge 1$ and denote by $k=\lambda^{1/(p-1)}$. By considering the first-order condition, we have
\begin{align*}
I_1=\begin{cases}
-k^{p-1}(x-b_1+b_2)^p,~~&x>b_1-b_2,\\[5pt]
\left(\frac{k}{k-1}\right)^{p-1}(b_1-b_2-x)^p,~~&x\le b_1-b_2,
\end{cases}
\end{align*}
\begin{align*}
I_2=\begin{cases}
0,~~&b_1-b_2<x<b_1+b_2,\\
-k^{p-1}(b_1-b_2-x)^p,~~&x\le b_1-b_2,\\
-k^{p-1}(x-b_1-b_2)^p,~~&x\ge b_1+b_2,
\end{cases}
\end{align*}
and 
\begin{align*}
I_3=\begin{cases}
-k^{p-1}(b_1+b_2-x)^p,~~&x<b_1+b_2,\\[5pt]
\left(\frac{k}{k-1}\right)^{p-1}(x-b_1-b_2)^p,~~&x\ge b_1+b_2
\end{cases}
\end{align*}
where we use the convention that $s/0=\infty$ if $s>0$.
Comparing $I_1$, $I_2$ and $I_3$, 
we have
\begin{align*}
(\ell^p)_p(x,\lambda)
=\left(\frac{k}{k-1}\right)^{p-1}(|x-b_1|-b_2)_+^p~~{\rm with}~k=\lambda^{1/(p-1)}.
\end{align*}
Denote by $A=\E^{G_0}[\ell^p(\bb^{\top}\mathbf Z)]=\E^{G_0}[(|\bb^{\top}\mathbf Z-b_1|-b_2)_+^p]$.
Substituting the above equation into the left-hand side of \eqref{eq-eqreg-1-Thm4} and noting that the constraint of $k$ can be replaced by $k\ge 1$ yield 
\begin{align*}
\inf_{\lambda\ge 0} \left\{\E^{G_0}[(\ell^p)_p(\bm\beta^\top \mathbf Z,\lambda)]+\lambda \varepsilon^p\|\bm\beta\|_*^p\right\}
&=\inf_{k\ge 1} \left\{\left(\frac{k}{k-1}\right)^{p-1}\E^{G_0}[ (|\bb^{\top}\mathbf Z-b_1|-b_2)_+^p ]+k^{p-1} \varepsilon^p\|\bm\beta\|_*^p\right\}\\
&=\inf_{k\ge 1} \left\{\left(\frac{k}{k-1}\right)^{p-1}A+k^{p-1}\varepsilon^p\|\bm\beta\|_*^p\right\}\\
&=\left(A^{1/p}+\epsilon\|\bb\|_*\right)^p=\left(\left(\mathbb{E}^{G_0}[\ell^p(\bb^\top \mathbf Z)]\right)^{1/p} + \epsilon||\bb||_*\right)^p.
\end{align*}
This completes the proof of \eqref{eq-eqreg-1-Thm4}.
\endproof
~~

The proof of Corollary \ref{cor1} relies on the following lemma and the regularization results in Theorem \ref{th-highorderloss}. Recall that
\begin{align*}
	\pi_{1,\ell}(F,t)=t+\left(\E^F[\ell^p(\bb^\top \mathbf Z,t)]\right)^{1/p}~~{\rm and}~~\pi_{2,\ell}(F,t)=\E^F[\ell^p(\bb^\top \mathbf Z,t)].
\end{align*}

\begin{lemma}\label{lm-minmaxeq}
	For any $p\in[1,\infty)$, $G_0\in \mathcal M_p(\R^n)$ and $\epsilon\ge 0$,
	the following two statements hold.
	\begin{itemize}
		\item[(i)] Suppose that $\ell(z,t)$ is nonnegative on $\R^2$, and convex in $t$ with $\lim_{t\to-\infty}{\partial \ell(z,t)}/{\partial t}<-1$ for all $z\in\R$, and Lipschitz continuous in $z$ for all $t\in\R$ with a uniform Lipschitz constant, i.e., there exists $M>0$ such that
		$$
		|\ell(z_1,t)-\ell(z_2,t)|\le M|z_1-z_2|,~~\forall t,z_1,z_2\in\R.
		$$
		Then we have
		\begin{align*}
			\sup_{F\in \mathcal  B_p(G_0,\epsilon)} \inf_{t \in \R} \pi_{1,\ell}(F,t)=
			\inf_{t \in \R} \sup_{F\in \mathcal  B_p(G_0,\epsilon)}\pi_{1,\ell}(F,t).
		\end{align*}
		
		\item[(ii)] Suppose that $\ell(z,t)$ is convex in $t$ with $\lim_{t\to-\infty}{\partial \ell(z,t)}/{\partial t}<0$ and $\lim_{t\to\infty}{\partial \ell(z,t)}/{\partial t}>0$ for all $z\in\R$, and Lipschitz continuous in $z$ for all $t\in\R$ with a uniform Lipschitz constant.
		Then we have
		\begin{align*}
			\sup_{F\in \mathcal  B_p(G_0,\epsilon)} \inf_{t \in \R} \pi_{2,\ell}(F,t)=
			\inf_{t \in \R} \sup_{F\in \mathcal  B_p(G_0,\epsilon)}\pi_{2,\ell}(F,t).
		\end{align*}
	\end{itemize}
\end{lemma}

\proof{Proof.}
	(i) First, we show three facts below. (a) $\pi_{1,\ell}(F,t)$ is concave in $F$ for all $t\in\R$; (b) $\pi_{1,\ell}(F,t)$ is convex in $t$ for all $F\in\mathcal M_p(\R^n)$; (c) $\lim_{t\to\pm\infty}\pi_{1,\ell}(F,t)=\infty$ for all $F\in\mathcal M_p(\R^n)$.
The fact (a) is trivial. For $F\in\mathcal M_p(\R^n)$, $\lambda\in[0,1]$ and $t_1,t_2\in\R$, it holds that
\begin{align*}
	\left(\E^F[\ell^p(\bb^\top \mathbf Z,\lambda t_1+(1-\lambda)t_2)]\right)^{1/p}&\le 	
	\left(\E^F\left[\left(\lambda\ell(\bb^\top \mathbf Z, t_1)+(1-\lambda)\ell(\bb^\top \mathbf Z, t_2)\right)^p\right]\right)^{1/p}\\[3pt]
	&\le  \lambda\left(\E^F[\ell^p(\bb^\top \mathbf Z, t_1)]\right)^{1/p}+(1-\lambda) \left(\E^F[\ell^p(\bb^\top \mathbf Z,t_2)]\right)^{1/p},
\end{align*}
where the first step holds because $\ell$ is nonnegative, and the second follows from the triangle inequality. This implies (b). To see (c), it is obvious that $\lim_{t\to\infty}\pi_{1,\ell}(F,t)=\infty$. Note that
$
(\E^F[\ell^p(\bb^\top \mathbf Z,t)])^{1/p}\ge \E^F[\ell(\bb^\top \mathbf Z,t)].
$
Combining with $\lim_{t\to-\infty}{\partial \ell(z,t)}/{\partial t}<-1$ and the convexity of $\ell(z,t)$ in $t$, we have $\lim_{t\to-\infty}\pi_{1,\ell}(F,t)=\infty$. Hence, we conclude the proof of (c). Using (b) and (c), the set of all minimizers of the problem $\inf_{t \in \R} \pi_{1,\ell}(F,t)$ is a closed interval. Denote by $t(F):=\inf \arg\min_t \pi_{1,\ell}(F,t)$. We will show that $\{t(F): F\in\mathcal  B_p(G_0,\epsilon)\}$ is a subset of a compact set. For any $F\in B_p(G_0,\epsilon)$ and $t\in\R$, let $\mathbf Z\sim F$ and $\mathbf{Z_0}\sim G_0$, and we have
\begin{align}\label{eq-minuspi}
	|\pi_{1,\ell}(F,t)-\pi_{1,\ell}(G_0,t)|&=\left|\left(\E[\ell^p(\bb^\top \mathbf Z,t)]\right)^{1/p}-\left(\E^F[\ell^p(\bb^\top \mathbf {Z_0},t)]\right)^{1/p}\right|\notag\\[3pt]
	&\le \left(\E[|\ell(\bb^\top \mathbf Z,t)-\ell(\bb^\top \mathbf Z_0,t)|^p]\right)^{1/p}\notag\\[3pt]
 &\le \left(\E[M^p|\bb^\top (\mathbf Z-\mathbf Z_0)|^p]\right)^{1/p}\notag\\[3pt]
	&\le M\|\bb\|_*\left(\E[\|\mathbf Z-\mathbf Z_0\|^p]\right)^{1/p} \le M\|\bb\|_*\epsilon,
\end{align}
where the first and the third inequalities follow from the triangle inequality and H\"{o}lder's inequality, respectively, and we have used the definition of the Wasserstein ball $\mathcal B_p(G_0,\epsilon)$ in the last step. Hence, it holds that
\begin{align}\label{eq-piF1}
	\pi_{1,\ell}(F,t(G_0))\le \pi_{1,\ell}(G_0,t(G_0))+M\|\bb\|_*\epsilon.
\end{align}
Note that $\pi_{1,\ell}(G_0,t)\to\infty$ as $t\to\pm\infty$. There exists $\Delta>0$ such that $\pi_{1,\ell}(G_0,t)>\pi_{1,\ell}(G_0,t(G_0))+2M\|\bb\|_*\epsilon$ for all $t\notin[t(G_0)-\Delta,t(G_0)+\Delta]$. This, combined with \eqref{eq-minuspi}, imply that
\begin{align}\label{eq-piF2}
	\pi_{1,\ell}(F,t)\ge \pi_{1,\ell}(G_0,t)-M\|\bb\|_*\epsilon>\pi_{1,\ell}(G_0,t(G_0))+M\|\bb\|_*\epsilon,~~\forall t\notin[t(G_0)-\Delta,t(G_0)+\Delta].
\end{align}
Applying \eqref{eq-piF1} and \eqref{eq-piF2}, we have $\{t(F): F\in\mathcal  B_p(G_0,\epsilon)\}\subseteq[t(G_0)-\Delta,t(G_0)+\Delta]$. Using a minimax theorem (see e.g., \cite{S58}), it holds that
\begin{align*}
	\sup_{F\in \mathcal  B_p(G_0,\epsilon)} \inf_{t \in \R} \pi_{1,\ell}(F,t)&=\sup_{F\in \mathcal  B_p(G_0,\epsilon)} \inf_{t\in[t(G_0)-\Delta,t(G_0)+\Delta]} \pi_{1,\ell}(F,t)\\[3pt]
	&= \inf_{t\in[t(G_0)-\Delta,t(G_0)+\Delta]} \sup_{F\in \mathcal  B_p(G_0,\epsilon)}\pi_{1,\ell}(F,t)\ge \inf_{t\in\R} \sup_{F\in \mathcal  B_p(G_0,\epsilon)}\pi_{1,\ell}(F,t).
\end{align*}
The converse direction is trivial. Hence, we complete the proof.

(ii) The proof is similar to (i). 
\endproof
~~

{\bf Proof of Corollary \ref{cor1}.}
(i) For the case $\ell(z,t):=c\ell(z-t)$  with $\ell=\ell_3$ or $\ell_4$ in Theorem \ref{th-highorderloss},  one can verify that $\ell(z,t)$ satisfies the conditions in Lemma \ref{lm-minmaxeq} (ii). Hence, the equivalency result follows immediately from Lemma \ref{lm-minmaxeq} and Theorem \ref{th-highorderloss}. For the case $\ell(z,t):=c\ell(z-t)$  with $\ell =\ell_1$ or $\ell_2$ in Theorem \ref{th-highorderloss},  we assume without loss of generality that $\ell_1(x)=x_+$ and $\ell_2(x)=x_-$. In this case, we have 
\begin{align*}
\mathcal V^F (\bb^\top \mathbf Z)=\inf_{t \in \R} c^p\E^F[\ell_i^p(\bb^\top \mathbf Z-t)]
=\begin{cases}
\lim_{t\to\infty} c^p\E^F[(\bb^\top \mathbf Z-t)_+^p]=0,~~&i=1,\\[8pt]
\lim_{t\to-\infty} c^p\E^F[(\bb^\top \mathbf Z-t)_-^p]=0,~~&i=2.
\end{cases}
\end{align*}

(ii) For $\ell(z,t):=c\ell(z-t)$  with $\ell=\ell_1$, $\ell_3$ or $\ell_4$ in Theorem \ref{th-highorderloss}, it holds that $\ell(z,t)$ satisfies the conditions in Lemma \ref{lm-minmaxeq} (i). Applying Lemma \ref{lm-minmaxeq} and Theorem \ref{th-highorderloss}, the equivalency result holds. For $\ell(z,t):=c\ell(z-t)$  with $\ell=\ell_2$ in Theorem \ref{th-highorderloss}, we assume without loss of generality that $\ell_2(x)=x_-$. In this case,    for any $c>1$, $\bb\in\mathcal D$ and $F\in\mathcal M_p(\R^n)$, it holds that 
\begin{align*}
\rho^F(\bb^\top \mathbf Z)
=\inf_{t \in \R} \left\{t+c\left(\E^F[(\bb^\top\mathbf Z-t)_-^p]\right)^{1/p}\right\}
=\lim_{t\to-\infty}\left\{t+c\left(\E^F[(\bb^\top\mathbf Z-t)_-^p]\right)^{1/p}\right\}=-\infty.
\end{align*}
This completes the proof.
\endproof
~~


\subsection{Proofs of Section \ref{subsec:DT}}\label{subsec:appDT}

To prove the results in  Section \ref{subsec:DT}, we list the definition of \emph{comonotonicity} and
some basic properties of distortion functionals as follows; see e.g.,\cite {WWW20}.

\begin{definition}[Comonotonicity]\label{def-com}
Two random variables $Z_1$ and $Z_2$ are said to be \emph{comonotonic} if $(Z_1,Z_2)$ is distributionally equivalent to $(F^{-1}_{Z_1}(U),F^{-1}_{Z_2}(U))$, where $F_Z^{-1}$ denotes the left-quantile function of $Z$, and
$U$ is a random variable uniformly distributed on the interval $[0,1]$ (see e.g., \cite{D02} for a discussion of comonotonic random variables).
\end{definition}

\begin{lemma} \label{lm-propertydis}
Let $h:[0,1]\to\R$ be a distortion function. The following statements hold.
\begin{itemize}
\item [(i)] $\rho_h$ is monotone, i.e., $\rho_h(Z_1)\le \rho_h(Z_2)$ for $Z_1\le Z_2$ a.s. if $h$ is increasing.
\item[(ii)] $\rho_h$ is translation invariant, i.e., $\rho_h(Z+c)=\rho_h(Z)+(h(1)-h(0))c$ for any $Z$ and $c\in\R$.
\item[(iii)] $\rho_h$ is positively homogeneous, i.e., $\rho_h(\lambda Z)=\lambda\rho_h(Z)$ for any $Z$ and $\lambda\ge 0$.
\item[(iv)] $\rho_h$ is subadditive, i.e., $\rho_h(Z_1+Z_2)\le \rho_h(Z_1)+\rho_h(Z_2)$ for any $Z_1$ and $Z_2$   if $h$ is convex. The equality holds when $Z_1$ and $Z_2$ are comonotonic.
\end{itemize}
\end{lemma}

{\bf Proof of Theorem \ref{th-DTp=1}.}
We first consider the case $p=1$. Assume without loss of generality that ${\rm Lip}(\ell)=1$.
By Lemma \ref{lm-onehigheq}, it suffices to prove that for any $Z\in L^1$ and  $\epsilon\ge 0$,
\begin{align*}
\sup_{\|V\|_1\le \epsilon}\rho_h(\ell(Z+V))=\rho_h(\ell(Z))+\|h'_-\|_\infty\epsilon.
\end{align*}
To see it, first note that
\begin{align}\label{eq-FDp=1}
\sup_{\|V\|_1\le \epsilon}\rho_h(\ell(Z+V))
&\le \sup_{\|V\|_1\le \epsilon}\rho_h(\ell(Z)+|V|)
\le \rho_h(\ell(Z))+\sup_{\|V\|_1\le \epsilon}\rho_h(|V|)\notag\\[3pt]
&\le \rho_h(\ell(Z))+\sup_{\|V\|_1\le \epsilon}\|h'_-\|_\infty\|V\|_1=\rho_h(\ell(Z))+\|h'_-\|_\infty\epsilon,
\end{align}
where  the first inequality follows from $\ell(Z+V)-\ell(Z)\le |V|$ and the monotonicity of $\rho_h$, the second inequality follows from the subadditivity 
of $\rho_h$, and the last inequality is due to H\"{o}lder's inequality. Let us now verify the other direction. 
Suppose that $\lim_{x\to\infty}\ell'_-(x)=1$ where $\ell'_-$ is the left-derivative of $\ell$. 
Denote by $U$ a uniform random variable on $[0,1]$ such that $U$ and $\ell(Z)$ are comonotonic. For $n\in\N$ and $\delta>0$, define 
\begin{align}\label{eq-A-V}
A_{n,\delta}
= \left\{\omega: 1-\frac{\epsilon}{n}-\delta<U(\omega)\le 1-\delta\right\}~~{\rm and}~~V_{n,\delta}=n\id_{A_{n,\delta}}.
\end{align}
One can check that $\|V_{n,\delta}\|_1=\epsilon$. 
Moreover, we have
\begin{align}\label{eq-1DTp1}
\rho_h(\ell(Z+V_{n,\delta}))&=\rho_h\left(\ell(Z)\id_{A_{n,\delta}^c}+\ell(Z+n)\id_{A_{n,\delta}}\right)\notag\\[3pt]
&=\rho_h\left(\ell(Z)+ (\ell(Z+n)-\ell(Z)) \id_{A_{n,\delta}}\right)\notag\\[3pt]
&\ge \E\left[ \left( \ell(Z)+ (\ell(Z+n)-\ell(Z)) \id_{A_{n,\delta}}\right)   h_-'(U)\right] \notag\\[3pt]
& =  \rho_h(\ell(Z))+\E\left[ (\ell(Z+n)-\ell(Z)) \id_{A_{n,\delta}}   h_-'(U)\right],
\end{align}
where the inequality follows from the   dual representation of $\rho_h$ (see e.g., Theorem 4.79 of \cite{FS16}). 
Noting that $\ell(Z)$ must be uniformly bounded on  $A_{n,\delta}$ for all  $n>\epsilon/(1-\delta)$ as $\ell(Z)$ and $U$ are comonotonic, 
there exists $x_0\in\R$ such that $Z\ge x_0$ on $A_{n,\delta}$ for all $n>\epsilon/(1-\delta)$  as $\ell$ is convex.
Therefore, we have
\begin{align*}
\E\left[ (\ell(Z+n)-\ell(Z)) \id_{A_{n,\delta}}   h_-'(U)\right] &=\E\left[\frac{\ell(Z+n)-\ell(Z)}{n}V_{n,\delta}h'_-(U)\right]\\[3pt]
&\ge \frac{\ell(x_0+n)-\ell(x_0)}{n}\E\left[V_{n,\delta}h'_-(U)\right]\\[3pt]
&=\frac{\ell(x_0+n)-\ell(x_0)}{n} \frac{\int_{1-\epsilon/n-\delta}^{1-\delta}h'_-(s)\d s}{\epsilon/n}\epsilon\rightarrow h'_-(1-\delta)\epsilon~~{\rm as}~n\to\infty,
\end{align*}
where the inequality follows from the convexity of $\ell$ and $Z\ge x_0$ on $A_{n,\delta}$, and the convergence is due to  
$\lim_{x\to\infty}\ell'_-(x)=1$. Note that $h'_-$ is left continuous on $(0,1]$ (see e.g., Proposition A.4 of \cite{FS16}). It holds that $\lim_{x\to 1-}h'_-(x)=h'_-(1)$.
Therefore, letting $\delta\to 0$ and combining with \eqref{eq-1DTp1}, we have concluded that 
$$
\sup_{\|V\|_1\le \epsilon}\rho_h(\ell(Z+V))\ge \rho_h(\ell(Z))+h'_-(1)\epsilon=\rho_h(\ell(Z))+\|h'_-\|_\infty\epsilon.
$$
Hence, we have verified the other direction for the case of $\lim_{x\to\infty}\ell'_-(x)=1$. If $\lim_{x\to-\infty}\ell'_-(x)=-1$, the proof is similar by letting $V_{n,\delta}=-n\id_{A_{n,\delta}}$.
Hence, we complete the proof of the case $p=1$.

Let us now focus on the case $p>1$. In the rest of the proof, we assume  $C=1$ and $h(0)=1-h(1)=0$ without loss of generality.
By Lemma \ref{lm-onehigheq}, 
it suffices to show that (ii) holds  if and only if the following (i)' holds.
\begin{itemize}
\item [(i)'] For any $Z\in L^p$ and $\epsilon\ge 0$, we have \eqref{eq-regEU} holds, i.e.,
\begin{align}\label{eq-uniqDT1}
  \sup_{\|V\|_p\le \epsilon}\rho_h(\ell(Z+V))
  =
\rho_h(\ell(Z))+\epsilon\|h'_-\|_q,~~\forall Z\in L^p,~\epsilon\ge 0.
  \end{align}
\end{itemize}

To see (ii) $\Rightarrow$ (i)', 
we note that $\rho_h(-Z)=\rho_{\widetilde{h}}(Z)$, where $\widetilde{h}(s)=h(1-s)$ for $s\in[0,1]$. It is easy to see that $\widetilde{h}$ is convex whenever $h$ is convex. Hence,  (a) is a special case of Proposition \ref{prop:disabsolute}, and we choose to omit its proof here for brevity.
In the case of (b), we assume without loss of generality that $\ell(x)=|x|$ for $x\in\R$.
Note that $h$ is an increasing and convex distortion function.
By Lemma \ref{lm-propertydis}, we have for any  $V$ with $\|V\|_p\le \epsilon$,
\begin{align}\label{eq-disabsolute2}
	\rho_h(|Z+V|)
	&\le \rho_h(|Z|)+\rho_h(|V|)
	\le \rho_h(|Z|)+\|V\|_p\|h'_-\|_q
	\le \rho_h(|Z|)+\epsilon\|h'_-\|_q,
\end{align}
where the second inequality follows from H\"{o}lder's inequality. Suppose that $|Z|=F_{|Z|}^{-1}(U)$ a.s. where $U$ is a uniform random variable on $[0,1]$. Define
$
\widetilde{V}={\rm sign}(Z){\epsilon(h'_-(U))^{q/p}}/{\|h'_-\|_q^{q/p}}.
$
One can verify that 
$\|\widetilde{V}\|_p=\epsilon$ and $\rho_h(|Z+\widetilde{V}|) = \rho_h(|Z|)+\epsilon\|h'_-\|_q$.  Therefore, we conclude that 
\eqref{eq-uniqDT1} holds. Hence, we complete the proof of (i)'.

To see (i)' $\Rightarrow$ (ii), 
note that $h$ is increasing with $h(0)=1-h(1)=0$. It holds that $\rho_h$ satisfies monotonicity and translation invariance, i.e., $\rho_h(Z+c)=\rho_h(Z)+c$ for all $Z\in\ L^p$ and $c\in\R$.
We first show that ${\rm Lip}(\ell)\le 1$ by contradiction. Otherwise, by that a convex function has derivative almost everywhere, there exists $x$ such that $|\ell'(x)|>1$. In this case, define $V_1\in L^p$ with quantile function as 
$
F_{V_1}^{-1}(s)={\epsilon(h'_-(s))^{q/p}}/{\|h'_-\|_q^{q/p}},~~s\in(0,1),
$
and one can easily check that $V_1\ge 0$ and 
$\|V_1\|_p= \epsilon$. Hence, 
\begin{align*}
\sup_{\|V\|_p\le \epsilon}\rho_h(\ell(x+V))&\ge 
\sup_{\|V\|_p\le \epsilon,V\ge 0}\rho_h(\ell(x)+\ell'(x)V)\\[3pt]
&=\ell(x)+\ell'(x)\sup_{\|V\|_p\le \epsilon,V\ge 0}\rho_h(V)\\[3pt]
&\ge \ell(x)+\ell'(x)\rho_h(V_1)\\[3pt]
&=\ell(x)+\ell'(x)\|h'_-\|_q\epsilon
>\ell(x)+\|h'_-\|_q\epsilon.
\end{align*}
If $\ell'(x)<-1$, then define $V_2=-V_1$, and we have $V_2\le 0$ and $\|V_2\|_p=\epsilon$.
Hence,
\begin{align*}
\sup_{\|V\|_p\le \epsilon}\rho_h(\ell(x+V))&\ge 
\sup_{\|V\|_p\le \epsilon,V\le 0}\rho_h(\ell(x)+\ell'(x)V)\\[3pt]
&=\ell(x)-\ell'(x)\sup_{\|V\|_p\le \epsilon,V\le 0}\rho_h(-V)\\[3pt]
&\ge \ell(x)-\ell'(x)\rho_h(-V_2)\\[3pt]
&=\ell(x)-\ell'(x)\rho_h(V_1)\\[3pt]
&=\ell(x)-\ell'(x)\|h'_-\|_q\epsilon>\ell(x)+\|h'_-\|_q\epsilon.
\end{align*}
Those two cases both yield a contradiction to \eqref{eq-uniqDT1}. 

Next, we aim  to show
\begin{align}\label{eq-assDT}
	|\ell'(x)|=1~{\rm{for~all}}~x\in\R~{\rm as~long~as}~\ell~{\rm is~differentiable~at}~x,
\end{align}
as this will complete the proof since a convex function that satisfies \eqref{eq-assDT} must be one of the forms of $\ell_1$ and $\ell_2$ with $C=1$.
To show \eqref{eq-assDT},  
we assume by contradiction that there exists $x_0\in\R$ such that $|\ell'(x_0)|<1$. If $p=\infty$, then we have
\begin{align*}
\sup_{\|V\|_\infty\le \epsilon}\rho(\ell(x_0+V)) 
=\max\{\ell(x_0-\epsilon),\ell(x_0+\epsilon)\}<\ell(x_0)+\epsilon,
\end{align*}
where
the strict inequality follows from $|\ell'(x_0)|<1$ and ${\rm Lip}(\ell)\le 1$, which yields a contradiction. Suppose now $p\in(1,\infty)$. To yield a contradiction, it suffices to prove that 
\begin{align}\label{eq-verify}
\sup_{\|V\|_p\le \epsilon}\rho_h(\ell(x_0+V))
<\ell(x_0)+\epsilon\|h'_-\|_q.
\end{align}
Let $U\sim {\rm U}[0,1]$ be a uniformly distributed random variable, and define 
$$
\mathcal X_p=\left\{F^{-1}(U): \int_0^1 |F^{-1}(u)|^pd u<\infty,~F^{-1}(U)\ge 0\right\}.
$$
We have that $\{F_X: X\in L^p\}=\{F_X: X\in\mathcal X_p\}$ and the random variables in $\mathcal X_p$ are all comonotonic. 
Define
\begin{equation}\label{eq-alpha}
\alpha^*=\sup\{\alpha\in[0,1]: h(\alpha)=0\}<1~~{\rm and}~~M=\epsilon\left(\frac{1-\alpha^*}{2}\right)^{-1/p}.
\end{equation}
It follows from Chebyshev's inequality that 
\begin{equation}\label{eq-Chebyshev}
\p(V>M)\le \frac{\E[V^p]}{M^p}\le \frac{\epsilon^p}{M^p}= \frac{1-\alpha^*}{2}~~{\rm if}~\|V\|_p\le \epsilon~{\rm and}~V\ge 0.
\end{equation}
Define
$$
k=\max\left\{\left|\frac{\ell(x_0+M)-\ell(x_0)}{M}\right|,\left|\frac{\ell(x_0)-\ell(x_0-M)}
{M}\right|\right\}.
$$
We have that 
$|\ell(x_0+v)-\ell(x_0)|\le k|v|$ for all $v\in[-M,M]$ as 
$\ell$ is convex. Further, it holds that $k<1$ as $|\ell'(x_0)|<1$ and ${\rm Lip}(\ell)=1$.
Therefore,
\begin{align}\label{eq-chain}
\sup_{\|V\|_p\le \epsilon}\rho_h(\ell(x_0+V))
&=\sup_{\|V\|_p\le \epsilon}\rho_h\left(\ell(x_0+V)\id_{\{|V|\le M\}}+\ell(x_0+V)\id_{\{|V|> M\}}\right)\notag\\[3pt]
&\le \sup_{\|V\|_p\le \epsilon}\rho_h\left((\ell(x_0)+k|V|)\id_{\{|V|\le M\}}+\ell(x_0+V)\id_{\{|V|> M\}}\right)\notag\\[3pt]
&\le\sup_{\|V\|_p\le \epsilon}\rho_h\left((\ell(x_0)+k|V|)\id_{\{|V|\le M\}}+(\ell(x_0)+|V|)\id_{\{|V|> M\}}\right)\notag\\[3pt]
&=\ell(x_0)+\sup_{\|V\|_p\le \epsilon}\rho_h\left(k|V|\id_{\{|V|\le M\}}+|V|\id_{\{|V|> M\}}\right)\notag\\[3pt]
&=\ell(x_0)+\sup_{\|V\|_p\le \epsilon, V\ge 0}\rho_h\left(kV\id_{\{V\le M\}}+V\id_{\{V> M\}}\right)\notag\\[3pt]
&=\ell(x_0)+\sup_{\|V\|_p\le \epsilon,V\in\mathcal X_p}\rho_h\left(kV\id_{\{V\le M\}}+V\id_{\{V> M\}}\right)\notag\\[3pt]
&=\ell(x_0)+\sup_{\|V\|_p\le \epsilon,V\in\mathcal X_p}\E\left[\left(kV\id_{\{V\le M\}}+V\id_{\{V> M\}}\right)h'_-(U)\right]\notag\\[3pt]
&=\ell(x_0)+\sup_{\|V\|_p\le \epsilon,V\in\mathcal X_p}\E\left[\left(V-(1-k)V\id_{\{V\le M\}}\right)h'_-(U)\right]\notag\\[3pt]
&=\ell(x_0)+\sup_{\|V\|_p\le \epsilon,V\in\mathcal X_p} \{\rho_h(V)-(1-k)\E[V\id_{\{V\le M\}}h'_-(U)]\}\notag\\[3pt]
&=:\ell(x_0)+I,
\end{align}
where the second inequality follows from ${\rm Lip}(\ell)\le 1$,
the fourth equality is due to the distribution-invariance of $\rho_h$, and the fifth equality holds because $V\in\mathcal X_p$ implies that $kV\id_{\{V\le M\}}+V\id_{\{V> M\}}$ and $h'_-(U)$ are comonotonic. 
For $\delta>0$, define
\begin{align}\label{eq-V1V2}
\mathcal V_1=\left\{V\in \mathcal X_p: \|V\|_p\le \epsilon,~\E\left[V\id_{\{V\le M\}}h'_-(U)\right]\le \delta\right\},~~\mathcal V_2=\{V\in \mathcal X_p: \|V\|_p\le \epsilon\}\setminus \mathcal V_1.
\end{align}
We can rewrite $I=\max\{I_1,I_2\}$ with 
$$
I_i=\sup_{V\in \mathcal V_i} \{\rho_h(V)-(1-k)\E[V\id_{\{V\le M\}}h'_-(U)]\}, ~~i=1,2.
$$
Below we aim to demonstrate that $I_i < \epsilon  \|h'_-\|_q$ for $i=1,2$ by selecting an appropriate $\delta$.
Note that
\begin{align}
I_1&=\sup_{V\in\mathcal V_1} \{\rho_h(V)-(1-k)\E[V\id_{\{V\le M\}}h'_-(U)]\}\notag\\[3pt]
&\le 
\sup_{V\in\mathcal V_1} \rho_h(V)\notag\\[3pt]
&=\sup_{V\in\mathcal V_1}\left\{\E[V\id_{\{V>M\}}h'_-(U)]+\E[V\id_{\{V\le M\}}h'_-(U)]\right\}\notag\\[3pt]
&\le \sup_{V\in\mathcal V_1}\|h'_-(U)\id_{\{V>M\}}\|_q\|V\|_p+\delta\notag\\[3pt]
&\le \sup_{V\in\mathcal V_1}\|h'_-(U)\id_{\{V>M\}}\|_q\epsilon+\delta
\notag\\[3pt]
&\le \left(\int_{\frac{1+\alpha^*}{2}}^1 (h'_-(s))^q\d s\right)^{1/q}\epsilon+\delta
=\left(\|h'_-\|_q^q-\int_{\alpha^*}^{\frac{1+\alpha^*}{2}}(h'_-(s))^q\d s\right)^{1/q}\epsilon+\delta,\label{eq-regDT2}
\end{align}
where the second inequality follows from H\"{o}lder's inequality and the definition of $\mathcal V_1$ in \eqref{eq-V1V2},  the last inequality  is due to \eqref{eq-Chebyshev}, and the last equality  follows from the definition of $\alpha^*$ in \eqref{eq-alpha}. Denote by 
\begin{align*}
A=\int_{\alpha^*}^{\frac{1+\alpha^*}{2}}(h'_-(s))^q\d s ~~{\rm and}~~ \Delta=\left[\|h'_-\|_q-(\|h'_-\|_q^q-A)^{1/q}\right]\epsilon.
\end{align*}
Recalling the definition of $\alpha^*$ in \eqref{eq-alpha} again, we have $A>0$ and $\Delta>0$, and hence, 
$I_1<\epsilon\|h'_-\|_q$ whenever $\delta<\Delta$.
For $\delta<\Delta$, we have
\begin{align*}
I_2=\sup_{V\in\mathcal V_2} \{\rho_h(V)-(1-k)\E[V\id_{\{V\le M\}}h'_-(U)]\}
\le 
\sup_{V\in\mathcal V_2} \|h'_-\|_q\|V\|_p-(1-k)\delta
<\epsilon \|h'_-\|_q,~~\forall \delta>0,
\end{align*}
where the first inequality follows from H\"{o}lder's inequality. Hence, we have  that $\max\{I_1,I_2\}<\epsilon\|h'_-\|_q$. Combining with \eqref{eq-chain}, we have verified \eqref{eq-verify}.
This completes the proof. \endproof

~~ 
 
To substantiate our claim that our analysis is not limited to the examples presented in this paper, we present the following proposition to demonstrate that the arguments in the proof of Theorem \ref{th-DTp=1} can be applied to establish the equivalence between Wasserstein DRO and regularization for a class of  non-convex loss functions, when the support of the reference distribution is a finite set.  

 \begin{proposition}\label{prop:noncx}
For $p=1$, let $h$ be  given in Theorem  \ref{th-DTp=1}. If the support of $G_0\in\mathcal M_1(\R)$ is  finite, and $\ell:\R\to\R$ with ${\rm Lip}(\ell)\in\R$ satisfies $\lim_{t\to\infty} (\ell(x+t)-\ell(x))/t={\rm Lip}(\ell)$ for all $x\in\R$ or  $\lim_{t\to-\infty}(\ell(x)-\ell(x+t))/t={\rm Lip}(\ell)$ for all $x\in\R$, then, for $\epsilon\ge 0$ and $\mathcal D\subseteq\R^n$
\begin{align*}
\inf_{\bm\beta\in \mathcal D}\sup_{F\in \mathcal B_1(G_0,\epsilon)}\rho_h^F(\ell(\bm\beta^\top \mathbf Z))=\inf_{\bm\beta\in\mathcal D}\left\{\rho_h^{G_0}(\ell(\bm\beta^\top\mathbf Z))+{\rm Lip}(\ell)\|h'_-\|_\infty\epsilon\|\bm\beta\|_*\right\}.
\end{align*}
\end{proposition}


\proof{Proof.} 
Without loss of generality,  assume  ${\rm Lip}(\ell)=1$.  We only give the proof for the case $\lim_{t\to\infty} (\ell(x+t)-\ell(x))/t=1$ for $x\in\R$   as the other case is similar. 
By Remark \ref{re:EC1}, it suffices to show that 
\eqref{eq-regEU}  holds with $p=1$ and $Z=\bb^{\top} \mathbf{Z}_0$ where $\mathbf{Z}_0\sim G_0$. For $n\in\N$ and $\delta>0$, define  $A_{n,\delta}$ and $V_{n,\delta}$ as  \eqref{eq-A-V}. First note that both \eqref{eq-FDp=1}  and \eqref{eq-1DTp1} hold  for $Z=\bb^{\top} \mathbf{Z}_0$. 
Noting that the support of $Z=\bb^\top\mathbf{Z}_0$ is finite, we denote it by $\{x_1,\dots,x_k\}$. For $Z=\bb^\top\mathbf{Z}_0$, it holds that  
\begin{align*}
\E\left[ (\ell(Z+n)-\ell(Z)) \id_{A_{n,\delta}}   h_-'(U)\right] &=\E\left[\frac{\ell(Z+n)-\ell(Z)}{n}V_{n,\delta}h'_-(U)\right]\\[3pt]
&\ge \min_{i\in [k]} \frac{\ell(x_i+n)-\ell(x_i)}{n}\E\left[V_{n,\delta}h'_-(U)\right]\\[3pt]
&=\min_{i\in [k]} \frac{\ell(x_i+n)-\ell(x_i)}{n} \frac{\int_{1-t/n-\delta}^{1-\delta}h'_-(s)\d s}{t/n}t\\[3pt]
& \rightarrow h'_-(1-\delta)t~~{\rm as}~n\to\infty.
\end{align*}
Letting $\delta\to 0$ and combining this with  \eqref{eq-1DTp1}, we have concluded that 
\begin{align*}
\sup_{\|V\|_1\le \epsilon}\rho_h(\ell(Z+V))\ge \rho_h(\ell(Z))+h'_-(1)\epsilon=\rho_h(\ell(Z))+\|h'_-\|_\infty \epsilon.
\end{align*}
 Hence, \eqref{eq-regEU}  holds with $p=1$ and $Z=\bb^{\top} \mathbf{Z}_0$, and thus we complete the proof.
\endproof

~~

 {\bf Proof of Proposition \ref{prop:disabsolute}.}
By Lemma \ref{lm-onehigheq}, it suffices to prove that 
\begin{align}\label{eq-dis0}
	\sup_{\|V\|_p\le \epsilon}\rho_h(Z+V)=\rho_h(Z)+\epsilon\|h'_-\|_q,~~\forall Z\in L^p.
\end{align}
By Lemma \ref{lm-propertydis}, we have for any $\|V\|_p\le \epsilon$,
\begin{align}\label{eq-dis1}
	\rho_h(Z+V)&\le \rho_h(Z)+\rho_h(V)
	\le \rho_h(Z)+\|V\|_p\|h'_-\|_q
	\le \rho_h(Z)+\epsilon\|h'_-\|_q,
\end{align}
where the second inequality follows from H\"{o}lder's inequality. Suppose that $Z=F_{Z}^{-1}(U)$ a.s.,  where $U$ is a uniform random variable on $[0,1]$ (see Lemma A.28 of \cite{FS16} for the existence of $U$).
In the following, we show \eqref{eq-dis0} by considering the following two cases.
\begin{itemize}
	\item[(a)] If $p=1$, then $q=\infty$ and $\|h'_-\|_{\infty}=\max\{|h'_-(0+)|,|h'_-(1)|\}=\max\{-h'_-(0+),h'_-(1)\}$, where the second equality follows from $h'_-$ is increasing.  We first assume that $h'_-(1)\ge -h'_-(0+)$. Define a sequence $\{V_n\}_{n\in\N}$ such that $V_n=n\epsilon\id_{\{U>1-1/n\}}$.
	For all $n\in\N$, it holds that $\|V_n\|_1=\epsilon$, and $V_n$ and $Z$ are comonotonic. Hence, we have
	\begin{align}\label{eq-dis2}
		\sup_{\|V\|_p\le \epsilon}\rho_h(Z+V)&\ge \rho_h(Z+V_n)
		=\rho_h(Z)+n\epsilon\int_{1-1/n}^1 h'_-(s)\d s\notag\\[3pt]
		&\to\rho_h(Z)+\epsilon h'_-(1)
		=\rho_h(Z)+\epsilon \|h'_-\|_\infty,~~{\rm as}~n\to\infty.
	\end{align}
	Combining \eqref{eq-dis1} and \eqref{eq-dis2}, we have verified \eqref{eq-dis0}.
	Assume now $h'_-(1)< -h'_-(0+)$. We can construct a sequence $\{V_n\}_{n\in\N}$ as $V_n=-n\epsilon\id_{\{U<1/n\}}$, and then follow the same analysis as the previous argument to verify the result.

	\item[(b)] If $p\in(1,\infty]$, then we define
	$
	\widetilde{V}={\rm sgn}(h'_-(U)){\epsilon|h'_-(U)|^{q/p}}/{\|h'_-\|_q^{q/p}}.
	$
	One can verify that $\|\widetilde{V}\|_p=\epsilon$ and $\rho_h(Z+\widetilde{V})=\rho_h(Z)+\epsilon\|h'_-\|_q$. 
	This, together with \eqref{eq-dis1}, implies that  \eqref{eq-dis0} holds.
\end{itemize}
Combining (a) and (b),  we complete the proof of \eqref{eq-dis0}, and thus complete the proof. \endproof 

{\color{blue}
\section{Proofs of Section \ref{beyond}}\label{sec:proof6}
}

Before presenting the proofs in Section \ref{beyond}, we first highlight some key elements required in our analysis for the non-affine case.  As stated in Section \ref{beyond}, while our analysis for the non-affine case continues to rely on a projection set perspective, the equivalence between the projection sets of a Wasserstein ball and a max-sliced Wasserstein ball no longer holds. Instead, we explore the possibility of first establishing a confidence bound for a fixed decision $\bb$ by identifying a set inclusion relationship with a one-dimensional Wasserstein ball and leveraging its measure concentration property (Lemma \ref{lm-CBWD}). We then extend these results to derive generalization bounds that hold uniformly across all $\bb \in {\cal D}$ using covering number techniques (see e.g., \cite{G22} and Section 3.5 of \cite{MRT18}). For $\tau>0$, a $\tau$-cover of $\mathcal D$, denoted by $\mathcal D_\tau$, is a subset of $\mathcal D$ such that
for each $\bm \beta\in \mathcal D$, there exists $\widetilde{\bm \beta}\in\mathcal D_\tau$ satisfying $\|\widetilde{\bm \beta}-\bm \beta\|_\mathcal D\le \tau$, where $\|\cdot\|_\mathcal D$ is a norm on $\mathcal D$.
The covering number $\mathcal N(\tau;\mathcal D,\|\cdot\|_\mathcal D)$ of $\mathcal D$ with respect to $\|\cdot\|_\mathcal D$ is the smallest cardinality of such a $\tau$-cover of $\mathcal D$.

\begin{lemma}[Theorem 2 of \cite{FG15}]\label{lm-CBWD}
Let $p\ge 1$ and $\eta\in(0,1)$. For $\widetilde{F}^*\in\mathcal M_p(\R)$, suppose that $A_0:=\E^{\widetilde{F}^*}[\exp(|Z|^a)]<\infty$ for some $a>p$ and $\widetilde{F}_N$ is the empirical distribution based on the independent sample drawn from $\widetilde{F}^*$. Then, we have
\begin{align*}
\p\left(W_p(\widetilde{F}_N, \widetilde{F}^*)\le \epsilon_{p,N}(\eta)\right)\ge 1-\eta,
\end{align*}
where 
\begin{align}\label{eq-epsilon}
\epsilon_{p,N}(\eta)=\begin{cases}
\left(\frac{\log(c_1/\eta)}{c_2 N}\right)^{1/(2p)}~~&\text{if}~N\ge \frac{\log(c_1/\eta)}{c_2},\\
\left(\frac{\log(c_1/\eta)}{c_2 N}\right)^{1/a}~~&\text{if}~N< \frac{\log(c_1/\eta)}{c_2},
\end{cases}
\end{align}
	and $c_1,c_2$ are positive constants that only depend on $a$, $A_0$, and $p$.
\end{lemma}

With these tools in place, we now proceed to the proofs for the generalization results in Section \ref{beyond}.

\subsection{Proof of Theorem \ref{co-nonunion}}


We proceed in three steps. First, we verify the following set inclusion result: 
\begin{align}\label{eq-setinall-0}
\overline{\mathcal B}_{p|f_{\bb}}(  F_0,\epsilon) \supseteq \mathcal B_p(F_{Y_0\cdot f_{\bb}(\mathbf X_0)},g(\epsilon)),
\end{align}
under Assumption \ref{ass:UB03},  where   $(Y_0,\mathbf X_0)\sim F_0$ and $\overline{\mathcal B}_{p|f_{\bb}}$  is the projection set defined by \eqref{eq-HONE-1} in  Section \ref{beyond}:
\begin{align}\label{appeq-HONE}
\overline{\mathcal B}_{p|f_{\bb}}(F_0, \epsilon)=\{F_{Y\cdot f_{\bb}(\mathbf X)}: F_{(Y,\mathbf X)}\in\overline{\mathcal B}_p(F_0,\epsilon)\}.
\end{align}  
Second, using the set inclusion result \eqref{eq-setinall-0} and the concentration result for the one-dimensional Wasserstein ball (Lemma \ref{lm-CBWD}), we establish a confidence bound for a fixed $\bb$. Finally, we apply covering number techniques to derive union bounds from the confidence bounds established in the second step, overcoming the challenge posed by the non-finiteness of the set ${\cal D}$.

~~

\noindent{\bf Step 1: Proving the set inclusion result \eqref{eq-setinall-0}.} 
We  first give a lemma which will be used in the proof of the   set inclusion result \eqref{eq-setinall-0}.  
\begin{lemma}\label{lm-auxcxcv}
Let $p\ge 1$ and $g:\R_+\to \R_+$ be an increasing function with $g(0)=0$. If $g$ is convex (resp. concave), then the function $x\mapsto (g(x^{1/p}))^{p}$ is convex (resp. concave) on $\R_+$. 
\end{lemma}

\proof{Proof.}
We provide the proof of convexity only for the case where 
$g$ is twice differentiable, as the concavity case is analogous and the general case can be established by approximating $g$
  with a twice differentiable convex function.
Denote by $h(x)=(g(x^{1/p}))^{p}$, $x\in\R_+$.
By standard calculations, we have
\begin{align*}
h''(x) &=\frac{p-1}{p}g'(y)(g(y))^{p-2}y^{1-2p}[yg'(y)-g(y)]
+\frac{1}{p} g''(y)(g(y))^{p-1}y^{2-2p}\\
& \ge \frac{p-1}{p}g'(y)(g(y))^{p-2}y^{1-2p}[yg'(y)-g(y)] \\
& \stackrel{\rm sgn}= yg'(y)-g(y),
\end{align*}
where we have denoted by $y=x^{1/p}\in\R_+$, the inequality follows from  $g''\ge 0$ and $g\ge 0$,  $A\stackrel{\rm sgn}=B$ means that $A$ and $B$ have the same sign, and $\stackrel{\rm sgn}= $ comes from  $ (p-1)g'(y)(g(y))^{p-2}y^{1-2p}\ge 0$.  
By $g(0)=0$ and the convexity of $g$
, we have that $g(y)/y=(g(y)-g(0))/y\le g'(y)$, that is, $yg'(y)-g(y)\ge 0$. It then follows that $h''\ge 0$, that is, $h$ is convex. 
This completes the proof. \endproof

Now we are ready to show  \eqref{eq-setinall-0} under Assumption \ref{ass:UB03}. 
\begin{lemma}\label{lm-setinclusion}
For $p\in[1,\infty]$, $\epsilon\ge0$,   $F_0\in\mathcal M_p(\Xi)$, 
and $(Y_0,\mathbf X_0)\sim F_0$, 
if Assumption \ref{ass:UB03} holds, then for  $\bb\in\mathcal D$,
we have   \eqref{eq-setinall-0} holds. 
\end{lemma}

\proof{Proof.}
For $\bb\in\mathcal D$,
let $F\in\mathcal B_p(F_{Y_0\cdot f_{\bb}(\mathbf X_0)},g(\epsilon))$. There exists $Z$ such that $Z\sim F$ and $\E[|Z-Y_0\cdot f_{\bb}(\mathbf X_0)|^p]\le (g(\epsilon))^p$. Our aim is to construct a random vector $(Y,\mathbf X)$ such that $F_{(Y,\mathbf X)}\in\overline{\mathcal B}_p(F_0,\epsilon)$ and $Y\cdot f_{\bb}(\mathbf X)\sim F$ as this implies that $F\in \overline{\mathcal B}_{p|f_{\bb}}(F_0,\epsilon)$.  Denote by $T=Z/Y_0-f_{\bb}(\mathbf X_0)$ and we have
\begin{align} \label{eq:1007-1}
\E[|T|^p]=\E\left[\left|\frac{Z}{Y_0}-f_{\bb}(\mathbf X_0)\right|^p\right]
=\E[|Z-Y_0\cdot f_{\bb}(\mathbf X_0)|^p]\le (g(\epsilon))^p.
\end{align} 
We first assert that there exist measurable mappings $\mathbf S_1$ 
and  $\mathbf S_2$ 
such that 
\begin{align*}
\mathbf S_1(\omega)\in\argmax_{\|\mathbf y\|\le g^{-1}(|T(\omega)|)} f_{\bb}({\bf X}_0(\omega)+\mathbf y)
~~{\rm and}~~
\mathbf S_2(\omega)\in\argmin_{\|\mathbf y\|\le g^{-1}(|T(\omega)|)} f_{\bb}({\bf X}_0(\omega)+\mathbf y),~~\omega\in \Omega.
\end{align*}
 To see  this, we  apply a measurable selection theorem  (Theorem 3.5  in \cite{R78}) to show the existence of ${\bf S}_1$ and that of ${\bf S}_2$ is similar. 
Define 
\begin{align*}
v(\omega)=\sup_{\mathbf y\in D(\omega)} u(\omega,\mathbf y),~~\omega\in \Omega,
\end{align*}
where $D(\omega)=\{\mathbf y\in\R^n: \|\mathbf y\|\le g^{-1}(|T(\omega)|)\}$
and $u(\omega,\mathbf y)=f_{\bb}(\mathbf X_0(\omega)+\mathbf y)$.
Further define
\begin{align*}
E=\{(\omega,\mathbf y)\in\Omega\times\R^n: \|\mathbf y\|\le g^{-1}(|T(\omega)|)\}.
\end{align*}
Let ${\rm P_1}(E)$ be the projection set of $E$ on the first argument.
The first three conditions in Theorem 3.5 of \cite{R78} are trivial in our  setting. To see the last condition,
for $c\in\R$ and $\omega\in {\rm P_1}(E)$, we have 
\begin{align*}
\{\mathbf y\in D(\omega): u(\omega,\mathbf y)\ge c\}
=D(\omega)\cap \{\mathbf y\in\R^n: f_{\bb}(\mathbf X_0(\omega)+\mathbf y)\ge c\}.
\end{align*}
Note that $D(\omega)$ is a compact set and the continuity of $f_{\bb}$ implies that $\{\mathbf y\in\R^n: f_{\bb}(\mathbf X_0(\omega)+\mathbf y)\ge c\}$ is a closed set. Hence, we conclude that $\{\mathbf y\in D(\omega): u(\omega,\mathbf y)\ge c\}$ is a compact set, which verifies the last condition in Theorem 3.5 of \cite{R78}. This implies the existence of a measurable maximizer ${\bf S}_1$.  
Denote by  $A_+=\{\omega: T(\omega)\ge 0\}$ and $A_-=\{\omega: T(\omega)<0\}$.
We define $\mathbf X'$ as 
\begin{align*}
\mathbf X'(\omega)=\mathbf X_0(\omega)+\mathbf S_1(\omega),~~\omega\in A_+~~{\rm and}~~
\mathbf X'(\omega)=\mathbf X_0(\omega)+\mathbf S_2(\omega),~~\omega\in A_-.
\end{align*}
By that $\mathbf S_1$ and $\mathbf S_2$ are measurable, we have $\mathbf X'$ is measurable.
Note that $\Omega=A_+\cup A_-$, and hence,
\begin{align}\label{eq-mon1}
\|\mathbf X'(\omega)-\mathbf X_0(\omega)\|
\le \max\{\|\mathbf S_1(\omega)\|,\|\mathbf S_2(\omega)\|\}
\le g^{-1}(|T(\omega)|),~~\omega\in\Omega.
\end{align}
By Assumption \ref{ass:UB03}, we have
\begin{align*}
f_{\bb}(\mathbf X'(\omega))-f_{\bb}(\mathbf X_0(\omega))\ge  g (g^{-1}(|T(\omega)|)) = T(\omega),~~\omega\in A_+
\end{align*}
and 
\begin{align*}
f_{\bb}(\mathbf X_0(\omega))-f_{\bb}(\mathbf X'(\omega)) \ge  g (g^{-1}(|T(\omega)|))=-T(\omega),~~\omega\in A_-.
\end{align*}
Note that $f_{\bb}$ is continuous. There exists a random variable $R$ taking values on $[0,1]$ such that 
\begin{align}\label{eq-moninclusion}
f_{\bb}(\mathbf X_0 +R (\mathbf X' -\mathbf X_0 ))-f_{\bb}(\mathbf X_0 )=T.
\end{align}
Define $\widetilde{\mathbf X}=\mathbf X_0+R(\mathbf X'-\mathbf X_0)$. It follows from \eqref{eq-moninclusion} that 
\begin{align}\label{eq-consdis}
Y_0\cdot f_{\bb}(\widetilde{\mathbf X})=Y_0\cdot f_{\bb}(\mathbf X_0)+Y_0T=Z\sim F.
\end{align}
Moreover, we have
\begin{align*}
\E[\|\widetilde{\mathbf X}-\mathbf X_0\|^p] =\E[\|R (\mathbf X' -\mathbf X_0 )\|^p] 
\le \E\left[(g^{-1}(|T|))^p\right]
\le \left(g^{-1}\big((\E[|T|^p])^{1/p}\big)\right)^p
\le \epsilon^p,
\end{align*}
where the first inequality is due to \eqref{eq-mon1};
the second inequality follows from the Jensen's inequality and the concavity of the function $x\mapsto (g^{-1}(x^{1/p}))^p$ by Lemma \ref{lm-auxcxcv}; and the last inequality follows from \eqref{eq:1007-1}. Hence,   $F_{(Y_0,\widetilde{\mathbf X})}\in\overline{\mathcal B}_p(F_0,\epsilon)$.
Combining this with \eqref{eq-consdis} yields $F\in \overline{\mathcal B}_{p|f_{\bb}}(F_0, \epsilon)$. This completes the proof.
\endproof

~~

\noindent{\bf Step 2: Establishing confidence bounds for fixed $\bb$.} 
Based on Lemma \ref{lm-setinclusion} and applying Lemma \ref{lm-CBWD}, we are able to derive the following confidence bounds under Assumptions \ref{ass:UB021} and \ref{ass:UB03}.
\begin{lemma}\label{lm-CB1}
Let $p\ge 1$ and $\eta\in(0,1)$. Suppose that Assumptions \ref{ass:UB021} and \ref{ass:UB03} hold. Then for each $\bb\in\mathcal D$, we have
\begin{align*}
\p\left(\rho^{F^*}(Y\cdot f_{\bb}(\mathbf X))\le \sup_{F\in \overline{\cal B}_p(\widehat{F}_N,g^{-1}(\epsilon_{p,N}(\eta)))}\rho^{F}(Y\cdot f_{\bb}(\mathbf X))\right) \ge 1-\eta,
\end{align*}
where $\epsilon_{p,N}(\eta)$ is defined by \eqref{eq-epsilon} with  constants $c_1,c_2$  depending only on $a$, $A$ and $p$. 
\end{lemma}

\proof{Proof.}
Denote by $\widehat{F}_{\bb,N}=\frac{1}{N}\sum_{i=1}^N \delta_{\widehat{y}_i\cdot f_{\bb}(\widehat{\mathbf x}_i)}$ and $F_{\bb}^*=F_{Y^*\cdot f_{\bb}(\mathbf X^*)}$, where $(Y^*,\mathbf X^*)\sim F^*$.  It holds that $\widehat{F}_{\bb,N}$ is an empirical distribution of $F_{\bb}^*$. 
Note that
\begin{align*}
	& \p\left(\rho^{F^*}(Y\cdot f_{\bb}(\mathbf X))\le
			\sup_{F\in \overline{\cal B}_p(\widehat{F}_N,g^{-1}(\epsilon_{p,N}(\eta)))}\rho^{F}(Y\cdot f_{\bb}(\mathbf X))\right)\\[3pt]
   &=  \p\left(\rho^{F^*_{\bm\beta}}(Z)\le
			\sup_{F\in \overline{\mathcal B}_{p|f_{\bb}}\left(\widehat{F}_N,g^{-1}(\epsilon_{p,N}(\eta))\right) }\rho^{F}(Z)\right) \\[3pt]
  & \ge \p\left(F^*_{\bm\beta}\in  \overline{\mathcal B}_{p|f_{\bb}}\left(\widehat{F}_N,g^{-1}(\epsilon_{p,N}(\eta))\right)\right)\notag\\[3pt]
&\ge \p\left(F_{\bb}^*\in \mathcal B_p(\widehat{F}_{\bb,N},\epsilon_{p,N}(\eta))\right)\ge 1-\eta,
\end{align*}
where $\overline{\mathcal B}_{p |f_{\bb}}(F_0,\epsilon)$ is defined by \eqref{appeq-HONE},
the second inequality follows from Lemma \ref{lm-setinclusion}, which states that $\mathcal B_p(\widehat{F}_{\bb,N},\epsilon)\subseteq \overline{\mathcal B}_{p |f_{\bb}}(\widehat{F}_N,g^{-1}(\epsilon))$,
and the last inequality is due to Lemma \ref{lm-CBWD} by noting that Assumption \ref{ass:UB021} implies $\E^{F_{\bb}^*}[\exp(|Z|^a)]=\E^{F^*}[\exp(|f_{\bb}(\mathbf X)|^a)]\le A<\infty$. Hence, we complete the proof.
\endproof

~~~

\noindent{\bf Step 3: Union bounds for $\bb\in \mathcal D$.}  Now we are ready to prove union bounds in Theorem \ref{co-nonunion}.

{\bf Proof of Theorem \ref{co-nonunion}.}
Note that for any $(Y,\mathbf X)\in \{-1,1\}\times\R^n$ and $\bb,\widetilde{\bb}\in \mathcal D$, we have
	\begin{align}\label{eq-nonunion1}
		\left|\rho^F(Y \cdot f_{\bb}(\mathbf X))-\rho^F(Y \cdot f_{\widetilde{\bm \beta}}(\mathbf X))\right|
		&\le M \left(\E^F[|f_{\bb}(\mathbf X)-f_{\widetilde{\bm \beta}}(\mathbf X)|^k]\right)^{1/k}\notag\\[3pt]
		&\le M \|\bm \beta-\widetilde{\bm \beta}\|_{\mathcal D}(\E^F[(a_1\|\mathbf X\|^r+a_2)^k])^{1/k}\notag\\[3pt]
& \le M \|\bm \beta-\widetilde{\bm \beta}\|_{\mathcal D}\left(a_1\left(\E^F[|\| \mathbf X\||^{rk}]\right)^{1/k}+a_2\right), 
	\end{align}
 where the three inequalities follow from Assumption \ref{assum-4}, Assumption \ref{ass:UB01} and triangle's inequality for the $L^k$-norm, respectively.
Then, it holds that
	\begin{align}\label{eq-nonunion2}
	\lefteqn{ \left|\sup_{F\in \overline{\cal B}_p(\widehat{F}_N,\epsilon)}\rho^{F}(Y\cdot f_{\bb}(\mathbf X))-
	\sup_{F\in \overline{\cal B}_p(\widehat{F}_N,\epsilon)}\rho^{F}(Y\cdot  f_{\widetilde{\bm \beta}}(\mathbf X))\right| }\notag\\[3pt]
		&\le \sup_{F\in \overline{\cal B}_p(\widehat{F}_N,\epsilon)}\left|\rho^{F}(Y\cdot  f_{\bm \beta}(\mathbf X))-\rho^{F}(Y\cdot f_{\widetilde{\bm \beta}}(\mathbf X))\right| \notag\\[3pt]
  &\le M\|\bm \beta-\widetilde{\bm \beta}\|_{\mathcal D}\sup_{F\in \overline{\cal B}_p(\widehat{F}_N,\epsilon)} \left(a_1 \left(\E^{F}[\|\mathbf X\|^{rk}]\right)^{1/k}+a_2\right)\notag\\[3pt]
 & \le M\|\bm \beta-\widetilde{\bm \beta}\|_{\mathcal D}\left(2^{r-1}a_1\left(\E^{\widehat{F}_{N}}[\|\mathbf X\|^{rk}]\right)^{1/k}+2^{r-1}a_1\epsilon^r+a_2\right),
\end{align}
where the second inequality follows from \eqref{eq-nonunion1},
and the last inequality holds because  $ (\E^F[\|\mathbf X\|^{rk}] )^{1/(rk)}\le  (\E^{\widehat{F}_N} [\|\mathbf X\|^{rk}] )^{1/(rk)}+\epsilon $ for any $F \in \overline{\cal B}_p(\widehat{F}_N,\epsilon)$ with $rk\in [1,p]$, which further implies 
$$
(\E^F[\|\mathbf X\|^{rk}] )^{1/k}\le  ((\E^{\widehat{F}_N} [\|\mathbf X\|^{rk}] )^{1/(rk)}+\epsilon)^r\le 2^{r-1}((\E^{\widehat{F}_N} [\|\mathbf X\|^{rk}] )^{1/k}+\epsilon^r).
$$
For $0<x\le y$ and $t\ge 0$,  it holds that $y-x>t^{1/k}$ implies $y^k-x^k>t$. Hence, we have
	\begin{align}\label{eq-nonunion3}
		&\p\left(\left(\E^{\widehat{F}_N}[\|\mathbf X\|^{rk}]\right)^{1/k}-\left(\E^{F^*}[\|\mathbf X\|^{rk}]\right)^{1/k}>\left({\rm\mathbb{V}ar}^{F^*}(\|\mathbf X\|^{rk})\right)^{1/(2k)}\right)\notag\\[3pt]
		&\le	\p\left(\E^{\widehat{F}_N}[\|\mathbf X\|^{rk}]-\E^{F^*}[\|\mathbf X\|^{rk}]>\sqrt{{\rm\mathbb{V}ar}^{F^*}(\|\mathbf X\|^{rk})}\right)\le \frac{1}{N},
	\end{align}
	where the second inequality follows from Chebyshev’s inequality.
Let $\tau>0$ and $\mathcal D_\tau$ be an $\tau$-cover of $\mathcal D$ with respect to the norm $\|\cdot\|_{\mathcal D}$. Denote by $\tau=1/N$,  $t=\mathcal N(\tau;\mathcal D,\|\cdot\|_{\mathcal D})$
and $\epsilon_N(\eta/t):=g^{-1}(\epsilon_{p,N}(\eta/t))$, where $s\mapsto \epsilon_{p,N}(s)$ is defined by \eqref{eq-epsilon} in Lemma \ref{lm-CBWD}.
It holds that
	\begin{align}
&1-\p\left(	\rho^{F^*}(Y\cdot f_{\bb}(\mathbf X))\le \sup_{F\in \overline{\cal B}_p(\widehat{F}_N,\epsilon_N(\eta/t))}\rho^{F}(Y\cdot f_{\bb}(\mathbf X))+N\tau\tau_N,~~\forall~\bm \beta\in\mathcal D \right)\notag\\[3pt]
&=\p\left(\exists \bb\in\mathcal D ~{\rm s.t.}~\rho^{F^*}(Y\cdot f_{\bb}(\mathbf X))> \sup_{F\in \overline{\cal B}_p(\widehat{F}_N,\epsilon_N(\eta/t))}\rho^{F}(Y\cdot f_{\bb}(\mathbf X))
		\right.\notag\\[3pt]
		&~~~~~~~~~~~\left.
		+\tau M\left[(2^{r-1}+1)a_1\left(\E^{F^*}[\|\mathbf X\|^{rk}]\right)^{1/k}\!\!\!\!\!+2^{r-1}a_1\left({\rm\mathbb{V}ar}^{F^*}(\|\mathbf X\|^{rk})\right)^{1/(2k)}\!\!\!\!\!+2^{r-1}a_1\epsilon_N^r(\eta/t)+2a_2\right]\right)\notag\\[3pt]
		&\le \frac{1}{N}+\p\left(\exists \bb\in\mathcal D ~{\rm s.t.}~\rho^{F^*}(Y\cdot f_{\bb}(\mathbf X))> \sup_{F\in \overline{\cal B}_p(\widehat{F}_N,\epsilon_N(\eta/t))}\rho^{F}(Y\cdot f_{\bb}(\mathbf X))
		\right.\notag\\[3pt]
		&~~~~~~~~~~~~~~~~~~~~~~~~~\left.+\tau M\left[a_1\left(\E^{F^*}[\|\mathbf X\|^{rk}]\right)^{1/k}\!\!\!+2^{r-1}a_1\left(\E^{\widehat{F}_{N}}[\|\mathbf X\|^{rk}]\right)^{1/k}\!\!\!+2^{r-1}a_1\epsilon_N^r(\eta/t)+2a_2\right]\right)\notag\\[3pt]
        &=\frac{1}{N}+\p\left(\exists \bb\in\mathcal D ~{\rm s.t.}~\rho^{F^*}(Y\cdot f_{\bb}(\mathbf X))-\tau M\left[a_1\left(\E^{F^*}[\|\mathbf X\|^{rk}]\right)^{1/k}\!\!\!\!+a_2\right]>\!\!\!\! \sup_{F\in \overline{\cal B}_p(\widehat{F}_N,\epsilon_N(\eta/t))}\!\!\!\rho^{F}(Y\cdot f_{\bb}(\mathbf X))
		\right.\notag\\[3pt]
		&~~~~~~~~~~~~~~~~~~~~~~~~~~\left.+\tau M\left[2^{r-1}a_1\left(\E^{\widehat{F}_{N}}[\|\mathbf X\|^{rk}]\right)^{1/k}\!\!\!+2^{r-1}a_1\epsilon_N^r(\eta/t)+a_2\right]\right)\label{eq-ineq21}\\[3pt]
		&\le \frac{1}{N}+\p\left(\exists {\widetilde{\bb}}\in\mathcal D_\tau~{\rm s.t.}~\rho^{F^*}(Y\cdot f_{{\widetilde{\bb}}}(\mathbf X))> \sup_{F\in \overline{\cal B}_p(\widehat{F}_N,\epsilon_N(\eta/t))}\rho^{F}(Y\cdot f_{\mathbf {\widetilde{\bb}}}(\mathbf X))\right)\label{eq-ineq22}\\[3pt]
		&\le \frac{1}{N}+\sum_{{\widetilde{\bb}}\in\mathcal D_\tau}\p\left(\rho^{F^*}(Y\cdot f_{{\widetilde{\bb}}}(\mathbf X))> \sup_{F\in \overline{\cal B}_p(\widehat{F}_N,\epsilon_N(\eta/t))}\rho^{F}(Y\cdot {f_{\widetilde{\bb}}}(\mathbf X))\right)\notag\\[3pt]
  &\le \frac{1}{N}+\frac{\mathcal N(\tau;\mathcal D,\|\cdot\|_{\mathcal D})\eta}{t}=\frac{1}{N}+\eta,\notag
	\end{align}
	where the first inequality follows from \eqref{eq-nonunion3}, the last inequality is due to Lemma \ref{lm-CB1}, and the second inequality holds because if the event in \eqref{eq-ineq21} happens with some $\bb\in\mathcal D$, then there exists $\widetilde{\bb}\in\mathcal D_{\tau}$ such that $\|\bb-\widetilde{\bb}\|_{\mathcal D}\le \tau$ and
\begin{align*}
\rho^{F^*}(Y\cdot f_{\widetilde{\bb}}(\mathbf X))&\ge
\rho^{F^*}(Y\cdot f_{\bb}(\mathbf X))-\tau M\left[a_1\left(\E^{F^*}[\|\mathbf X\|^{rk}]\right)^{1/k}+a_2\right]\\
&> \sup_{F\in \overline{\cal B}_p(\widehat{F}_N,\epsilon_N(\eta/t))}\rho^{F}(Y\cdot f_{\bb}(\mathbf X))+\tau M\left[2^{r-1}a_1\left(\E^{\widehat{F}_{N}}[\|\mathbf X\|^{rk}]\right)^{1/k}+2^{r-1}a_1\epsilon_N^r(\eta/t)+a_2\right]\\
&\ge \sup_{F\in \overline{\cal B}_p(\widehat{F}_N,\epsilon_N(\eta/t))}\rho^{F}(Y\cdot f_{\mathbf {\widetilde{\bb}}}(\mathbf X)),
\end{align*}
where the first and the last inequalities follow from 
\eqref{eq-nonunion1} and \eqref{eq-nonunion2}, respectively, which implies that the event in \eqref{eq-ineq22} happens.
Hence, we have
\begin{align}\label{eq-th8}
\p\left(\rho^{F^*}(Y\cdot f_{\bb}(\mathbf X))\le \sup_{F\in \overline{\cal B}_p(\widehat{F}_N,g^{-1}(\epsilon_{p,N}(\eta/t)))}\rho^{F}(Y\cdot f_{\bb}(\mathbf X))+\tau_N,~~\forall~\bm \beta\in\mathcal D \right)\ge 1-\eta-\frac{1}{N}.
\end{align}
We recall that $t=\mathcal N(\tau;\mathcal D,\|\cdot\|_{\mathcal D})$ with $\tau=1/N$.
Note that $\mathcal D$ is contained in $U_{\mathcal D} B$ where $B$ is a unit ball of $\R^m$.
For any $\epsilon>0$, we have (see Example 5.8 of \cite{W19})
\begin{align*}
t=\mathcal N\left(\frac{1}{N};\mathcal D,\|\cdot\|_{\mathcal D}\right)
\le \left(1+{2U_{\mathcal D}N}\right)^m=:\widetilde{t}.
\end{align*}
Let $\epsilon_{p,N}$ be defined in Theorem \ref{co-nonunion}.
Noting that $g^{-1}$ is increasing and $s\mapsto \epsilon_{p,N}(s)$ is decreasing, we obtain 
$$
g^{-1}(\epsilon_{p,N})=g^{-1}\left(\epsilon_{p,N}\left(\eta/\widetilde{t}\right)\right)\ge g^{-1}(\epsilon_{p,N}(\eta/t)).
$$
Combining with \eqref{eq-th8}, we complete the proof.
\endproof

\subsection{Proof of Theorem \ref{co-nonunion1}} 
The proof of Theorem \ref{co-nonunion1} closely follows the three-step structure used for Theorem \ref{co-nonunion}, with a key difference in the first step. Specifically, establishing the inclusion relation \eqref{eq-setinall-0} becomes challenging because Assumption \ref{ass:UB04} is weaker than Assumption \ref{ass:UB03}. However, as shown in the first step of the proof below, the ``full" inclusion relation \eqref{eq-setinall-0} is unnecessary; instead, only a partial inclusion relation is required. Specifically, the worst-case risk evaluated over the full Wasserstein ball coincides with that evaluated over a subset of the Wasserstein ball, and this subset can be shown to belong to the projection set. This enables the derivation of confidence bounds for fixed $\bb$ in the second step, and subsequently, the generalization bounds using the same covering number arguments as in Theorem \ref{co-nonunion}.

~~

{\bf Step 1. Proving partial set inclusion and coincidence of worst-case risk.}  
For $H_0\in\mathcal M_p(\R)$ and $\epsilon\ge 0$,
define a subset of a one-dimensional Wasserstein ball as
\begin{align}\label{eq-Bbeta+}
\mathcal B_{p}^+(H_0,\epsilon)=  \{F\in\mathcal B_{p}(H_0,\epsilon): F\succeq_{\rm st}H_0\},
\end{align}
where $\succeq_{\rm st}$ denotes the \emph{usual stochastic order} (see \cite{SS07}), also known as the first-order stochastic dominance, and is defined as:  for two distributions $H_1,H_2$ on $\R$,  $H_1\succeq_{\rm st} H_2$ if $H_1(t)\le H_2(t)$ for all $t\in\R$.
We aim to show the following 
set inclusion result:
\begin{align}\label{eq-setinclusion+}
\mathcal B_{p|f_{\bb}}( G_0,\epsilon)\supseteq \mathcal B_{p}^+(F_{f_{\bb}(\mathbf X_0)},g(\epsilon)),
\end{align}
where $\mathbf X_0\sim G_0\in \mathcal M_p(\R^n)$, $\mathcal B_{p}^+(F_{f_{\bb}(\mathbf X_0)},g(\epsilon))$ is defined by 
\eqref{eq-Bbeta+},
and  
\begin{align}\label{appeq-HONE1}
\mathcal B_{p|f_{\bb}}( G_0,\epsilon): =\{F_{f_{\bb}(\mathbf X)}: F_{\mathbf X}\in{\mathcal B}_p(G_0,\epsilon)\} .
\end{align}



\begin{lemma}\label{lm-setinclusion+}
For $p\in[1,\infty]$, $\epsilon\ge0$, $G_0\in\mathcal M({\mathbb R}^n)$, and $\mathbf X_0\sim G_0$, the set inclusion result \eqref{eq-setinclusion+} holds 
for each $\bb\in\mathcal D$ under Assumption \ref{ass:UB05}.
\end{lemma}

\proof{Proof.}
Let $\bb\in\mathcal D$.
For $F\in \mathcal B_{p}^+(F_{f_{\bb}(\mathbf X_0)},g(\epsilon))$, we have $F\in \mathcal B_{p}(F_{f_{\bb}(\mathbf X_0)},g(\epsilon))$ and $F\succeq_{\rm st} F_{f_{\bb}(\mathbf X_0)}$. Let $U$ be a uniform random variable on $[0,1]$ such that $f_{\bb}(\mathbf X_0)=F_{f_{\bb}(\mathbf X_0)}^{-1}(U)$ almost surely (see Lemma A.28 of \cite{FS16} for the exsistence of $U$).
Let $Z=F^{-1}(U)\sim F$ and $T=Z-f_{\bb}(\mathbf X_0)$. Since $F\succeq_{\rm st} F_{f_{\bb}(\mathbf X_0)}$, we have 
$T\ge 0$. Moreover, it holds that
\begin{align}\label{eq-+T}
\E[|T|^p]=\E[|Z-f_{\bb}(\mathbf X_0)|^p]=\int_0^1 |F^{-1}(s)-F_{f_{\bb}(\mathbf X_0)}^{-1}(s)|^p\d s
=W_p(F,F_{f_{\bb}(\mathbf X_0)})^p\le g(\epsilon)^p,
\end{align}
where 
the inequality is due to $F\in \mathcal B_{p}(F_{f_{\bb}(\mathbf X_0)},g(\epsilon))$.
Using the similar arguments in the proof of Lemma \ref{lm-setinclusion} where a measurable selection theorem in \cite{R78} has been used, we know that there exists a measurable mapping $\mathbf S$ such that 
\begin{align*}
\mathbf S(\omega)\in\argmax_{\|\mathbf y\|\le g^{-1}(T(\omega))} f_{\bb}({\mathbf X}_0(\omega)+\mathbf y),~~\omega\in \Omega.
\end{align*}
Using Assumption \ref{ass:UB05}, we have
\begin{align*}
f_{\bb}(\mathbf X_0(\omega) + \mathbf S(\omega))-f_{\bb}({\mathbf X}_0(\omega))\ge g(g^{-1}(T(\omega)))=T(\omega),~~\omega\in \Omega.
\end{align*}
Note that $f_{\bb}$ is continuous. There exists a random variable $R$ taking values on $[0,1]$ such that 
\begin{align}\label{eq-mon2}
f_{\bb}({\mathbf X}_0+R\mathbf S)-f_{\bb}({\mathbf X}_0)=T.
\end{align}
Define $\widetilde{\mathbf X}={\mathbf X}_0+R\mathbf S$. By \eqref{eq-mon2}, we have $f_{\bb}(\widetilde{\mathbf X})=f_{\bb}({\mathbf X}_0)+T=Z\sim F$. By  $ \|\mathbf S\|\le g^{-1}(T)$, we have
\begin{align*}
\E[\|\widetilde{\mathbf X}-{\mathbf X}_0\|^p]
\le \E\left[(g^{-1}(T))^p\right]
\le \left(g^{-1}((\E[T^p])^{1/p})\right)^p
\le \epsilon^p,
\end{align*}
where 
the second inequality follows from Jensen's inequality and   the concavity of the mapping $x\mapsto (g^{-1}(x^{1/p}))^p$ by Lemma \ref{lm-auxcxcv}, and the last inequality is due to \eqref{eq-+T}.
Hence, we have $F_{\widetilde{\mathbf X}}\in \mathcal B_p(G_0,\epsilon)$, which implies 
$F=F_Z=F_{f_{\bb}(\widetilde{\mathbf X})}\in  \mathcal B_{p|f_{\bb}}(G_0,\epsilon)$. This completes the proof.\endproof

Below we demonstrate that the worst-case values of a monotone risk mapping are identical under $\mathcal B_p(H_0,\epsilon)$ and $\mathcal B_p^+(H_0,\epsilon)$ for any $H_0\in\mathcal M_p(\R)$ and $\epsilon\ge 0$. 

\begin{lemma}\label{lm-eq+}
Let $p\in[1,\infty]$ and 
suppose that $\rho: L^p\to\R$ satisfies monotonicity. For $H_0\in\mathcal M_p(\R)$ and $\epsilon\ge 0$, we have
\begin{align}\label{appeq-1proofth7}
\sup_{F\in \mathcal B_p^+(H_0,\epsilon)}\rho^F(Z)=\sup_{F\in \mathcal B_p(H_0,\epsilon)}\rho^F(Z).
\end{align}
\end{lemma}
\proof{Proof.}
Note that $\mathcal B_p^+(H_0,\epsilon) \subseteq \mathcal B_p(H_0,\epsilon)$. It suffices to show \begin{align}\label{appeq-1proofth7-1}
\sup_{F\in \mathcal B_p^+(H_0,\epsilon)}\rho^F(Z)\ge \sup_{F\in \mathcal B_p(H_0,\epsilon)}\rho^F(Z).
\end{align}
For any $F\in\mathcal B_{p}(H_0, \epsilon)$, define $H(x)=\min\{F(x),H_0(x)\}$ for $x\in\R$. It holds that 
$H^{-1}(s)=\max\{F^{-1}(s), H_0^{-1}(s)\}$ for $s\in [0,1]$, and thus, 
\begin{align*}
W_p(H,H_0)^p
&=\int_0^1 \left|\max\{F^{-1}(s), H_0^{-1}(s)\}-H_0^{-1}(s)\right|^p\d s\\
&\le\int_0^1 \left|F^{-1}(s)-H_0^{-1}(s)\right|^p\d s =W_p(F,H_0)^p\le \epsilon^p,
\end{align*}
where 
the last inequality is due to $F\in \mathcal B_p(H_0,\epsilon)$. This means $H\in \mathcal B_p(H_0,\epsilon)$.
This, together with  $H\le H_0$, that is, $H\succeq_{\rm st} H_0$,  implies $H\in \mathcal B_{p}^+(H_0, \epsilon)$. 
By the monotonicity of $\rho$, we have  $\rho^H(Z)\ge \rho^F(Z)$ as $H\succeq_{\rm st} F$.
Therefore, we have \eqref{appeq-1proofth7-1} holds, and thus,  we complete the proof.
\endproof

~~

{\bf Step  2: Establishing confidence bounds for fixed $\bb$.}  Similar to Lemma \ref{lm-CB1} in Step 2 for Theorem \ref{co-nonunion},  we can obtain the confidence bounds for fixed $\bb$ as stated in the following lemma.

\begin{lemma}\label{lm-thm8-3}
For $p\ge 1$ and $\eta\in(0,1)$, suppose that Assumptions \ref{ass:UB021} and \ref{ass:UB04} hold. Then for each $\bb\in\mathcal D$ and a monotone measure $\rho$,  we have
\begin{align*}
\p\left(\rho^{F^*}( f_{\bb}(\mathbf X))\le \sup_{F\in  {\cal B}_p(\widehat{F}_{\mathbf X, N},g^{-1}(\epsilon_{p,N}(\eta)))}\rho^{F}( f_{\bb}(\mathbf X))\right) \ge 1-\eta,
\end{align*}
where $\epsilon_{p,N}(\eta)$ is defined by \eqref{eq-epsilon}  with  constants $c_1,c_2$  depending only on $a$, $A$ and $p$.
\end{lemma}
\proof{Proof.}  
Denote by $\widehat{F}_{\bb,N}=\frac{1}{N}\sum_{i=1}^N\delta_{f_{\bb}(\widehat{\mathbf x}_i)}$ and $F^*_{\bb}=F_{f_{\bb}(\mathbf X^*)}$, where $(Y^*,\mathbf X^*)\sim F^*$. 
Note that by Lemma \ref{lm-setinclusion+}, we have 
\begin{align*} 
\mathcal B_{p|f_{\bb}}( \widehat{F}_{\mathbf X, N},\epsilon)\supseteq \mathcal B_{p}^+(\widehat{F}_{\bb,N},g(\epsilon)),
\end{align*}
where $\mathcal B_{p|f_{\bb}}( \widehat{F}_{\mathbf X, N},\epsilon)=\{F_{f_{\bb}(\mathbf X)}: F_{\mathbf X}\in \mathcal B_p(\widehat{F}_{\mathbf X, N},\epsilon)\}$ and 
$ 
\mathcal B_{p}^+(\widehat{F}_{\bb,N},g(\epsilon))=  \{F\in \mathcal B_{p}(\widehat{F}_{\bb,N},g(\epsilon)): F\succeq_{\rm st} \widehat{F}_{\bb,N}\}.
$  
Therefore,  we have 
\begin{align}\label{appeq-1proof0th7}
 \sup_{F\in\mathcal B_{p|f_{\bb}}( \widehat{F}_{\mathbf X, N},\epsilon)}\rho^{F}(Z)\ge \sup_{F\in\mathcal B_{p}^+(\widehat{F}_{\bb,N}, g(\epsilon))}\rho^{F}(Z) =\sup_{F\in\mathcal B_{p}(\widehat{F}_{\bb,N}, g(\epsilon))}\rho^{F}(Z),
\end{align}
where the  equality follows from  Lemma \ref{lm-eq+} as $\rho$ is monotone. 
It then follows that 
\begin{align*}
\p\left(\rho^{F^*}(f_{\bb}(\mathbf X))\le \sup_{F\in {\cal B}_p(\widehat{F}_{\mathbf X,N},\epsilon)}\rho^{F}(f_{\bb}(\mathbf X))\right)
&=\p\left(\rho^{F_{\bb}^*}(Z)\le \sup_{F\in {\mathcal B}_{p|f_{\bb}}(\widehat{F}_{\mathbf X, N},\epsilon)}  \rho^{F}(Z)\right)\notag\\[3pt]
&\ge \p\left(\rho^{F_{\bb}^*}(Z)\le \sup_{F\in\mathcal B_{p}(\widehat{F}_{\bb,N}, g(\epsilon))}\rho^{F}( Z)\right)\notag\\[3pt]
&\ge \p\left( F^*_{\bb} \in  {{\cal B}}_p(\widehat{F}_{\bb,N} ,g(\epsilon))\right),
\end{align*}
where the first inequality follows from 
\eqref{appeq-1proof0th7}.
Substituting $\epsilon: = \epsilon_N=g^{-1}(\epsilon_{p,N}(\eta))$ into the above inequalities, we have 
\begin{align*} 
\p\left(\rho^{F^*}(f_{\bb}(\mathbf X))\le \sup_{F\in {\cal B}_p(\widehat{F}_{\mathbf X,N},\epsilon_N)}\rho^{F}(f_{\bb}(\mathbf X))\right)
& \ge \p\left( F^*_{\bb} \in  {{\cal B}}_p(\widehat{F}_{\bb,N} ,\epsilon_{p,N}(\eta))\right)\ge 1-\eta.
\end{align*}
where the last inequality follows from  Lemma \ref{lm-CBWD}, similar to the proof of Lemma \ref{lm-CB1}. This completes the proof.
\endproof

{\bf Step 3:  Union bounds for $\bb \in {\cal D}$.} Now we are ready to prove the last step, union bounds, for Theorem \ref{co-nonunion1}.

{\bf Proof of Theorem \ref{co-nonunion1}.} The proof of union bounds based on confidence bounds is similar to that  of Theorem \ref{co-nonunion} by applying the covering number arguments. 
 \endproof

\subsection{Proof of Theorem \ref{co-nonunion2}} \label{app:th10}

{

{\it Proof of Theorem \ref{co-nonunion2}.} 
Note that  for $Z_1,Z_2\in L^1,$
\begin{align*}
\rho(Z_1)\le \sup_{\E[|V|]\le \E[|Z_1-Z_2|]}\rho(Z_2+V)\le \sup_{|V|\le \lambda\E[|Z_1-Z_2|]}\rho(Z_2+V)\le \rho(Z_2)+M\lambda\E[|Z_1-Z_2|],
\end{align*}
where the second inequality follows from Assumption
\ref{ass:order1} and the last inequality is due to the Lipschitz continuity of $\rho$ with respect to $L^\infty$-norm, which implies
 $
|\rho(Z_2 +V)-\rho(Z_2)|\le M\esssup|V|$.
  By symmetry, we conclude that 
\begin{align*}
|\rho(Z_1)-\rho(Z_2)|\le \lambda M \E[|Z_1-Z_2|],~~\forall Z_1,Z_2\in L^1.
\end{align*}
Hence, Assumption \ref{assum-4} holds with $k=1$ and $M$ being replaced by $\lambda M$. 
Further, let $Z_N\sim \widehat{F}_{\bb,N}$, and it holds for any $\epsilon\ge 0$ that 
\begin{align}\label{eq-morethan}
\sup_{\mathcal B_{\infty}(\widehat{F}_{\bb,N},\epsilon)}\rho^F(Z)
=\sup_{ |Y-Z_N|\le  \epsilon} \rho(Y) 
&=\sup_{ |V|\le  \epsilon} \rho(Z_N+V)\notag\\[3pt]
&\ge \sup_{ \E[|V|]\le \epsilon/\lambda} \rho(Z_N+V)\notag\\[3pt]
&=\sup_{ \E[|Y-Z_N|]\le \epsilon/\lambda} \rho(Y)
=\sup_{\mathcal B_{1}(\widehat{F}_{\bb,N},\epsilon/\lambda)}\rho^F(Z),
\end{align}
where the inequality is due to Assumption \ref{ass:order1}.
Then,
we have the following chain of inequalities:
\begin{align}\label{eq-chainmon}
\p\left(\rho^{F^*}(f_{\bb}(\mathbf X))\le \sup_{F\in {\cal B}_p(\widehat{F}_{\mathbf X,N},\epsilon)}\rho^{F}(f_{\bb}(\mathbf X))\right)
&\ge \p\left(\rho^{F^*}(f_{\bb}(\mathbf X))\le \sup_{F\in {\cal B}_{\infty}(\widehat{F}_{\mathbf X,N},\epsilon)}\rho^{F}(f_{\bb}(\mathbf X))\right)\notag\\[3pt]
&=\p\left(\rho^{F_{\bb}^*}(Z)\le \sup_{F\in \mathcal B_{\infty|f_{\bb}}(\widehat{F}_{\mathbf X, N},\epsilon)}  \rho^{F}(Z)\right)\notag\\[3pt]
&\ge \p\left(\rho^{F_{\bb}^*}(Z)\le \sup_{F\in \mathcal B_{\infty}^+(\widehat{F}_{\bb,N}, g(\epsilon))}  \rho^{F}(Z)\right)\notag\\[3pt]
&= \p\left(\rho^{F_{\bb}^*}(Z)\le \sup_{F\in\mathcal B_{\infty}(\widehat{F}_{\bb,N}, g(\epsilon))}\rho^{F}( Z)\right)\notag\\[3pt]
&\ge\p\left(\rho^{F_{\bb}^*}(Z)\le \sup_{F\in\mathcal B_1(\widehat{F}_{\bb,N}, \, g(\epsilon)/\lambda)}\rho^{F}( Z)\right)\notag\\[3pt]
&\ge \p\left( F^*_{\bb} \in  {{\cal B}}_1(\widehat{F}_{\bb,N} ,\, g(\epsilon)/\lambda)\right),
\end{align}
where the first inequality is due to ${\mathcal B}_{\infty}(\widehat{F}_{\mathbf X,N},\epsilon_N)\subseteq {\mathcal B}_p(\widehat{F}_{\mathbf X,N},\epsilon_N)$, the second inequality follows from Lemma \ref{lm-setinclusion+}  which states that $\mathcal B_{\infty}^+(\widehat{F}_{\bb,N},g(\epsilon))\subseteq \mathcal B_{\infty|f_{\bb}}(\widehat{F}_{\mathbf X, N},\epsilon)$, the second equality is due to Lemma \ref{lm-eq+}, and the third inequality follows from \eqref{eq-morethan}. For $\eta\in(0,1)$, combining \eqref{eq-chainmon} with Lemma \ref{lm-CBWD} yields the following confidence bound:
\begin{align*}
\p\left(\rho^{F^*}(f_{\bb}(\mathbf X))\le \sup_{F\in {\cal B}_p(\widehat{F}_{\mathbf X,N},g^{-1}(\lambda\epsilon_{1,N}(\eta)))}\rho^{F}(f_{\bb}(\mathbf X))\right)
\ge \p\left( F^*_{\bb} \in  {{\cal B}}_1(\widehat{F}_{\bb,N} ,\epsilon_{1,N}(\eta))\right)\ge 1-\eta,
\end{align*}
where $\epsilon_{1,N}(\eta)$ is defined by \eqref{eq-epsilon}.
Note that Assumptions \ref{ass:UB2} and \ref{ass:UB03} hold, Assumption \ref{ass:UB02} holds with $r=1$, and Assumption \ref{assum-4} holds with $k=1$ and $M$ being replaced by $\lambda M$.
The rest proof is similar to that  of Theorem \ref{co-nonunion} by applying the covering number arguments. 
\endproof
}

\subsection{A sufficient condition for Assumption \ref{ass:UB2}}\label{sec:sufcoU2}

Below we present a sufficient condition for Assumption \ref{ass:UB2} that is more straightforward to verify.

\begin{proposition}\label{prop-lighttail}
Assumption \ref{ass:UB2} holds if
  there exist $b_1,b_2\ge 0$,  $s\in[1,p]$ and $a>p$ such that 
  \begin{align}\label{eq:Assumpt2-suff}
|f_{\bb}(\mathbf x)|\le b_1\|\mathbf x\|^s+b_2,  ~~\forall \mathbf x\in\R^n,~\bb\in\mathcal D,
\end{align}
and
  \begin{align}\label{eq:Assumpt2-suff-2}
\E^{F^*}[\exp(2^{a-1}b_1^a\|\mathbf X\|^{as})]<\infty.
\end{align}

\end{proposition}

{\it Proof.} 
Note that 
\begin{align*}
\sup_{\bb\in\mathcal D}\E^{F^*}[\exp(|f_{\bb}(\mathbf X)|^a)]
&\le \sup_{\bb\in\mathcal D}\E^{F^*}[\exp(|b_1\|\mathbf x\|^s+b_2)^a]\\[3pt]
&\le \sup_{\bb\in\mathcal D}\E^{F^*}[\exp(2^{a-1}b_1^a\|\mathbf x\|^{as}+2^{a-1}b_2^a)]\\[3pt]
&=\exp(2^{a-1}b_2^a)\sup_{\bb\in\mathcal D}\E^{F^*}[\exp(2^{a-1}b_1^a\|\mathbf x\|^{as})]<\infty,
\end{align*}
which confirms Assumption \ref{ass:UB2}. 
This completes the proof.
\endproof

Proposition \ref{prop-lighttail} shows that Assumption \ref{ass:UB2} can be verified through a combination of a union-bound condition (\ref{eq:Assumpt2-suff}) on $\{f_{\bb}\}_{\bb\in\mathcal D}$ and a light-tail condition (\ref{eq:Assumpt2-suff-2}) on the data-generating distribution $F^*$. 
It is worth noting that (\ref{eq:Assumpt2-suff})  holds if $f_{\bb_0}(\mathbf x)\equiv0$ when $\bb_0=\mathbf 0\in \R^m$, Assumption \ref{ass:UB02} holds on $\mathcal D\cup\{\bb_0\}$ and $U_{\mathcal D}:=\sup_{\bb\in\mathcal D}\|\bb\|_{\mathcal D}<\infty$. 
To see this, note that by Assumption \ref{ass:UB02}, we have
\begin{align*}
|f_{\bb}(\mathbf x)|\le (a_1\|\mathbf x\|^r+a_2)\|\bb\|_{\mathcal D}\le a_1U_{\mathcal D}\|\mathbf x\|^r+a_2U_{\mathcal D},~~\forall \mathbf x\in\R^n,~\bb\in\mathcal D,
\end{align*}
and thus, \eqref{eq:Assumpt2-suff} holds with $b_1=a_1U_{\mathcal D}$, $b_2=a_2U_{\mathcal D}$ and $s=r$.

\bibliographystyle{nonumber}

\end{document}